\documentclass[twoside,11pt, preprint]{article}

\usepackage{jmlr2e}

\usepackage{amsmath,dsfont,mathrsfs}
\usepackage{array}
\usepackage{multirow}
\usepackage{caption}
\usepackage{color}
\usepackage{subcaption}

\usepackage{algorithm}
\usepackage[noend]{algorithmic}

\def\<{\langle}
\def\>{\rangle}

\def\Lout{L_{\text{out}}}
\newcommand{\normal}{\mathcal{N}}
\newcommand{\reals}{\mathbb{R}}
\newcommand{\iid}{{\rm i.i.d.}}
\newcommand{\bzero}{\boldsymbol{0}}
\newcommand{\beq}{\begin{equation}}
\newcommand{\eeq}{\end{equation}}

\newcommand{\sign}{\textrm{\sign}}
\newcommand{\E}{\mathbb{E}}
\newcommand{\tB}{\widetilde{B}}
\newcommand{\tZ}{\widetilde{Z}}
\newcommand{\Var}{\mathrm{Var}}
\newcommand{\Cov}{\mathrm{Cov}}

\newcommand{\mb}{\mathbb}
\newcommand\Tau{\mathrm{T}}
\newcommand\Mu{\mathrm{M}}
\newcommand\Tr{\mathrm{Tr}}

\DeclareMathOperator*{\argmax}{argmax}
\DeclareMathOperator*{\argmin}{argmin}

\newcommand{\hB}{\widehat{B}}
\newcommand{\hbeta}{\hat{\beta}}
\newcommand{\hR}{\widehat{R}}

\usepackage{lastpage}
\jmlrheading{23}{2023}{1-\pageref{LastPage}}{3/23}{}
{21-0000}{Nelvin Tan and Ramji Venkataramanan}

\ShortHeadings{Mixed Regression via AMP}{Nelvin Tan and Ramji Venkataramanan}
\firstpageno{1}

\begin{document}

\title{Mixed Regression via Approximate Message Passing}

\author{\name Nelvin Tan \email tcnt2@cam.ac.uk \\
       \addr Department of Engineering, 
\       University of Cambridge
       \AND
       \name Ramji Venkataramanan \email rv285@cam.ac.uk \\
       \addr Department of Engineering, \ 
       University of Cambridge
       }

\editor{TOFILL}

\maketitle

\begin{abstract}
    We study the problem  of regression in a generalized linear model (GLM) with multiple signals  and latent variables. This model, which we call a matrix GLM, covers many widely studied problems in statistical learning, including mixed linear regression, max-affine regression, and mixture-of-experts. In mixed linear regression, each observation comes from one of $L$ signal vectors (regressors), but we do not know which one; in max-affine regression, each observation comes from the maximum of $L$ affine functions, each defined via a different signal vector.  The goal in all these problems is to estimate the signals, and possibly some of the latent variables, from the observations. We propose a novel approximate message passing (AMP) algorithm for estimation in a matrix GLM and rigorously characterize its performance in the high-dimensional limit. This characterization is in terms of a state evolution recursion, which allows us to precisely compute performance measures such as the  asymptotic mean-squared error. The state evolution characterization can be used to tailor the AMP algorithm to take advantage of any structural information known about the signals. Using state evolution, we derive an optimal choice of AMP `denoising' functions that minimizes the estimation error in each iteration.
    
    The theoretical results are validated by numerical simulations  for mixed linear regression, max-affine regression, and mixture-of-experts.      For max-affine regression, we propose an algorithm that combines AMP with expectation-maximization to estimate intercepts of the model along with the signals.   The numerical results  show that AMP significantly outperforms other estimators for mixed linear regression and  max-affine regression in most parameter regimes.
\end{abstract}

\begin{keywords}
    Approximate Message Passing, Mixed Linear Regression, Max-Affine Regression, Mixture-of-Experts, Expectation-Maximization
\end{keywords}

\section{Introduction}

We study the problem  of regression in a generalized linear model with multiple signals (regressors)  and latent variables.  Specifically, consider $L$ signal vectors 
$\beta^{(1)},\dots,$ $\beta^{(L)} \in \reals^p$, and define the signal matrix $B := (\beta^{(1)},\dots,\beta^{(L)}) \in \reals^{p \times L}$. Then, the goal is to estimate $B$ from an observed matrix   $Y := [Y_1,\dots,Y_n]^\top\in\mb{R}^{n\times \Lout}$, whose $i$th row $Y_i \in\mb{R}^{\Lout}$ is generated as:
\begin{align}
    Y_i&=q(B^\top X_i \, , \, \Psi_i),
    \quad i\in[n]. \label{eq:matrix-GLM}
\end{align}
 Here $X_i \in  \reals^p$ is the $i$th feature vector, $\Psi_i \in \reals^{L_\Psi}$ is a vector of unobserved auxiliary variables, and $q:\mb{R}^L \times\mb{R}^{L_\Psi}\rightarrow\mb{R}^{\Lout}$ is a known function.
We refer to the model \eqref{eq:matrix-GLM} as the matrix generalized linear model, or \emph{matrix GLM}.   As we show below, the matrix  GLM covers many widely studied regression models including mixed linear regression, max-affine regression, mixed GLMs, and mixture-of-experts. 

\subsection{Mixed Linear Regression}

In this model, we wish to estimate  $L$ signal vectors from \emph{unlabeled} observations of each. Specifically, the components of the  observed vector $
Y:= (Y_1, \ldots, Y_n)^{\top}$ are generated as:
\begin{align}
    Y_i
    &=\langle X_i,\beta^{(1)}\rangle c_{i1}+\dots+\langle X_i,\beta^{(L)}\rangle c_{iL}\, + \, \epsilon_i,  \ \ i \in [n]. \label{eq:MLR_model}
\end{align}
Here  $\epsilon_i$ is a noise variable, and $c_{i1}, \ldots, c_{iL} \in \{ 0,1 \}$ are binary-valued latent variables such that 
$\sum_{l=1}^L c_{il}=1$, for $i \in [n]$. The notation $\< \cdot , \cdot \>$ denotes the Euclidean inner product and $[n]:=\{1,\dots,n\}$. In words, each observation comes from exactly one of the $L$ signal vectors, but we do not know which one. 
The mixed linear regression (MLR) model in \eqref{eq:MLR_model} is a special case of the matrix GLM in \eqref{eq:matrix-GLM}, with the rows of the auxiliary matrix given by $\Psi_i=(c_{i,1},\dots,c_{i,L},\epsilon_i)$,  for $i \in [n]$.

The case of $L=1$ is standard linear regression, which implicitly assumes a homogeneous population, i.e., a single regression vector captures the population characteristics of the entire sample. However,  this assumption may not be realistic in some situations as the sample may contain several sub-populations. Standard linear regression may provide biased estimates in such situations when the population heterogeneity is unobserved. The MLR model is more flexible as it allows for differences in regressors across unobserved sub-populations. MLR has been used for analyzing heterogeneous data in a variety of fields including biology, physics, and economics \citep{Mcl04,Gru07,Li19,Dev20}. 

In the MLR model \eqref{eq:MLR_model}, a natural approach for estimating $\beta^{(1)},\dots,\beta^{(L)}$ from $\{X_i,Y_i\}_{i=1}^n$ is via the global least-squares estimator given by:
\begin{align}
    & \widehat{\beta}^{(1)},\dots,\widehat{\beta}^{(L)}
    =\argmin_{\substack{\beta^{(1)},\dots,\beta^{(L)} \in\mb{R}^p \\ c_1,\dots,c_L\in\{0,1\}^n \\ \sum_{l=1}^L c_{il}=1, \, i \in [n]}}
    \sum_{i=1}^n\bigg(Y_i-\sum_{l=1}^L\langle X_i,\beta^{(l)}\rangle \, c_{il}\bigg)^2.
    \label{eq:globalLS}
\end{align}
However, this optimization problem is non-convex, and  computing the global minimum is known to be NP-hard \citep{Yi14}. A range of alternative approaches has been proposed including estimators based on: Bayesian methods \citep{Vie02}, spectral methods \citep{Cha13, Yi14}; expectation-maximization \citep{Far10, Sta10, Zha20}; alternating minimization \citep{Yi14, She19, Gho20};   convex relaxation \citep{Che14};  moment descent methods \citep{Li18,Che20}; and tractable non-convex objective functions \citep{Zho16,Bar22}. Most of these techniques are generic, and while some can incorporate certain constraints like sparsity, they are not well-equipped to exploit specific structural information about $\beta^{(1)},\dots,\beta^{(L)}$, such as   a known prior on the signals. Moreover, these methods are sub-optimal with respect to sample complexity: for accurate recovery they require the number of observations $n$ to be at least of order $p \log p$ \citep{Yi14,Li18,Che20}. In contrast, here we consider the high-dimensional regime where $n$ is proportional to $p$ and provide \emph{exact} asymptotics for the performance of the proposed estimator.

\subsection{Max-Affine Regression}

In the max-affine regression (MAR) model, we have
\begin{align}
    Y_i
    &=\max\big\{\langle X_i,\beta^{(1)}\rangle+b_{1},\dots,\langle X_i,\beta^{(L)}\rangle+b_{L}\big\}+\epsilon_i, \ \ i \in [n]. \label{eq:max-affine_model}
\end{align}
Here $b_1,\dots,b_L\in\mb{R}$ are the intercepts (typically unknown), and $\epsilon_i$ is a zero-mean noise variable that is independent of $X_i$. In words, each  observation comes from the maximum of $L$ affine functions, each defined via a different signal vector.

When $L=1$ and $b_1=0$, the model \eqref{eq:max-affine_model} corresponds to standard linear regression. When $L=2$ and $\beta^{(1)}=-\beta^{(2)}=\beta$ along with $b_1=b_2=0$, then \eqref{eq:max-affine_model} reduces to $Y_i=|\langle X_i,\beta\rangle|+\epsilon$. This  is the widely studied phase retrieval problem \citep{Net13,Can15}, which arises in applications such as scientific imaging  \citep{Fog16}. For general $L$, the function $x\mapsto\max_{l\in\{1,\dots,L\}}\{\langle x,\beta_l\rangle+b_l\}$ is always convex  and thus, estimation under model \eqref{eq:max-affine_model} can be used to fit convex functions to the observed data. Indeed, the MAR model serves as a parametric approximation to the non-parametric convex regression model
\begin{align}
    Y_i=\varphi(X_i)+\epsilon_i,   \ \ i \in [n] \label{eq:convex_reg}
\end{align}
where $\varphi:\mb{R}^p\rightarrow\mb{R}$ is an unknown convex function \citep{Bal15,Gho22}. Unfortunately, convex regression suffers from the curse of dimensionality unless $p$ is small \citep{Gun13}.  Since convex functions can be approximated to arbitrary accuracy by maxima of affine functions, it is reasonable to simplify the problem by considering only those convex functions that can be written as a maximum of a fixed number of affine functions. 
This assumption directly leads to the MAR model \eqref{eq:max-affine_model}, which has been studied  as a tractable alternative to  the non-parametric convex regression model \eqref{eq:convex_reg} in applications where $p$ is large, such as data in economics, finance and operations research \citep{Bal16}. MAR  can also be used as a tractable model for the problem of estimating convex sets from support function measurements \citep{Soh21}, which arises in  tomography applications \citep{Pri90,Gre02}.

To write the MAR model  as an instance of the matrix GLM \eqref{eq:matrix-GLM}, let us concisely denote the unknown parameters by $\beta^{(l)}_{\text{ma}}= \begin{bmatrix} \beta^{(l)} \\ b_{l} \end{bmatrix}\in\mb{R}^{p+1}$ for $l\in[L]$, and the observations by $\big(X_i^{({\text{ma}})},y_i\big)$ for $i\in[n]$, where $X_i^{({\text{ma}})}= \begin{bmatrix} X_i \\  1 \end{bmatrix}\in\mb{R}^{p+1}$ are the augmented features. Under the augmented features and signals, the model \eqref{eq:max-affine_model} becomes
\begin{align}
    Y_i=\max_{l\in[L]} \, \left\{\langle X_i^{({\text{ma}})} ,\beta^{(l)}_{{\text{ma}}}\rangle \right\} + \epsilon_i,   \ \ i \in [n],
    \label{eq:MAR_model_concise}
\end{align}
which is of the form in \eqref{eq:matrix-GLM}.
A natural approach for estimating $\beta^{(1)}_{\text{ma}},\dots,\beta^{(L)}_{\text{ma}}$ is the least squares estimator, defined as
\begin{align}
\widehat{\beta}^{(1)}_{{\text{ma}}},\dots,\widehat{\beta}^{(L)}_{{\text{ma}}} 
= \argmin_{\beta^{(1)}_{{\text{ma}}},\dots,\beta^{(L)}_{{\text{ma}}}\in\mb{R}^{p+1}}\sum_{i=1}^n\Big(Y_i-\max_{l\in[L]} \,  \left\{\langle X_i^{(\text{ma})},\beta^{(l)}_{{\text{ma}}}  \rangle\right\} \Big)^2. \label{eq:max_affine_LS}
\end{align}
\citet[Lemma 1]{Gho22} showed that a global minimizer of the least-squares criterion above always exists but will not in general be unique, since any relabelling of the indices (of the signal vectors) of a minimizer will also be a minimizer. Furthermore, the optimization problem in  \eqref{eq:max_affine_LS}  is non-convex, and for a worst-case choice of the design matrix $X=(X_1,\dots,X_n)^\top$, the problem is NP-hard \citep{Gho22}. 

\subsection{ Mixed Generalized Linear Models and Mixture-of-Experts}

A mixed GLM is a generalization of the MLR model \eqref{eq:MLR_model}, where the output function is not necessarily linear. Specifically, for some known function $\breve{q}: \reals^2 \to \reals$, we have 
\begin{align}
    Y_i
    &=\breve{q}( \langle X_i,\beta^{(1)}\rangle c_{i1}+\dots+\langle X_i,\beta^{(L)}\rangle c_{iL}\,  , \, \epsilon_i),  \ \ i \in [n]. \label{eq:MGLM_model}
\end{align}
As before, $\epsilon_i$ is a noise variable, and $c_{i1}, \ldots, c_{iL} \in \{ 0,1 \}$ are binary-valued latent variables such that 
$\sum_{l=1}^L c_{il}=1$, for $i \in [n]$.  The case of $L=1$ is the standard GLM which, with suitable choices for $\breve{q}$ and $\epsilon$, covers  a range of  statistical learning problems including logistic regression, phase retrieval, and one-bit compressed sensing. In all these settings, the mixed GLM model \eqref{eq:MGLM_model} allows the flexibility to account for unlabeled data coming from multiple sub-populations \citep{Kha07,Sed16}.

The mixture-of-experts model, introduced by \cite{Jac91,Jor94}, is a generalization of the mixed GLM, where the probability of selecting each regressor can depend on the feature vector.  In addition to the $L$ regressors $\beta^{(1)},\dots,\beta^{(L)}\in\mb{R}^p$, here we have $L$ gating parameters $w^{(1)},\dots,w^{(L)}\in\mb{R}^p$, using which the observations are generated as follows. For each $i \in [n]$: 
\begin{align}
 Y_i =   \tilde{q}(\langle X_i,\beta^{(l)}\rangle \, , \,  \epsilon_i) \     \ \text{ with probability }  \frac{\exp(\langle X_i,w^{(l)}\rangle)}{\sum_{l'=1}^L\exp(\langle X_i,w^{(l')}\rangle)} \ 
  \text{ for } l \in [L].
  \label{eq:MOE_def}
\end{align}
Here $\tilde{q}:\reals \to \reals$ is a known activation function and $\epsilon_i$ is a noise variable. Mixture-of-experts models and its variants have been widely studied in machine learning \citep{Yuk12,Hua12,Mak19,Mak20} and  applications such as  computer vision \citep{Gro17}, natural language processing \citep{Sha17}, and econometrics \citep{Hua13,Com16}.  
  
To see that the mixture-of-experts model is a special case of the matrix GLM in \eqref{eq:matrix-GLM}, we take the signal matrix to be $B = [\beta^{(1)}, \ldots, \beta^{(L)}, w^{(1)}, \ldots, w^{(L)}]$ and the  auxiliary matrix $\Psi \in \reals^{n \times 2}$ with rows $\Psi_i=(\psi_i,\epsilon_i)$, where $\psi_i\sim_{\iid}\text{Uniform}[0,1]$ for $i \in [n]$ and independent of $\{ \epsilon_i \}_{i \in [n]}$. Then the model   \eqref{eq:MOE_def} can be written as:
\begin{align}
    & Y_i
    =q(B^\top X_i,\Psi_i) \nonumber \\
    &=\sum_{l=1}^L\tilde{q}(\langle X_i,\beta^{(l)}\rangle,  \, \epsilon_i )\mathds{1}
    \bigg\{
    \sum_{l'=1}^{l-1}\frac{\exp(\langle X_i,w^{(l')}\rangle)}{\sum_{l^*=1}^L\exp(\langle X_i,w^{(l^*)}\rangle)}<\psi_i\leq\sum_{l'=1}^{l}\frac{\exp(\langle X_i,w^{(l')}\rangle)}{\sum_{l^*=1}^L\exp(\langle X_i,w^{(l^*)}\rangle)}
    \bigg\}, \label{eq:MOE_in_GLM_form}
\end{align}
where  $\mathds{1}\{ \cdot\}$ is the indicator function. 

\subsection{Approximate Message Passing}

The main contribution of this work is to design and analyze an  Approximate message passing (AMP)  algorithm  for estimation in the matrix GLM model \eqref{eq:matrix-GLM}. We then apply the algorithm to mixed linear regression, max-affine regression, and mixture-of-experts.

Approximate message passing (AMP) is a family of iterative algorithms which can be tailored to take advantage of structural information about the signals and the model, e.g., a known prior on the signal vector or on the proportion of observations that come from each signal. AMP algorithms were first proposed for the standard linear model \citep{Kab03,Don09,Bay11,Krz12}, but have since been applied to a range of statistical problems, including estimation in generalized linear models \citep{Ran11,Sch14,Bar19,Ma19,Sur19,Mai20,Mon21} and low-rank matrix estimation \citep{Des14,Fle18,Kab16,Les17,Mont21,Li23}. In all these settings,  under suitable model assumptions the performance of AMP in the high-dimensional limit is characterized by a succinct deterministic recursion called \emph{state evolution}. The state evolution characterization has been used to show that AMP achieves Bayes-optimal performance for some models \citep{Des14,Don13,Mont21,Bar19}, and a conjecture from statistical physics states that AMP is optimal among polynomial-time algorithms for  a wide range of statistical estimation problems. 

\subsection{Main Contributions} \label{subsec:main_contrib}

We propose an AMP algorithm for the matrix GLM \eqref{eq:matrix-GLM}, under the assumption that the features $\{ X_i \}_{i \in [n]}$ are i.i.d.~Gaussian. Our first technical contribution is a state evolution result for the AMP algorithm (Theorem \ref{thm:GAMP}), which gives a rigorous characterization of its performance in the high-dimensional limit as $n,p \to \infty$ with a fixed ratio $\delta = n/p$, for a constant $\delta >0$. This allows us to compute exact asymptotic formulas for performance measures such as the mean-squared error (MSE) and the normalized correlation between the signals and their estimates. The AMP algorithm uses a pair of `denoising' functions to produce updated signal estimates in each iteration.  The accuracy of these estimates can be tracked using a signal-to-noise ratio defined in terms of the state evolution parameters. Our second contribution (Proposition \ref{prop:optimal_fk}) is to derive an optimal choice of denoising functions that maximizes this signal-to-noise ratio. The optimal choice for one of the these functions depends on the prior on the signals, while the other depends only on the output function $q(\cdot,\cdot)$ in \eqref{eq:matrix-GLM}. 

In Section \ref{sec:sims}, we present numerical simulation results for mixed linear regression, max-affine regression, and mixture-of-experts. The case of max-affine regression requires special attention as the AMP derived for the matrix GLM cannot be directly applied. This is because the matrix GLM AMP and its state evolution analysis is derived assuming that the features are all i.i.d.~Gaussian. However, to write MAR as an instance of the matrix GLM, recall from \eqref{eq:MAR_model_concise} that we use the augmented features $X_i^{({\text{ma}})}= \begin{bmatrix} X_i \\ 1 \end{bmatrix}\in\mb{R}^{p+1}$,  $i \in [n]$, which are not i.i.d. Gaussian due to the last component being 1. We address this by using the original formulation of MAR in \eqref{eq:max-affine_model}, with the intercepts $b_1, \ldots, b_L$ treated as unknown model parameters. We estimate these intercepts via an expectation-maximization (EM) algorithm that uses the AMP iterates to approximate certain intractable quantities. 
This leads to a combined EM-AMP algorithm which is described in Section \ref{sec:max-affine_reg}.  For both mixed linear regression and  max-affine regression, the numerical results show that AMP  significantly outperforms other popular estimators (such as alternating minimization) in most parameter regimes. 

Though the algorithms and results in this paper focus on  estimating the signals $\beta^{(1)}, \ldots$
$\beta^{(L)}$,  they can be often be translated to estimating the latent variables as well. For example, in mixed linear regression, given signal estimates
$\widehat{\beta}^{(1)},\dots,\widehat{\beta}^{(L)}$, the labels can be estimated as $\hat{c}_i = \argmin_{l \in[L]} \,  (Y_i - \langle X_i, \widehat{\beta}^{(l)} \rangle)^2$, for $i \in [n]$.

A preliminary version of this paper was published in the proceedings of the 26th International Conference on Artificial Intelligence and Statistics (AISTATS) 2023 \citep{Tan23}. The focus of the preliminary version was largely on mixed linear regression. In the current paper, in addition to MLR, we provide results for max-affine regression and mixture-of experts, including the novel EM-AMP algorithm for MAR.

\paragraph{Technical Ideas.} The state evolution performance characterization in Theorem \ref{thm:GAMP} is proved using a change of variables that maps the proposed  algorithm to an abstract AMP recursion with matrix-valued iterates. A state evolution characterization for this abstract AMP was established by \cite{Jav13}; this result is translated via the change of variables to obtain the state evolution characterization for the proposed AMP. 

Our combined  EM-AMP algorithm for max-affine regression is inspired by the work of \cite{Vil13}, who used a similar approach for the problem of sparse linear regression with unknown parameters in the signal prior.

Though our AMP algorithm and its analysis  assume i.i.d.~Gaussian  features, we expect that they can be extended to a much broader class of i.i.d.~designs using the recent universality results of \cite{Wan22}. Another exciting direction for future work is to generalize the  AMP algorithm and its state evolution to mixed regression models with rotationally invariant design matrices. This can be done via a  reduction to an abstract AMP  recursion for rotationally invariant matrices, similar to the ones studied in \cite{Fan22} and \cite{Zho21}.

\subsection{Other related work}

\paragraph{Mixtures of linear and generalized linear models.}  
The special case of symmetric mixed linear regression where $\beta^{(1)} = - \beta^{(2)}$ has been studied in many recent works. We note that symmetric MLR is a version of the phase retrieval problem \citep{Net13,Can15,Fog16}.  \cite{Bal17} and \cite{Klu19} obtained statistical guarantees on the performance of the EM algorithm for a class of problems, including symmetric MLR. Variants of the EM algorithm for symmetric MLR in the high-dimensional setting (with sparse signals) were analyzed by \cite{Wan15},\cite{Yi15}, and \cite{Zhu17}. \cite{Fan18} obtained minimax lower bounds for a class of computationally feasible algorithms for symmetric MLR. 

\cite{Kon20} studied MLR as a canonical example of meta-learning. They consider the setting where the number of signals ($L$) is large,  and derive conditions under which a large number of signals with a few observations can compensate for the lack of signals with abundantly many observations. The case of MLR with sparse signals was studied by \cite{Kri19} and \cite{Pal21}, and  the gap between statistical and computational performance limits for sparse MLR was recently characterized by \cite{Arp23}. \cite{Pal22} studied the prediction error of MLR in the non-realizable setting, where no generative model is assumed for the data. 

The convergence rate of maximum-likelihood estimation for the parameters of a  mixed GLM was derived by \cite{Ho22}.  \cite{Cha21} analyzed the performance of a class of iterative algorithms (not including AMP) for mixed GLMs, providing  a sharp characterization of the  per-iteration error with sample-splitting in the regime $n \sim p\, \text{polylog}(p)$, assuming  a Gaussian design and a random initialization. Spectral estimators for mixed GLMs were studied in the recent work of \cite{Zha22}, which characterizes their asymptotic performance for Gaussian designs and independent signals.

\paragraph{Max-affine regression.}
 For the non-parametric convex regression model in \eqref{eq:convex_reg}, the least squares estimator is
 $   \hat{\varphi}^{(\text{ls})}\in
    \argmin_{\varphi}\sum_{i=1}^n(Y_i-\varphi(X_i))^2$,
where the minimization is over all convex functions $\varphi$. This least-squares estimator can be computed by solving a quadratic program. Theoretical properties of this estimator and algorithms to compute it were studied by \cite{Sei11}, \cite{Lim12} and  \cite{Maz19}.
For the MAR model \eqref{eq:max-affine_model}, several approaches for signal estimation have been proposed,  including alternating minimization \citep{Mag09, Gho22}, convex adaptive partitioning \citep{Han13}, and adaptive max-affine partitioning \citep{Bal16}. Among them, theoretical guarantees have been established only  for alternating minimization; these guarantees are in the regime where $n$ is at least of order $p\log(n/p)$ \citep{Gho22}. In contrast, in this paper we consider the high-dimensional regime where $n$ is proportional to $p$ as $n \to \infty$.

\section{Preliminaries} \label{sec:prelim}

\paragraph{Notation.} All vectors (even rows of matrices) are treated as column vectors unless otherwise stated. Matrices are denoted by upper case letters, and given a matrix $A$, we write $A_i$ for its $i$th row. The notation  $M \succeq 0$ denotes that the square matrix $M$ is positive semidefinite. We write $I_p $  for the  $p \times p$ identity matrix.
For $r\in[1,\infty)$, we write $\|x\|_r$ for the $\ell_r$-norm of $x=(x_1,\dots,x_n)\in\mb{R}^n$, so that $\|x\|_r=\big(\sum_{i=1}^n|x_i|^r\big)^{1/r}$.  Given random variables $U,V$, we write $U \stackrel{d}{=} V$ to denote equality in distribution. 

\paragraph{Complete convergence.} The asymptotic results in this paper are stated in terms of \textit{complete convergence} \citep{Hsu47}, \citep[Sec. 1.1]{Fen21}. This is a stronger mode of stochastic convergence than almost sure convergence, and is denoted  using the symbol $\stackrel{c}{\rightarrow}$.  Let $\{X_n\}$ be a sequence of random elements taking values in a Euclidean space $E$. We say that $X_n$ converges completely to a deterministic limit $x\in E$, and write $X_n\stackrel{c}{\rightarrow}x$, if $Y_n\rightarrow x$ almost surely for any sequence of $E$-valued random elements $\{Y_n\}$ with $Y_n\stackrel{d}{=}X_n$  for all $n$.
 
\paragraph{Wasserstein distances.} For $D \in \mathbb{N}$, let $\mathcal{P}_D(r)$ be the set of all Borel probability measures on $\reals^D$ with finite $r$th-moment. That is, any $P \in \mathcal{P}_D(r)$ satisfies $\int_{\mb{R}^D}\|x\|_2^rdP(x)<\infty$. For $P, Q \in\mathcal{P}_D(r)$, the \textit{r-Wasserstein distance} between $P$ and $Q$ is defined by $d_r(P,Q)=\inf_{(X,Y)}\E[\|X-Y\|_2^r]^{1/r}$, where the infimum is taken over all pairs of random vectors $(X,Y)$ defined on a common probability space with $X\sim P$ and $Y\sim Q$.

\subsection{Model assumptions} \label{subsec:prelim_model} In  the model \eqref{eq:matrix-GLM}, each feature vector $X_i \in \reals^p$ is assumed to have independent Gaussian entries with zero mean and variance $1/n$, i.e., $X_i \sim_{\iid} \normal(0, I_p/n)$. The $n \times p$ design matrix $X$ is formed by stacking the sensing vectors $X_1, \ldots, X_n$, i.e., $X = [X_1, \ldots, X_n]^{\top}$. Similarly, the auxiliary variable matrix $\Psi \in \reals^{n \times L_{\Psi}}$ is defined as $\Psi =[\Psi_1, \ldots, \Psi_n]^\top$.
The  design matrix $X$ is independent of both the signal matrix $B = (\beta^{(1)},\dots,\beta^{(L)}) \in \reals^{p \times L}$ and the auxiliary variable matrix $\Psi \in \reals^{n \times L_{\Psi}}$. 

As $p\to\infty$, we assume that $n/p = \delta$, for some constant $\delta > 0$. As $p \to \infty$, the empirical distributions of the  rows of the signal matrix and the auxiliary variable matrix are assumed to converge in Wasserstein distance to well-defined limits. More precisely, for some $r \in [2, \infty)$,  there exist random variables $\bar{B}\sim P_{\bar{B}}$ (where $\bar{B}\in\mb{R}^{L}$) and $\bar{\Psi}\sim P_{\bar{\Psi}}$ (where $\bar{\Psi}\in\mb{R}^{L_\Psi}$) with $\E[\bar{B}^\top\bar{B}]>0$ and $\E\big[\sum_{l=1}^{L}|\bar{B}_l|^r\big],\E\big[\sum_{l=1}^{L_{\Psi}}|\bar{\Psi}_l|^r\big]<\infty$, such that writing $\nu_p(B)$ and $\nu_n(\Psi)$ for the empirical distributions of the rows of $B$ and $\Psi$ respectively, we have $d_r(\nu_p(B),P_{\bar{B}})\stackrel{c}{\rightarrow}0$ and $d_r(\nu_n(\Psi),P_{\bar{\Psi}})\stackrel{c}{\rightarrow}0$.

\section{AMP for the Matrix GLM} \label{sec:AMP_main}

Consider the task of estimating the signal matrix $B$ given $\{ X_i,Y_i \}_{i \in [n]}$, generated according to \eqref{eq:matrix-GLM}.

\paragraph{Algorithm.} In each iteration $k \ge 1$, the AMP algorithm iteratively produces estimates $\hB^k$ and $\Theta^k$ of $B\in\mb{R}^{p\times L}$ and $\Theta:=XB\in\mb{R}^{n\times L}$, respectively. Starting with an initializer $\hB^0\in\mb{R}^{p\times L}$ and defining $\hR^{-1} := 0\in\mb{R}^{n\times L}$, for $k \ge 0$ we compute:
\begin{align}
\begin{split}
        & \Theta^k=X\hB^k-\hR^{k-1}(F^k)^\top, \quad \hR^k=g_k(\Theta^k,Y), \\ 
        & B^{k+1}=X^\top\hR^k-\hB^k (C^k)^\top, \quad \hB^{k+1}=f_{k+1}(B^{k+1}).
\end{split}
\label{eq:GAMP}
\end{align}
Here  the functions $g_k:\mb{R}^{L}\times\mb{R}^{\Lout}\rightarrow\mb{R}^{L}$ and  $f_{k+1}:\mb{R}^{L}\rightarrow\mb{R}^{L}$ act row-wise on their matrix inputs, and  the matrices  $C^k, F^{k+1} \in \reals^{L \times L}$ are defined as
\begin{align}
    C^k= \frac{1}{n}\sum_{i=1}^ng_k'(\Theta_i^k,Y_i), \   \ 
    F^{k+1}=\frac{1}{n}\sum_{j=1}^pf_{k+1}'(B_j^{k+1}), 
    \label{eq:CkF_k1_def}
\end{align}
where $g_k', f_{k+1}' \in \reals^{L \times L}$ denote the Jacobians of $g_k, f_{k+1}$, respectively, with respect to their first arguments.   We note that the time complexity of each iteration of \eqref{eq:GAMP} is $\mathcal{O}(npL)$.

\paragraph{State evolution.} The ``memory'' terms $-\hR^{k-1}(F^k)^\top$ and $-\hB^k(C^k)^\top$ in 
\eqref{eq:GAMP} play a crucial role in debiasing the iterates 
$\Theta^k$ and $B^{k+1}$, ensuring that their  joint empirical distributions are accurately captured by state evolution in the high-dimensional limit.  Theorem \ref{thm:GAMP} below shows that for each $k \ge 1$,  the empirical distribution of the rows of $B^k$ converges to the distribution of $\Mu^{k}_B\bar{B}+G^k_B \in \reals^L$, where
$G^k_B\sim \normal(0,\Tau^k_B)$ is independent of $\bar{B}$,  the random variable representing the limiting distribution of the rows of the signal matrix $B$. The deterministic matrices $\Mu^{k}_B, \Tau^k_B \in \reals^{L\times L}$ are recursively defined below. The result implies that the empirical distribution of the rows of $\hB^k$ converges to the distribution of $f_k(\Mu^{k}_B\bar{B}+G^k_B)$. Thus $f_k$ can be viewed as a denoising function that can be tailored to take advantage of  the prior on $\bar{B}$. Theorem \ref{thm:GAMP} also shows that the empirical distribution of the rows of $\Theta^k$ converges to the distribution of $\Mu_\Theta^k Z+G^k_\Theta \in \reals^L$, where $Z \sim \normal(0,  \frac{1}{\delta}\E[\bar{B}\bar{B}^\top])$ and $G^k_\Theta\sim \normal(0,\Tau^k_\Theta)$ are independent. 

We now describe the state evolution recursion defining the matrices $\Mu^{k}_B, \Tau^k_B, \Mu_\Theta^k, \Tau^k_\Theta \in \reals^{L\times L}$. Recalling  that the observation $Y$  is generated via the function $q$ according to \eqref{eq:matrix-GLM}, it is convenient to rewrite $g_k$ in \eqref{eq:GAMP} in terms of another function $h_k: \reals^L \times \reals^L \times \reals^{L_{\Psi}} \to \reals^L$ defined as:
\begin{align}
    h_k(z, u, v) := g_k(u, \, q(z, v)).
    \label{eq:hk_def}
\end{align}
Then,  for $k \ge 0$,  given $\Sigma^k \in \reals^{2L \times 2L}$, take $\begin{bmatrix} Z \\ Z^k \end{bmatrix} \sim \normal(0,\Sigma^k)$ to be independent of $\bar{\Psi}\sim P_{\bar{\Psi}}$ and compute:
\begin{align}
    &\Mu^{k+1}_{B} =\E[\partial_Z h_k(Z,Z^k,\bar{\Psi})], 
    \label{eq:SE_Mk1B} \\ 
    & \Tau^{k+1}_{B} =\E[h_k(Z,Z^k,\bar{\Psi})h_k(Z,Z^k,\bar{\Psi})^\top], \label{eq:SE_Tk1B} \\
    &\Sigma^{k+1} =
    \begin{bmatrix}
\Sigma_{(11)}^{k+1} & \Sigma_{(12)}^{k+1} \\ \Sigma_{(21)}^{k+1} & \Sigma_{(22)}^{k+1}
    \end{bmatrix},
    \label{eq:SE_Sigk1}
\end{align}
where the four $L \times L$ matrices constituting $\Sigma^{k+1} \in \reals^{2L \times 2L}$ are given by:
\begin{equation}
    \begin{split}
          &  \Sigma_{(11)}^{k+1} = \frac{1}{\delta} \E\big[\bar{B}\bar{B}^\top\big],  
      \\
        & \Sigma_{(12)}^{k+1} = \left( \Sigma_{(21)}^{k+1} \right)^{\top} = \frac{1}{\delta} \E\big[\bar{B}f_{k+1}(\Mu^{k+1}_{B}\bar{B}+G^{k+1}_B)^\top\big], \\
        & \Sigma_{(22)}^{k+1} =
        \frac{1}{\delta}  \E\big[f_{k+1}(\Mu^{k+1}_{B}\bar{B}+G^{k+1}_B)f_{k+1}(\Mu^{k+1}_{B}\bar{B}+G^{k+1}_B)^\top\big].
         \label{eq:Sig_comps}
    \end{split}
\end{equation}
Here we take  $G^{k+1}_B\sim \normal(0,\Tau^{k+1}_{B})$ to be independent of $\bar{B}\sim P_{\bar{B}}$. Note that $\partial_Z h_k$ denotes the partial derivative (Jacobian) of $h_k$ with respect to its first argument $Z\in\mb{R}^{L}$, so it is an $L\times L$ matrix.  The state evolution recursion \eqref{eq:SE_Mk1B}-\eqref{eq:SE_Sigk1} is initialized with $\Sigma^0 \in \reals^{2L \times 2L}$ defined below in \eqref{eq:Sig0_def}.

For $\begin{bmatrix} Z \\ Z^k \end{bmatrix} \sim \normal(0,\Sigma^k)$, using standard properties of Gaussian random vectors, we have 
\begin{equation}
    (Z,Z^k,\bar{\Psi})\stackrel{d}{=}(Z,\Mu^{k}_{\Theta}Z+G^k_\Theta,\bar{\Psi}),
    \label{eq:ZZk_joint}
\end{equation}
where $G_\Theta^k\sim \normal(0,\Tau_\Theta^k)$,
\begin{align}
    & \Mu^{k}_{\Theta}  =\Sigma_{(21)}^k \big(\Sigma_{(11)}^k \big)^{-1}, \label{eq:SE_Mk1Th} \\
    & \Tau^{k}_{\Theta} = \Sigma_{(22)}^k \, - \, \Sigma_{(21)}^k \big(\Sigma_{(11)}^k \big)^{-1} \,  \Sigma_{(12)}^k. \label{eq:SE_TkTh}
\end{align}

\paragraph{Main result.} 
We begin with two assumptions required for the main result.  The first is on the AMP initializer $\hB^0 \in \reals^{p \times L}$, and the second is on the functions $g_k, f_{k+1}$ used to define the AMP in \eqref{eq:GAMP}.

\textbf{(A1)}
There exists $\Sigma^0 \in \reals^{2L \times 2L}$ and $c_0 \in \reals$ such that as $n, p \to \infty$ (with $n/p \to \delta$), we have
\begin{align}
         &  \frac{1}{n}
        \begin{bmatrix}
            B^\top B & B^\top\widehat{B}^0 \\
            (\widehat{B}^0)^\top B & (\widehat{B}^0)^\top\widehat{B}^0
        \end{bmatrix}
        \stackrel{c}{\rightarrow}\Sigma^0,  \label{eq:Sig0_def} \\
        &  \frac{1}{p}\sum_{j=1}^p\sum_{l=1}^{L}|\widehat{B}_{jl}^0|^r
        \stackrel{c}{\rightarrow} c_0. 
\end{align}
Here $r \in [2, \infty)$ is the same as that used for the assumptions on the signal matrix at the end of Section \ref{subsec:prelim_model}. Furthermore, there exists a Lipschitz $F_0:\mb{R}^{L}\rightarrow\mb{R}^{L}$ such that $\frac{1}{p}(\widehat{B}^0)^\top\phi(B)\stackrel{c}{\rightarrow}\E[F_0(\bar{B})\phi(\bar{B})^\top]$ and $\Sigma_{(22)}^0 - \E[F_0(\bar{B})F_0(\bar{B})^\top]$ is positive semi-definite for all Lipschitz $\phi:\mb{R}^{L}\rightarrow\mb{R}^{L}$.

\textbf{(A2)} For $k \ge 0$, the function $f_{k+1}$ is non-constant and Lipschitz on $\mb{R}^{L}$, and $h_k$ defined in \eqref{eq:hk_def} is Lipschitz on $\mb{R}^{2L + L_{\Psi}}$ with $P_{\bar{\Psi}}(\{v:(z,u)\rightarrow h_k(z,u,v)\text{ is a non-constant}\})>0$.  Furthermore, $f_{k+1}'$ is continuous Lebesgue almost everywhere, and writing $\mathcal{D}_k\subseteq\mb{R}^{L +L_{\Psi}}$ for the set of discontinuities of $g_k'$, we have $\mb{P}[(Z^k,\bar{Y})\in \mathcal{D}_k]=0$.

Assumptions \textbf{(A1)} and \textbf{(A2)} are similar to those required for AMP initialization in (non-matrix) generalized linear models \citep[Section 4]{Fen21}. Moreover, \textbf{(A1)} is implied by the assumptions on the signal matrix if an initialization $\widehat{B}^0$ is chosen to be a scaled version of the all ones matrix.

The result is stated in terms of \emph{pseudo-Lipschitz} test functions. Let $\text{PL}_m(r, C)$ be the set of functions $\phi:\mb{R}^m\rightarrow\mb{R}$ such that $|\phi(x)-\phi(y)|\leq C (1+\|x\|_2^{r-1}+\|y\|_2^{r-1})\|x-y\|_2$ for all $x,y\in\mb{R}^m$. A function $\phi \in \text{PL}_m(r, C)$ is called pseudo-Lipschitz of order $r$. 

\begin{theorem} \label{thm:GAMP}
Consider the AMP in \eqref{eq:GAMP} for the matrix GLM model in \eqref{eq:matrix-GLM}. Suppose that the model assumptions in Section \ref{subsec:prelim_model} as well as  \textbf{(A1)} and \textbf{(A2)} are satisfied, and that $\Tau^{1}_{B}$ is positive definite. Then for each $k \ge 0$, we have
\begin{align}
    &\sup_{\phi\in\textup{PL}_{2L}(r,1)}\Big|\frac{1}{p}\sum_{j=1}^{p}\phi(B_j^{k+1},B_j)-\E[\phi(\Mu^{k+1}_{B}\bar{B}+G^{k+1}_B,\bar{B})]\Big|\stackrel{c}{\rightarrow}0 ,\label{eq:matrix-GAMP_SE_result1} \\
    &\sup_{\phi\in\textup{PL}_{2L + L_{\Psi}}(r,1)}\Big|\frac{1}{n}\sum_{i=1}^{n}\phi(\Theta_i^{k},\Theta_i,\Psi_i)-\E[\phi(\Mu^{k}_{\Theta} \, Z+G^{k}_\Theta,Z,\bar{\Psi})]\Big|\stackrel{c}{\rightarrow}0, \label{eq:matrix-GAMP_SE_result2}
\end{align}
as $n,p\rightarrow\infty$ with $n/p\rightarrow\delta$, where $\Theta_i=B^\top X_i$ for $1\leq i\leq n$. In the above, $G^{k+1}_B\sim \normal(0,\Tau^{k+1}_{B})$ is independent of $\bar{B}$, and $G^{k}_\Theta\sim \normal(0,\Tau^{k}_{\Theta})$ is independent of $(Z,\bar{\Psi})$.
\end{theorem}
The proof of the theorem is given in Section \ref{subsec:proof_GAMP}. The result \eqref{eq:matrix-GAMP_SE_result1} is equivalent to the statement that the joint empirical distributions of the rows of $(B^{k+1}, \, B)$ converges completely in $r$-Wasserstein distance to the joint distribution of $(\Mu^{k+1}_{B}\bar{B}+G^{k+1}_B, \, \bar{B})$; see \cite[Corollary 7.21]{Fen21}. An analogous statement holds for \eqref{eq:matrix-GAMP_SE_result2}.

\paragraph{Performance measures.} Theorem \ref{thm:GAMP} allows us to compute the limiting values of performance measures such as the mean-squared error (MSE), and the normalized correlation between each signal and its AMP estimate.  For $k \ge 1$, writing $\hbeta^{(\ell), k}$ for the  $\ell$th column of the AMP iterate $\hB^{k}$, we have  $\hB^{k}  = \begin{bmatrix}
    \hbeta^{(1), k}, \ldots, \hbeta^{(L), k}
\end{bmatrix}$. Note that  $\hbeta^{(\ell), k}$ is the estimate of the signal $\beta^{(\ell)}$ after $k$ iterations,  and define the shorthand $\bar{B}^{k} := \Mu^{k}_{B}\bar{B}+G_{k}^B$.
Then Theorem \ref{thm:GAMP} implies that the  normalized squared correlation between each signal and its AMP estimate after $k$ iterations converges as:
\begin{align}
     \frac{\langle\hat{\beta}^{(\ell),k},\beta^{(\ell)}\rangle^2}{\|\hat{\beta}^{(\ell),k}\|_2^2\|\beta^{(\ell)}\|_2^2}
    \stackrel{c}{\rightarrow}
        \frac{(\E[f_{k, \ell}(\bar{B}^{k}) \bar{B}_\ell ])^2}{\E[ f_{k, \ell}(\bar{B}^{k})^2]\E[\bar{B}_\ell^2]}, \ \ \ell \in [L].
        \label{eq:norm_sq_corr}
\end{align}
Here $f_{k, \ell}$ is the $\ell$th component of the function $f_k: \reals^L \to  \reals^L$, and $\bar{B}_\ell$ is the $\ell$th component of $\bar{B} \in \reals^L$. Similarly, the MSE of the AMP estimate after $k$ iterations converges as:
\begin{align}
    \frac{\|  \beta^{(\ell)} - \hat{\beta}^{(\ell),k}  \|_2^2}{p} \stackrel{c}{\rightarrow} \E\Big[ \big(\bar{B}_\ell - f_{k, \ell}(\bar{B}^{k}) \big)^2 \Big], \ \ \ell \in [L].
\end{align}

\subsection{Choosing the Functions of AMP}

Recalling that the empirical distributions of the rows of $\Theta^k$ and $B^{k+1}$  converge to the laws of $\Mu^{k}_{\Theta} \, Z+G^{k}_\Theta$ and $\Mu^{k+1}_{B}\bar{B}+G^{k+1}_B$, respectively, we  define the random vectors:  
\begin{equation}
\begin{split}
        & \tZ^k := Z \, + \, \left( \Mu^{k}_{\Theta}\right)^{-1} G^{k}_\Theta, \\
        & \tB^{k+1} := \bar{B} \, + \, \left( \Mu^{k+1}_{B}\right)^{-1} G^{k+1}_B.
\end{split}
\label{eq:eff_noise_cov}
\end{equation}
(If the inverse doesn't exist we premultiply by the pseudoinverse.)
Since $G^{k+1}_B\sim \normal(0,\Tau^{k+1}_{B})$ and $G^{k}_\Theta\sim \normal(0,\Tau^{k}_{\Theta})$, the effective noise covariance matrices are:
\begin{align}
\begin{split}
    & \Cov(\tZ^k - Z)  = \left(\Mu^{k}_{\Theta}\right)^{-1} \Tau^{k}_{\Theta}
    \left(\big(\Mu^{k}_{\Theta}\big)^{-1}\right)^\top
     =: N_{\Theta}^k, \\
    & \Cov(\tB^{k+1} - \bar{B}) = \left( \Mu^{k+1}_{B}\right)^{-1}
    \Tau^{k+1}_{B} \left( \big(\Mu^{k+1}_{B}\big)^{-1} \right)^\top =: N_{B}^{k+1}.
\end{split}
\label{eq:eff_cov_def}
\end{align}
From \eqref{eq:SE_TkTh}, we observe that $\Mu^{k}_{\Theta}, \Tau^{k}_{\Theta}$ are both determined by $\Sigma^k$, which in turn is determined by the choice of $f_k$  (from \eqref{eq:Sig_comps}). Similarly, from \eqref{eq:hk_def}--\eqref{eq:SE_Mk1B},  
$\Mu^{k+1}_{B}$, $\Tau^{k+1}_{B}$ are determined by $g_k$. A natural objective is to choose $f_k$ and $g_k$ to minimize the trace of the effective noise covariance matrices $N_{\Theta}^k$ and $N_{B}^{k+1}$ in \eqref{eq:eff_cov_def}.
We can interpret the quantities   $\Tr(N_{\Theta}^k)$ and $\Tr(N_{B}^{k+1})$  as the effective noise variances for estimating $Z, \bar{B}$ from $\tZ^k, \tB^{k+1}$, respectively.
In the special case where there is only one signal, minimizing these effective noise variances is equivalent to maximizing  the scalar signal-to-noise ratios $(\Mu^{k}_{\Theta})^2/\Tau^{k}_{\Theta}$ and $(\Mu^{k+1}_B)^2/\Tau^{k+1}_B$, respectively, which is achieved by the Bayes-optimal AMP  for generalized linear models \citep{Ran11, Fen21}.    
    
Assuming that the signal prior $P_{\bar{B}}$ and the distribution of auxiliary variables $P_{\Psi}$ are known, the following proposition gives optimal choices for $f_k, g_k$.
\begin{proposition} Let $k \ge 1$.  Then:

1) Given $\Mu^{k}_{B}$, $\Tau^{k}_{B}$, the quantity $\Tr(N_{\Theta}^k)$ is minimized when $f_k = f_k^*$, where
\begin{equation}
    f_k^*(s)=\E[\bar{B}\mid \Mu^k_B\bar{B}+G_B^k = s], \label{eq:fk_opt_def}
\end{equation}
where $G_B^k \sim \normal(0, \Tau^{k}_{B})$ and $\bar{B} \sim P_{\bar{B}}$ are independent. 

2) Given $\Mu^{k}_{\Theta}$, $\Tau^{k}_{\Theta}$, the quantity $\Tr(N_B^{k+1})$ is minimized when $g_k =g_k^*$, where 
\begin{align}
    g_k^*(u, \, y)  = & \Cov[Z \mid Z^k=u]^{-1}\big(\E[Z \mid Z^k=u, \bar{Y}=y]  - \, \E[Z \mid Z^k=u] \big). \label{eq:gk_opt_def}
\end{align}
Here $\begin{bmatrix} Z \\  Z^k \end{bmatrix} \sim \normal(0,\Sigma^k)$ and $\bar{Y}=q(Z, \bar{\Psi})$, with $\bar{\Psi} \sim P_{\bar{\Psi}}$ independent of $Z$.
\label{prop:optimal_fk}
\end{proposition}

The proof is given in Section \ref{subsec:prop_proof}. 

\section{Numerical Simulations} \label{sec:sims}

In this section, we present numerical results for mixed linear regression (Eq. \eqref{eq:MLR_model}),  max-affine regression (Eq. \eqref{eq:max-affine_model}), and mixture-of experts (Eq. \eqref{eq:MOE_def}). For MLR and MAR, we compare the performance of AMP with other popular estimators.

\subsection{Mixed Linear Regression (Two Signals)}

Consider the MLR model \eqref{eq:MLR_model} with two signals, where 
\begin{align}
    Y_i
    &=\langle X_i,\beta^{(1)}\rangle c_{i} \, + \, \langle X_i,\beta^{(2)}\rangle (1-c_i) \, + \, \epsilon_i, \ \ i \in [n].
    \label{eq:MLR_two_components}
\end{align}
We take $c_i \sim_{\iid} \text{Bernoulli}(\alpha)$ for $\alpha \in (0,1)$, $\epsilon_i \sim_{\iid} \normal(0, \sigma^2)$, and $X_i \sim_{\iid} \normal(0, I_p/n)$, for $i \in [n]$. We set the signal dimension $p=500$ and vary the value of $n$ in our experiments.

The AMP algorithm in \eqref{eq:GAMP} is implemented with $g_k =g_k^*$, the optimal choice given by \eqref{eq:gk_opt_def}. For the function $f_k$, we use the Bayes-optimal $f_k^*$ in \eqref{eq:fk_opt_def} unless stated otherwise. In Appendix \ref{appen:MLR_implementation},  we provide the implementation details, and show how the functions $f_k, g_k$ and  their derivatives can approximated  from the data. 

The performance in all the plots is measured via the normalized squared correlation between the AMP estimate and the signal (see \eqref{eq:norm_sq_corr}). Each point on the plots is obtained from 10 independent runs, where in each run, AMP is executed for 10 iterations. We report the average and error bars at 1 standard deviation of the final iteration.

\paragraph{Gaussian prior.} In Figures \ref{fig:diff_noise}, \ref{fig:diff_corr}, and \ref{fig:diff_prop}, we set the Bernoulli parameter $\alpha=0.7$ and choose the two signals to be jointly Gaussian, with their entries generated as
\begin{align}
    (\beta^{(1)}_j, \beta^{(2)}_j)\sim_{\iid} \normal\left( 
    \begin{bmatrix}
    0 \\
    0
    \end{bmatrix},
    \begin{bmatrix}
    1 & \rho \\
    \rho & 1
    \end{bmatrix}
    \right), \quad j \in [p].
    \label{eq:equal_mean_prior_and_init}
\end{align}
The initializer $\hB^0 \in \reals^{p \times 2}$ is chosen randomly according to the same distribution, independently of the signal.

\begin{figure}[t]
\centering
\subfloat[$\beta^{(1)}$\label{fig:beta1_diff_noise}]{\includegraphics[width=.45\columnwidth]{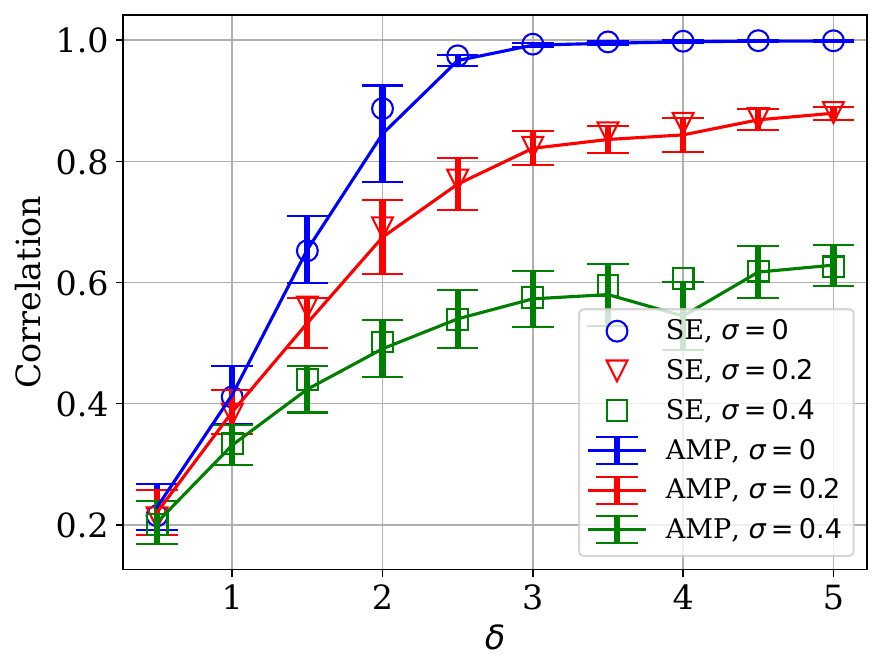}}
\subfloat[$\beta^{(2)}$ \label{fig:beta2_diff_noise}]{\includegraphics[width=.45\columnwidth]{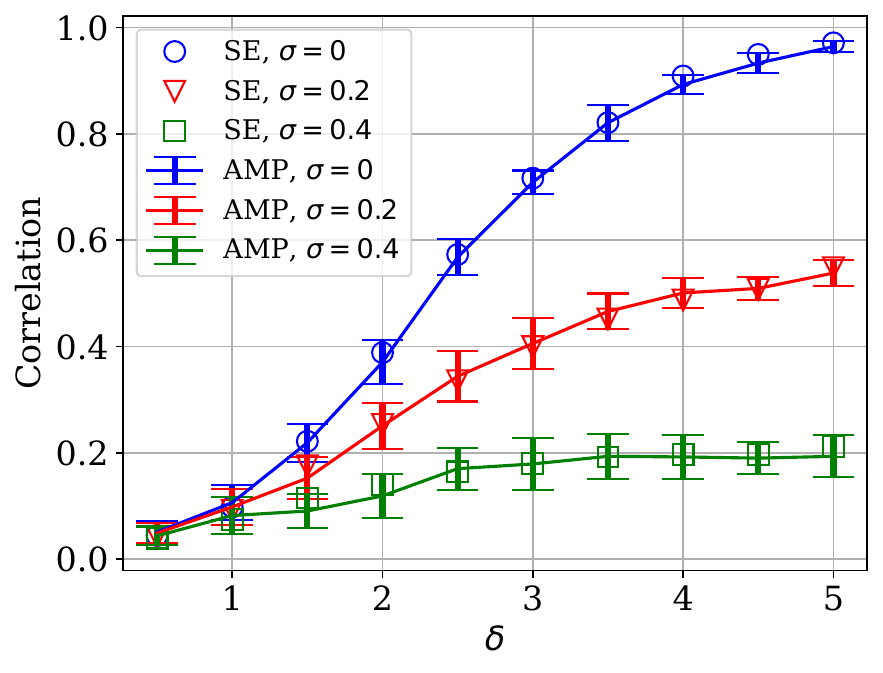}}
\caption{MLR, Gaussian prior with $\rho=0$: normalized  squared correlation vs. $\delta$ for various noise levels $\sigma$, with $\alpha=0.7$.}
\label{fig:diff_noise}
\end{figure}
\begin{figure}[t]
\centering
\subfloat[$\beta^{(1)}$\label{fig:beta1_diff_corr}]{\includegraphics[width=.45\columnwidth]{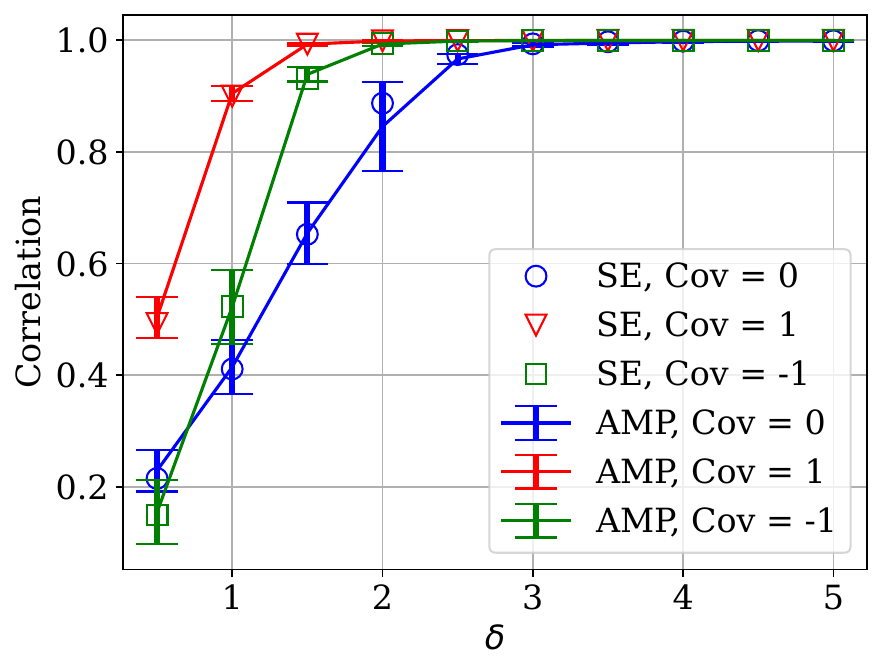}}
\subfloat[$\beta^{(2)}$\label{fig:beta2_diff_corr}]{\includegraphics[width=.45\columnwidth]{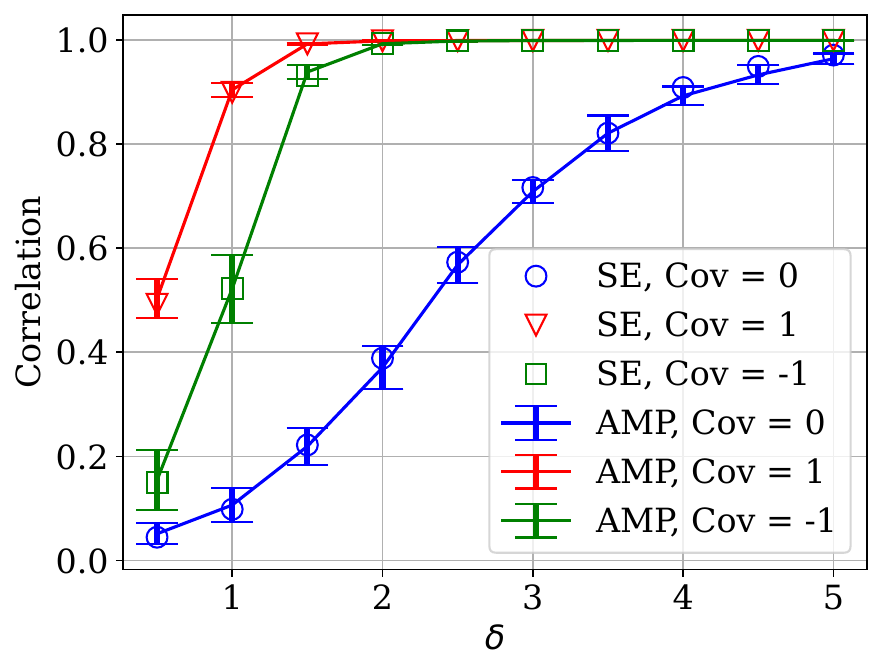}}
\caption{MLR, Gaussian prior with different values of signal covariance $\rho$: Normalized squared correlation vs. $\delta$, with $\alpha=0.7$, $\sigma=0$.}
\label{fig:diff_corr}
\end{figure}

Figure \ref{fig:diff_noise} shows the performance of AMP for independent signals ($\rho=0$). The normalized squared correlation is plotted as a function of the sampling ratio $\delta=n/p$, for different noise levels $\sigma$. The state evolution predictions closely match the performance of AMP for practical values of $n,p$, validating the result of Theorem \ref{thm:GAMP}. As expected, the correlation improves with increasing $\delta$ and degrades with increasing $\sigma$. The performance for $\beta^{(1)}$ is  better  than for  $\beta^{(2)}$ as $70\%$ of the observations come from $\beta^{(1)}$.  Figure \ref{fig:diff_corr} plots the performance as a function of  $\delta$ for signal correlation $\rho \in \{ 0,1,-1 \}$, with $\sigma=0$ (noiseless). When $\rho=1$, both signals are identical and the problem reduces to standard linear regression. When $\rho=-1$, we have $\beta^{(1)} = - \beta^{(2)} = \beta$, so there is still effectively only one signal vector. However, the  $\rho=-1$ case is harder than $\rho=1$  since each measurement is unlabelled and could come from either $\beta$ or $-\beta$ (with probabilities $0.7$ and $0.3$, respectively).  We note that the case of $\rho=-1$ and $\alpha=0.5$ is the phase retrieval problem, for which AMP algorithms have been studied in a number of works, e.g., \citep{Sch14,Ma19}. AMP needs to be initialized carefully in this setting since a random initialization independent of the signal leads to  state evolution predicting zero correlation between the signal and the AMP iterates \citep{Ma18,Mon21}. 

\begin{figure}[t]
\centering
\subfloat[$\beta^{(1)}$\label{fig:beta1_diff_prop}]{\includegraphics[width=.45\columnwidth]{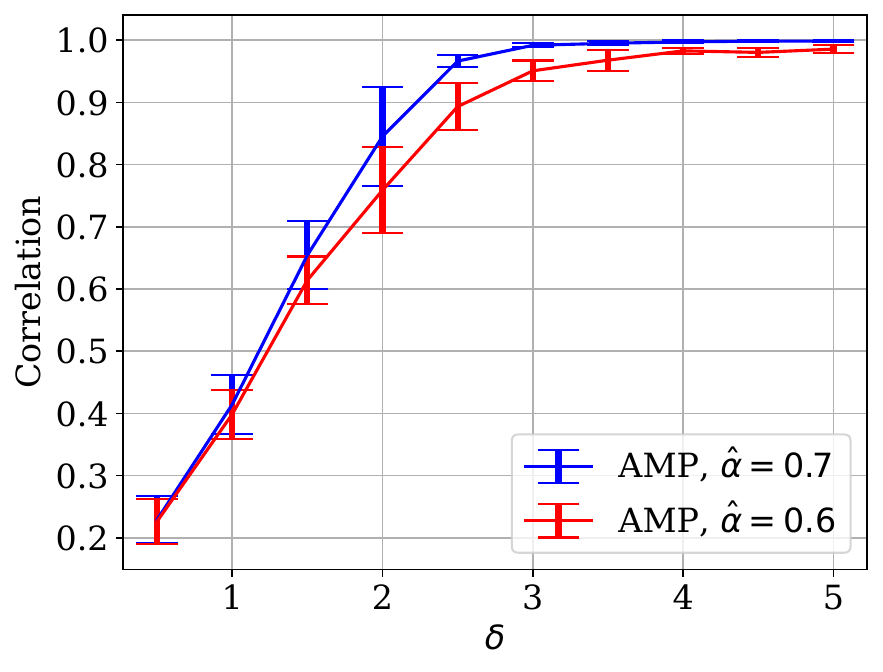}}
\subfloat[$\beta^{(2)}$\label{fig:beta2_diff_prop}]{\includegraphics[width=.45\columnwidth]{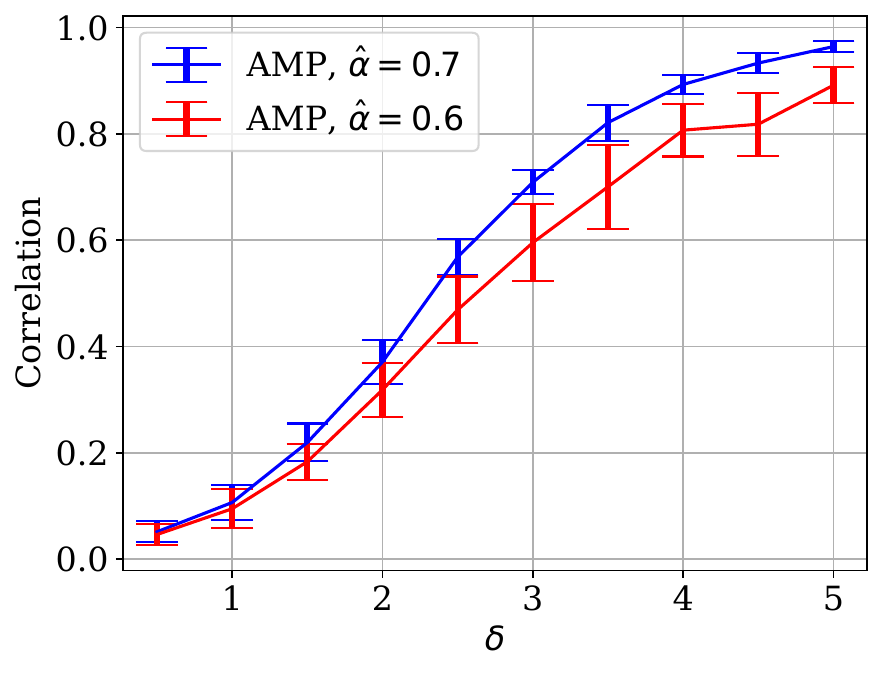}}
\caption{MLR, Gaussian prior with $\rho=0$ and different values of estimated proportion $\hat{\alpha}$: Normalized squared correlation vs. $\delta$, with true $\alpha=0.7$, $\sigma=0$.}
\label{fig:diff_prop}
\end{figure}

In practical applications, we may not know the exact proportion of observations that come  that come from the first signal.  Figure \ref{fig:diff_prop} shows the performance when AMP is run assuming a proportion parameter $\hat{\alpha}=0.6$ which is different from the true value $\alpha =0.7$.  The functions $f_k^*, g_k^*$ defining the AMP depend on $\alpha$, hence replacing $\alpha$ with $\hat{\alpha}$ in these functions is effectively running AMP with a different (sub-optimal) choice of denoising functions.

\paragraph{Sparse prior.}  We  next consider a sparse prior for each of the two signals, with their entries generated as
\begin{align}
     \beta^{(1)}_j, \beta^{(2)}_j
    \sim_\iid
    \frac{\varepsilon}{2}\delta_{+1}+(1-\varepsilon)\delta_{0}+\frac{\varepsilon}{2}\delta_{-1}, \ \ j \in [p].
    \label{eq:sparse_disc_prior}
\end{align}
Here $\delta_{(\cdot)}$ denotes the Dirac delta function, and we note that the two signals are independent. The initializer is generated randomly from the same prior, independently of the signals.  We investigate the performance of AMP with two choices of denoising function: the Bayes-optimal denoising function (defined in  \eqref{eq:fk_opt_def}) and the soft-thresholding denoising function (defined in \eqref{eq:ST-def}-\eqref{eq:ST_denoiser} below). For the case of standard linear regression, the soft-thresholding function is a popular choice of denoiser for AMP when  the signal is known to be sparse, 
but the exact sparsity level and the distribution of the non-zero coefficients are not known \citep{Mon12}. AMP with soft-thresholding denoising is also closely related to LASSO, which is widely used for sparse linear regression \citep{Bay11Lasso}. 

We evaluate the two denoisers 
in   Figures \ref{fig:sparse_prior_bayes_fk} and \ref{fig:sparse_prior_ST_fk}, respectively, by plotting heatmaps showing the performance for various values of the pair $(\delta, \varepsilon)$. For each point in the heatmap, we take the minimum of the mean normalized squared correlation of the two estimates with the respective signals. This is obtained by executing 10 runs of AMP using the desired $f_k$ function with 10 iterations per run (i.e., $k=1,\dots,10$), and taking the average of the 10 correlations (from the 10 runs) at the final iteration. 

For the Bayes-optimal denoiser $f_k$, the heatmaps are shown in Figure \ref{fig:sparse_prior_bayes_fk}. We have the following observations:
\begin{itemize}
    \item Performance is better for $\alpha=0.6$ compared to $\alpha=0.7$. This is because we are taking the minimum between the two correlations, and when $\alpha$ is larger ($\alpha=0.7$), there is less data available for the  group with fewer observations (for a given  $(\varepsilon, \delta))$.
    \item For a given $\delta$, performance is generally better  at $\varepsilon$ closer to $0$ or $1$. This is because at $\varepsilon=0.1$, most of the signal entries are 0, and at $\varepsilon=1$, all the values are either $1$ or $-1$. Around $\varepsilon=0.5$, we have a significant proportion  of all three values, causing estimation to be harder. 
\end{itemize}

\begin{figure}[t]
\centering
\begin{subfigure}[b]{0.45\textwidth}
  \centering
  \includegraphics[width=\textwidth]{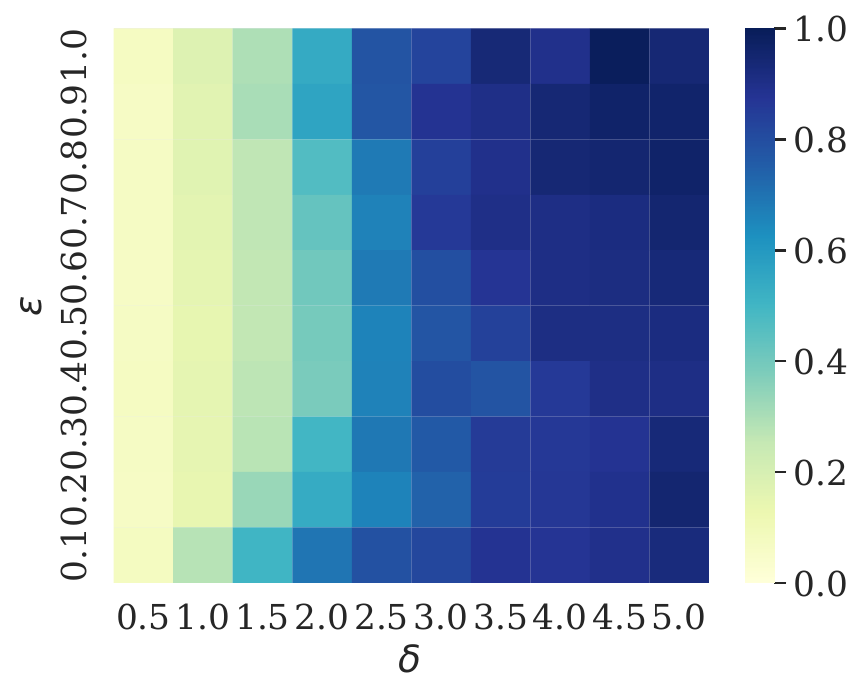}
  \caption{$\alpha=0.6$}
\end{subfigure}
\begin{subfigure}[b]{0.45\textwidth}
  \centering
  \includegraphics[width=\textwidth]{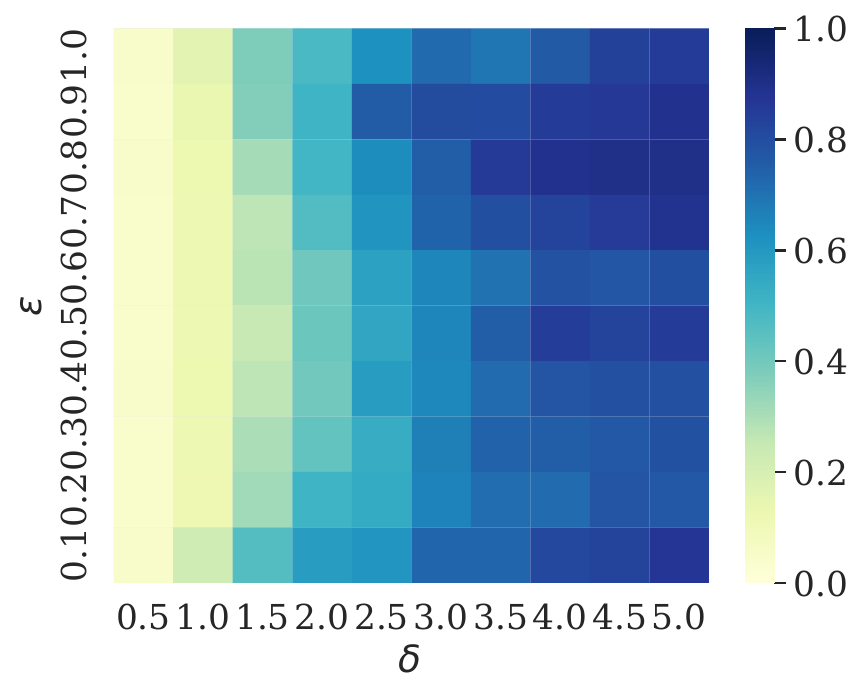}
  \caption{$\alpha=0.7$}
\end{subfigure}
\caption{MLR: Heatmap of minimum normalized correltion for Bayes-optimal $f_k$, with $p=500$, $\sigma =0$.}
\label{fig:sparse_prior_bayes_fk}
\end{figure}

The soft-thresholding function with threshold  $\theta$, denoted by $\text{ST}( \cdot \, ;\theta): \reals  \to \reals$, is defined as
\begin{align}
    \text{ST}(x;\theta)=
    \begin{cases}
    x-\theta &\text{if $x>\theta$} \\
    0 &\text{if $-\theta\leq x\leq \theta$} \\
    x+\theta &\text{if $x\leq-\theta$}.
    \end{cases}
\label{eq:ST-def}
\end{align}
To set the threshold  for the soft-thresholding denoiser $f_{k}$, we recall from 
Theorem \ref{thm:GAMP} that the empirical distribution of  $\big\{ B^k_j \big \}$ converges to the distribution  of the random vector $\Mu^{k}_{B}\bar{B} + G^{k}_B$. Therefore, the empirical distribution of $\big\{ (\Mu^{k}_{B})^{-1} B^{k}_j  \big\} _{j \in [p]}$ converges to the distribution of $\bar{B}+ (\Mu^{k}_{B})^{-1} G^{k}_B$, where $(\Mu^{k}_{B})^{-1}G^{k}_B \sim \normal (0, N^k_B)$, where
\begin{align}
 N^k_B =    \text{Cov}\Big((\Mu^{k}_{B})^{-1}G^{k}_B \Big) =  (\Mu^{k}_{B})^{-1}T^{k}_B (\Mu^{k}_{B})^{-1}.
\end{align}
Letting $\zeta >0$ be a tuning parameter, we set the soft-thresholding denoiser $f_k$ to be:
\begin{align}
    f_{k}(B_j^{k})&=
    \begin{bmatrix}
    \text{ST}\Big(\big\{(\Mu^{k}_B)^{-1}B_j^{k}\big\}_1; \, \zeta\sqrt{\big\{N_B^{k}\big\}_{11}}\Big) \\
    \text{ST}\Big(\big\{(\Mu^{k}_B)^{-1}B_j^{k}\big\}_2; \, \zeta\sqrt{\big\{N_B^{k}\big\}_{22}}\Big)
    \end{bmatrix}, \label{eq:ST_denoiser}
\end{align}
This  implies that
\begin{align}
    \nabla f_{k}(B_j^{k})=
    \begin{bmatrix}
    \frac{\partial\{f_{k}\}_1}{\partial\{B_j^{k}\}_1} & \frac{\partial\{f_{k}\}_1}{\partial\{B_j^{k}\}_2}\\
    \frac{\partial\{f_{k}\}_2}{\partial\{B_j^{k}\}_1} & \frac{\partial\{f_{k}\}_2}{\partial\{B_j^{k}\}_2}
    \end{bmatrix}, \label{eq:ST_gradient}
\end{align}
where for $i_1,i_2\in\{1,2\}$,
\begin{align}
    \frac{\partial\{f_{k}\}_{i_1}}{\partial\{B_j^{k}\}_{i_2}}=
    \begin{cases}
    \{(M^{k}_B)^{-1}\}_{i_1i_2} & \text{if $\big\{(\Mu^{k}_B)^{-1}B_j^{k}\big\}_{i_1}>\zeta\sqrt{\big\{N_B^{k}\big\}_{i_1i_1}}$} \\
    0 & \text{if $\big|\big\{(\Mu^{k}_B)^{-1}B_j^{k}\big\}_{i_1}\big|\leq\zeta\sqrt{\big\{N_B^{k}\big\}_{i_1i_1}}$} \\
    \{(M^{k}_B)^{-1}\}_{i_1i_2} & \text{if $\big\{(\Mu^{k}_B)^{-1}B_j^{k}\big\}_{i_1}<-\zeta\sqrt{\big\{N_B^{k}\big\}_{i_1i_1}}$}.
    \end{cases}
\end{align}
Here the notation $\{\cdot\}_{i_1}$ denotes the $i_1$-th entry of the vector and $\{\cdot\}_{i_1i_2}$   the $i_1,i_2$-th entry of the matrix inside the parentheses. 

\begin{figure}[t]
\centering
\begin{subfigure}[b]{0.45\textwidth}
  \centering
  \includegraphics[width=\textwidth]{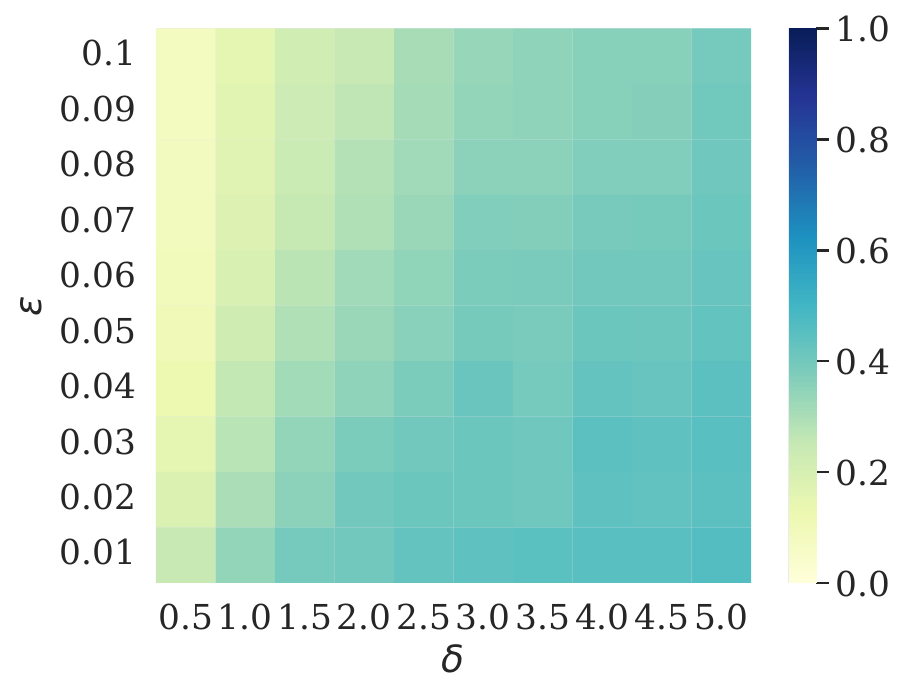}
  \caption{$\alpha=0.5$}
\end{subfigure}
\begin{subfigure}[b]{0.45\textwidth}
  \centering
  \includegraphics[width=\textwidth]{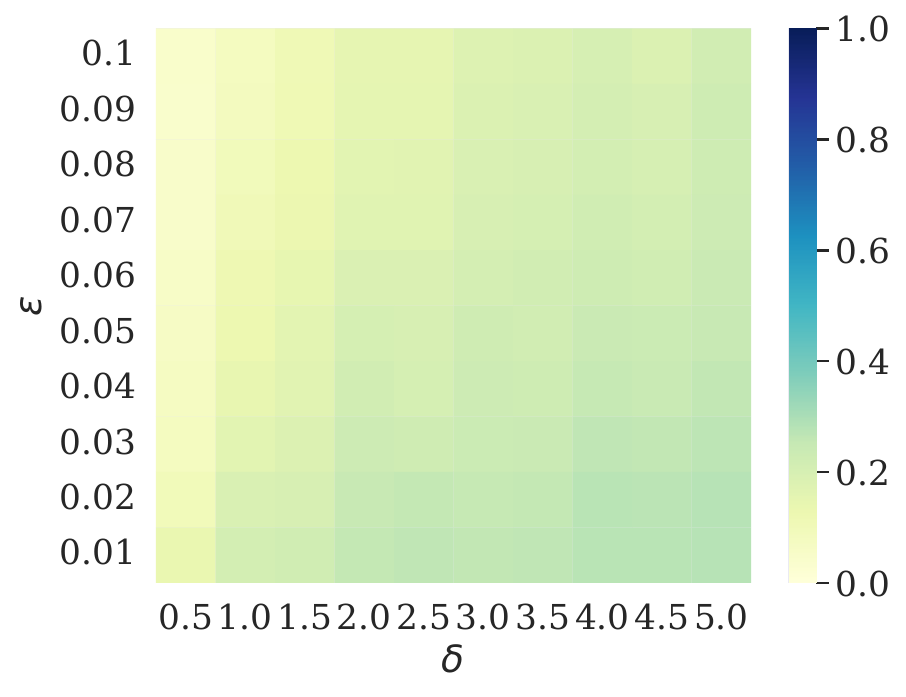}
  \caption{$\alpha=0.6$}
\end{subfigure}
\caption{MLR: Heatmap of minimum normalized  correlation for soft-thresholding $f_k$, with $p=1000$, $\sigma=0$. Soft-thresholding tuning parameter $\zeta =1.1402$}
\label{fig:sparse_prior_ST_fk}
\end{figure}

 Figure \ref{fig:sparse_prior_ST_fk} shows the heatmaps for the soft-thresholding, with the tuning parameter $\zeta$ set to 1.1402.  
 This value of $\zeta$ attains the minimax MSE of the soft-thresholding  denoiser over the class of sparse signal priors which assign a probability mass at least $0.9$ to the value $0$ \citep{Mon12}. We observe that the performance for $\alpha=0.5$ is stronger as the samples are more evenly spread out between the two signals. As expected the correlation improves as  the signal becomes sparser (i.e., $\varepsilon$ decreases), and as $\delta$ increases.
Figure \ref{fig:Bayes_v_ST_fk} compares the Bayes-optimal function with the soft-thresholding function for fixed values of sparsity level $\varepsilon = 0.1$ and mixture parameter $\alpha=0.6$. The significantly better  performance of the  the Bayes-optimal denoiser compared to soft-thresholding is because the former optimally utilizes knowledge of the signal prior, whereas soft-thresholding only uses an estimate of the proportion of zeros in the signals.

\begin{figure}[t]
\centering
\begin{subfigure}[b]{0.45\textwidth}
  \centering
  \includegraphics[width=\textwidth]{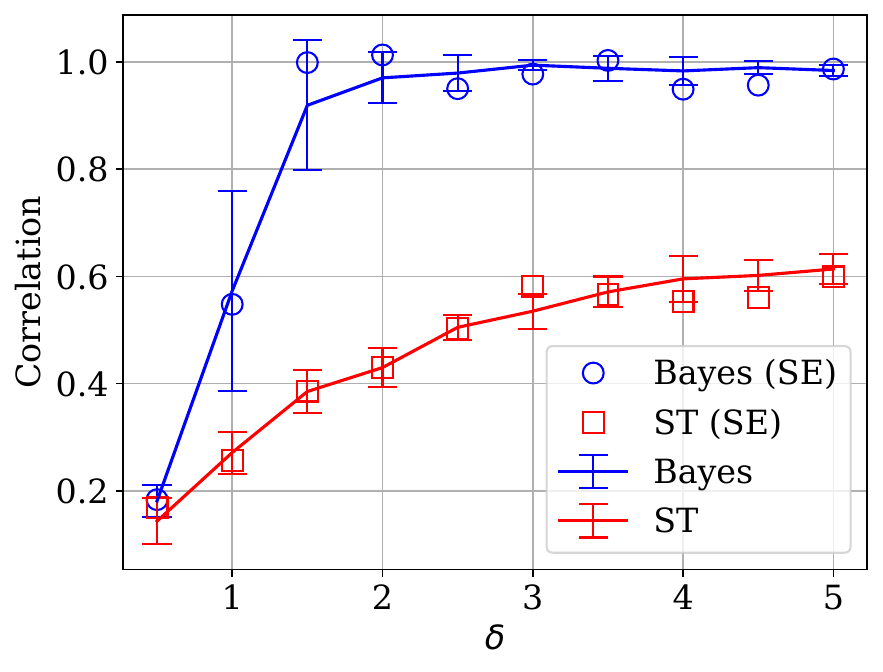}
  \caption{$\beta^{(1)}$}
\end{subfigure}
\begin{subfigure}[b]{0.45\textwidth}
  \centering
  \includegraphics[width=\textwidth]{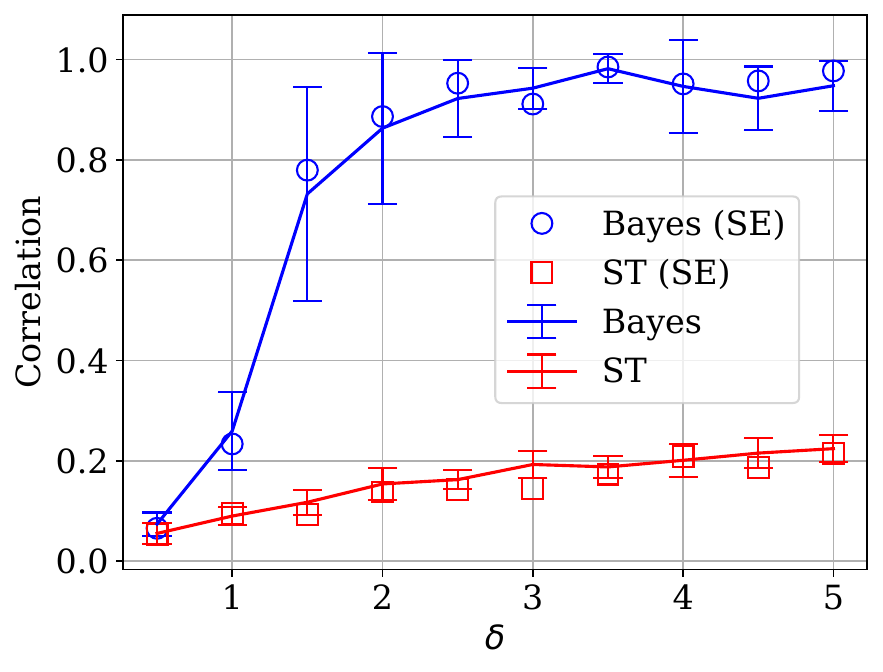}
  \caption{$\beta^{(2)}$}
\end{subfigure}
\caption{MLR: Comparison of minimum normalized  correlation for Bayes-optimal $f_k$ vs.~soft-thresholding $f_k$, with $p=1000$, $\alpha=0.6$, $\sigma=0$, $\zeta=1.1402$, and $\varepsilon=0.1$.}
\label{fig:Bayes_v_ST_fk}
\end{figure}

\paragraph{Comparison with other estimators.} Figure \ref{fig:diff_estG} compares the performance of AMP with other widely studied estimators for mixed linear regression, for the Gaussian signal prior in \eqref{eq:equal_mean_prior_and_init} with independent signals ($\rho=0$). The other estimators are:  the spectral estimator proposed in \cite[Algorithm 2]{Yi14}; alternating minimization (AM) \citep[Algorithm 1]{Yi14}; and  expectation maximization (EM) \citep[Section 2.1]{Far10}. Figure \ref{fig:diff_est_sparse} compares the performance of AMP with these estimators for the sparse signal prior in \eqref{eq:sparse_disc_prior} with $\varepsilon =0.1$.
For this prior, we  modified the least squares step of the AM algorithm in \cite[Algorithm 2]{Yi14} to use Lasso instead of  standard least squares -- this gives better performance as it takes advantage of the signal sparsity. We also tried using the lasso-type EM algorithm \cite{Sta10}, but it did not give a noticeable improvement in performance. In both setups, AMP significantly outperforms the other estimators as it is tailored to take advantage of the signal prior via  the choice of the denoising function $f_k$.

\begin{figure}[t]
\centering
\subfloat[$\beta^{(1)}$\label{fig:beta1_diff_estG}]{\includegraphics[width=.45\columnwidth]{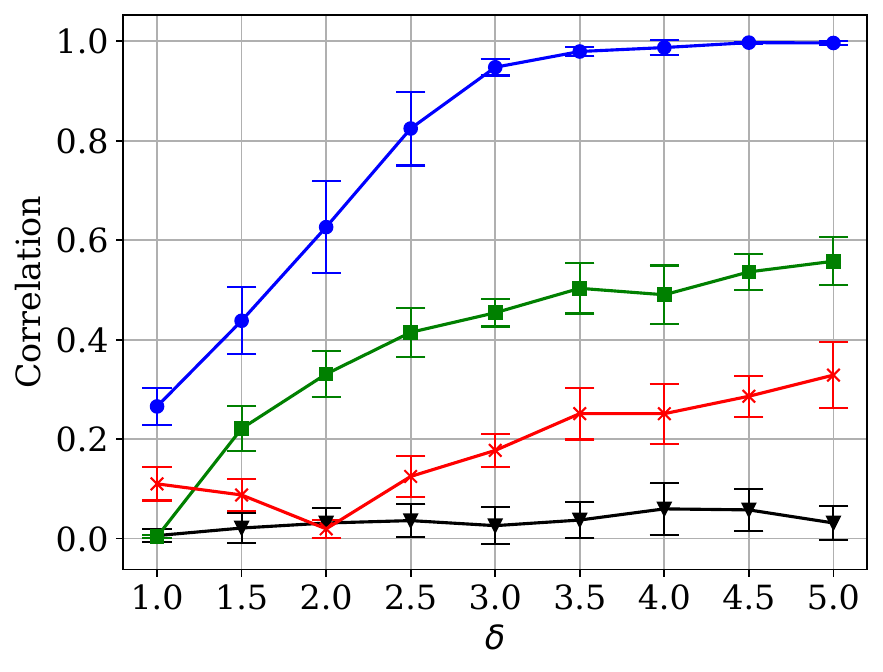}}
\subfloat[$\beta^{(2)}$\label{fig:beta2_diff_estG}]{\includegraphics[width=.45\columnwidth]{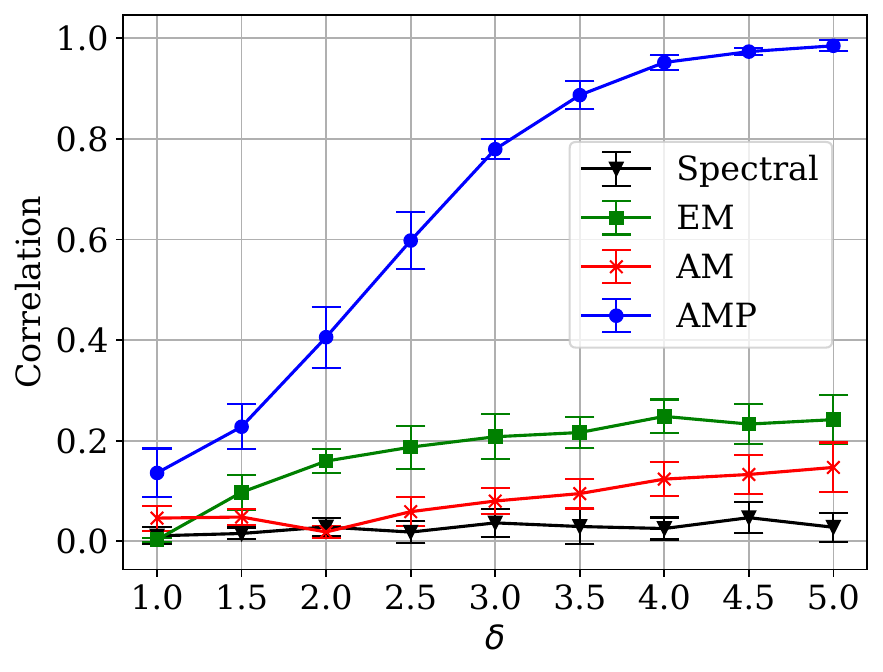}}
\caption{MLR, comparison of different estimators for Gaussian prior with $\rho=0$: Normalized squared correlation vs. $\delta$, with  $\alpha=0.6$, $\sigma=0$.}
\label{fig:diff_estG}
\end{figure}

\begin{figure}[t]
\centering
\subfloat[$\beta^{(1)}$\label{fig:beta1_diff_est_sparse}]{\includegraphics[width=.45\columnwidth]{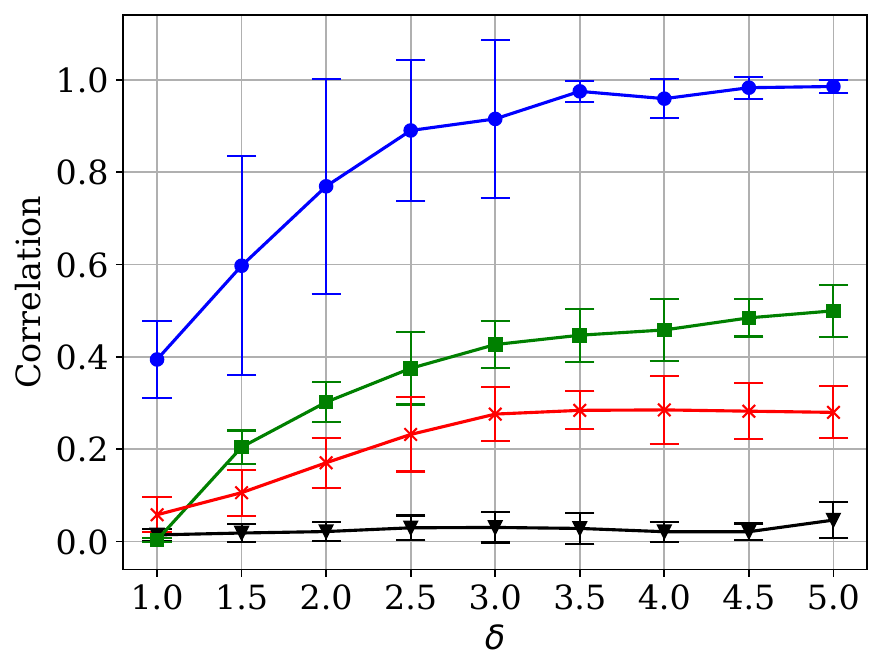}}
\subfloat[$\beta^{(2)}$\label{fig:beta2_diff_est_sparse}]{\includegraphics[width=.45\columnwidth]{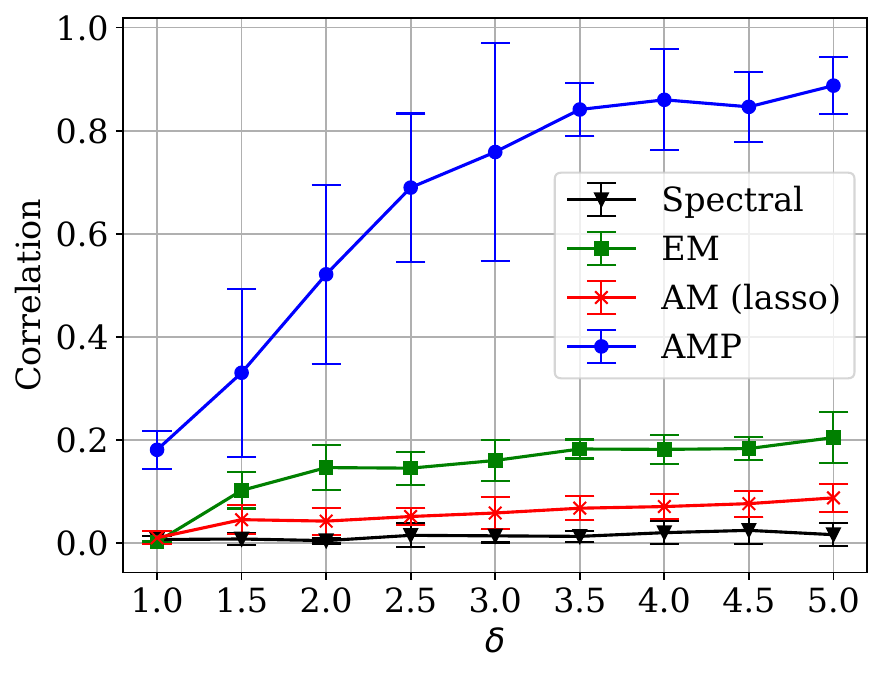}}
\caption{MLR, comparison of different estimators for sparse prior: Normalized squared correlation vs. $\delta$, with  $\alpha=0.6$, $\sigma=0.1$.}
\label{fig:diff_est_sparse}
\end{figure}

\subsection{Mixed Linear Regression (Three Signals)}

To illustrate AMP's ability to tackle MLR with more than two signals, we now consider the model \eqref{eq:MLR_model} with three signals:
\begin{align}
    Y_i
    =\langle X_i,\beta^{(1)}\rangle c_{i1}+\langle X_i,\beta^{(2)}\rangle c_{i2}&+\langle X_i,\beta^{(3)}\rangle c_{i3} +\epsilon_i, \qquad i\in[n].
    \label{eq:MLR_three_components}
\end{align}
We take $[c_{i1},c_{i2},c_{i3}]^\top$ to be a one-hot vector, and denote the position of the one in the one-hot vector by $c_i\sim_{\iid}\text{Categorical}(\{\alpha_1,\alpha_2,\alpha_3\})$. As before, $\epsilon_i\sim_\iid\normal(0,\sigma^2)$, and $X_i\sim_\iid\normal(0,I_p/n)$, for $i\in[n]$. We set the signal dimension $p=500$ and vary the value of $n$ in our experiments. The AMP algorithm in \eqref{eq:GAMP} is implemented with $g_k=g_k^*$ and $f_k=f_k^*$  (i.e., the optimal choices given in Proposition \ref{prop:optimal_fk}).

We use independent Gaussian priors for the three signals. Specifically,  we generate:
\begin{equation}
\begin{split}
    &(\beta_j^{(1)},\beta_j^{(2)},\beta_j^{(3)})
    \sim_\iid
    \normal(\E[\bar{B}],I_3), \ \ j\in[p]  \\
    &c_i\sim_\iid\text{Categorical}(\{\alpha_1,\alpha_2,\alpha_3\}),
    \ \ i\in[n].
    \label{eq:MLR_3sig_prior1} 
    \end{split}
\end{equation}
The initializer $\hB^0 \in \reals^{p \times 3}$ is chosen randomly according to the same distribution, independent of the signal.  We consider the following three scenarios, where $\sigma=0$ (noiseless):
\begin{itemize}
    \item \textbf{Signals with same mean and same proportions.} Figure \ref{fig:MLR_3sig_same_mean_same_prop} shows the performance with $\E[\bar{B}]=[0,0,0]^\top$ and $(\alpha_1,\alpha_2,\alpha_3)=(1/3,1/3,1/3)$. We observe that the  performance does not improve much with increasing $\delta$ as the algorithm finds it challenging to distinguish the signals when they all have the same prior and correspond to the same proportion of observations.
    \item \textbf{Signals with same mean and different proportions.} Figure \ref{fig:MLR_3sig_same_mean_diff_prop} shows the performance with $\E[\bar{B}]=[0,0,0]^\top$ and $(\alpha_1,\alpha_2,\alpha_3)=(0.5,0.3,0.2)$. The performance here is  significantly better than the previous case where signals have the same mean and proportions. As expected, the correlation for $\beta_1$ is significantly better than that for $\beta_2$ and $\beta_3$ since $\beta_1$ has the highest proportion of observations.
    \item \textbf{Signals with different means and same proportions.} Figure \ref{fig:MLR_3sig_diff_mean_same_prop} shows the performance with $\E[\bar{B}]=[0,0.5,1]^\top$ and $(\alpha_1,\alpha_2,\alpha_3)=(1/3,1/3,1/3)$.  This is the case with the best estimation performance.  This is because the distinct means help distinguish the signals from one another and the equal proportions ensure that all three have sufficient number of observations for large enough $\delta$.
\end{itemize}

\begin{figure}[t]
\centering
\begin{subfigure}[b]{0.32\textwidth}
  \centering
  \includegraphics[width=\textwidth]{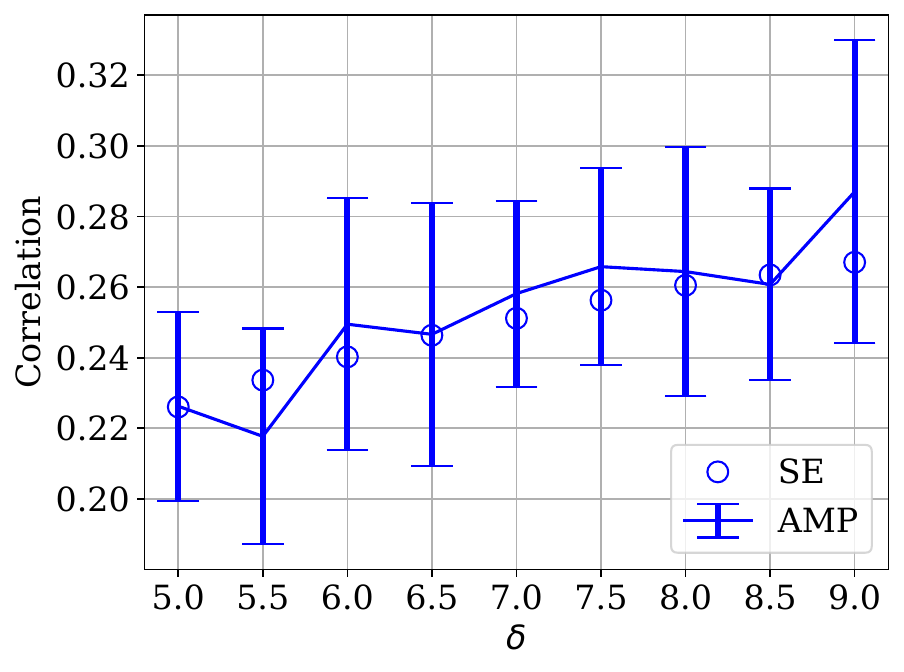}
  \caption{$\beta^{(1)}$}
\end{subfigure}
\begin{subfigure}[b]{0.32\textwidth}
  \centering
  \includegraphics[width=\textwidth]{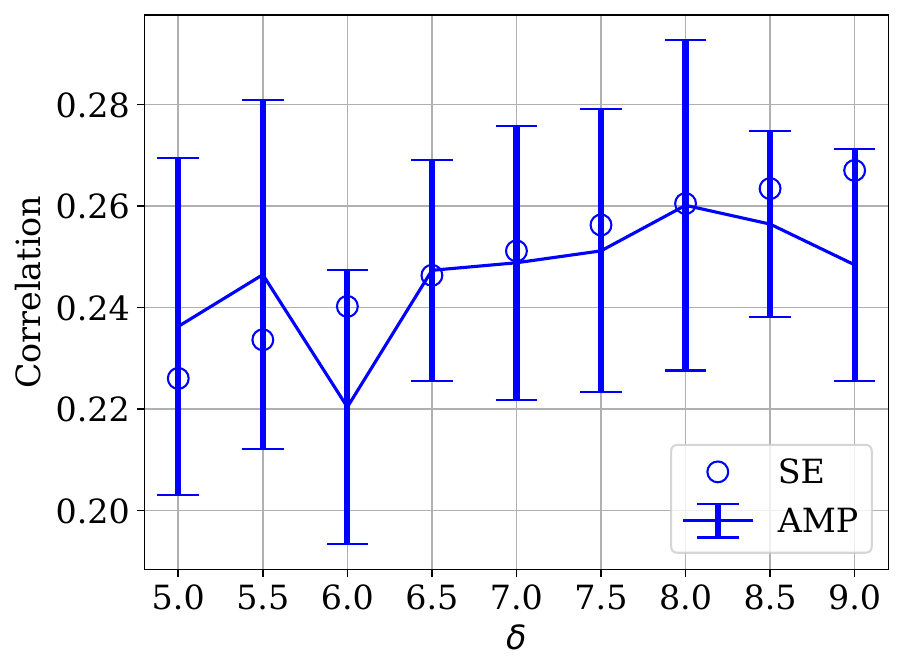}
  \caption{$\beta^{(2)}$}
\end{subfigure}
\begin{subfigure}[b]{0.32\textwidth}
  \centering
  \includegraphics[width=\textwidth]{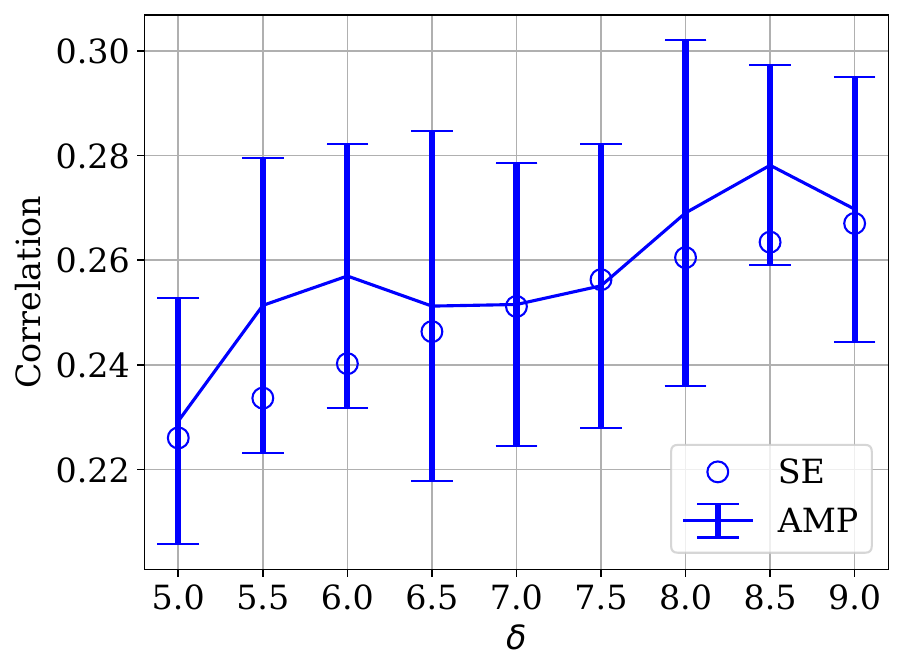}
  \caption{$\beta^{(3)}$}
\end{subfigure}
\caption{MLR with three signals:  signals with same mean and same proportions.}
\label{fig:MLR_3sig_same_mean_same_prop}
\end{figure}

\begin{figure}[t]
\centering
\begin{subfigure}[b]{0.33\textwidth}
  \centering
  \includegraphics[width=\textwidth]{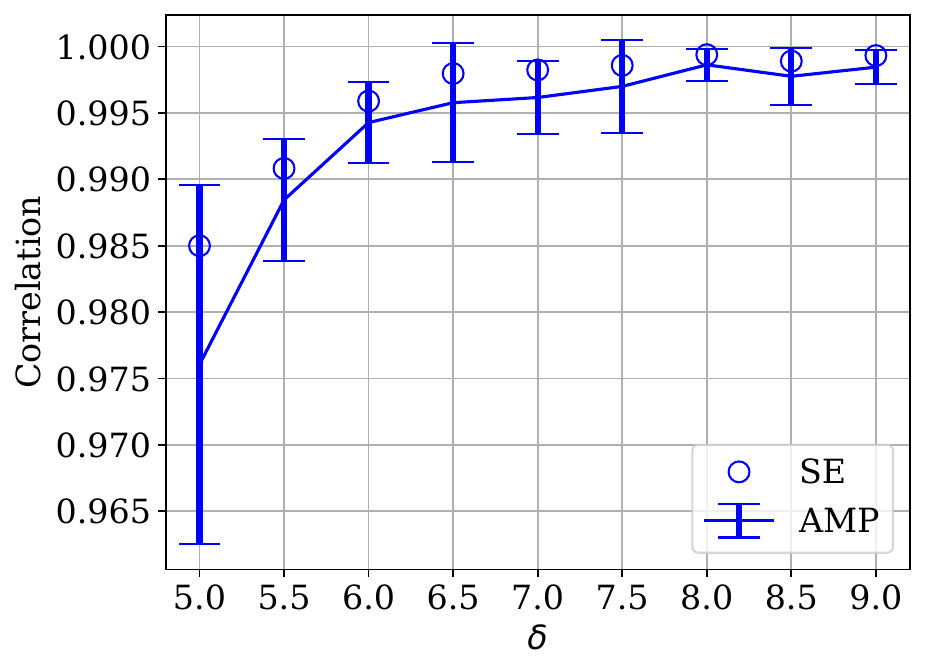}
  \caption{$\beta^{(1)}$}
\end{subfigure}
\begin{subfigure}[b]{0.32\textwidth}
  \centering
  \includegraphics[width=\textwidth]{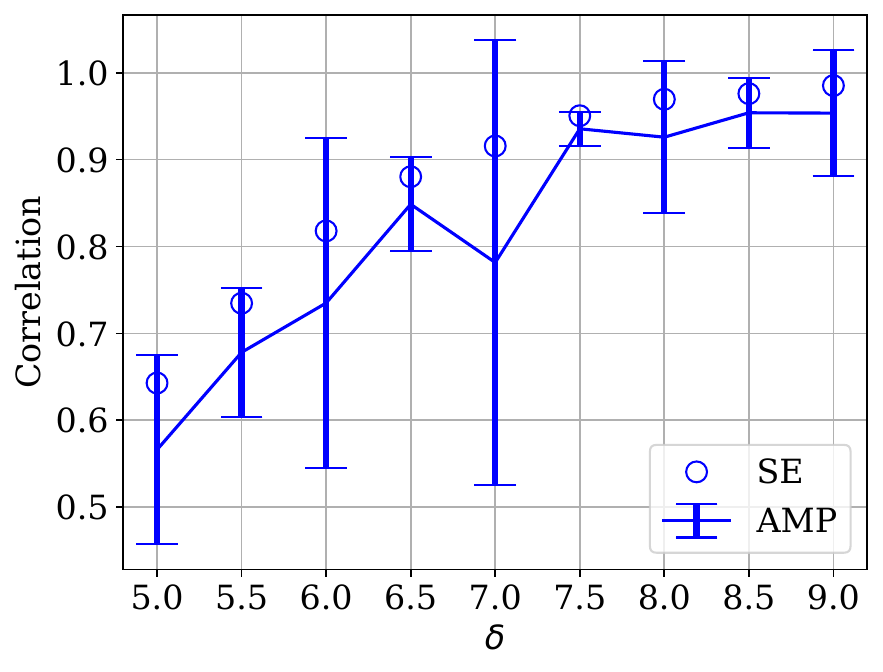}
  \caption{$\beta^{(2)}$}
\end{subfigure}
\begin{subfigure}[b]{0.32\textwidth}
  \centering
  \includegraphics[width=\textwidth]{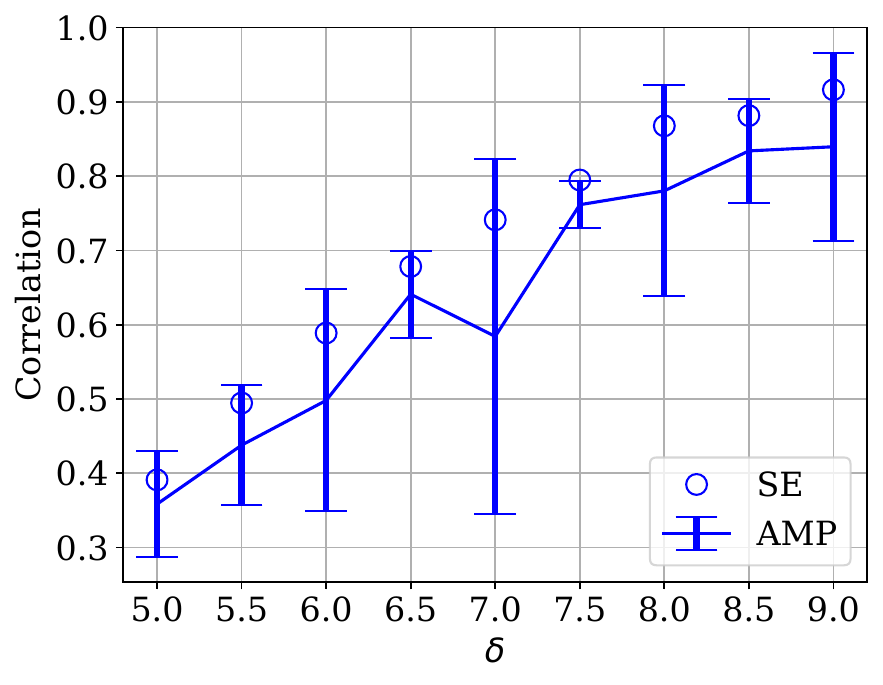}
  \caption{$\beta^{(3)}$}
\end{subfigure}
\caption{MLR with three signals: signals with same mean and different proportions.}
\label{fig:MLR_3sig_same_mean_diff_prop}
\end{figure}

\begin{figure}[t]
\centering
\begin{subfigure}[b]{0.32\textwidth}
  \centering
  \includegraphics[width=\textwidth]{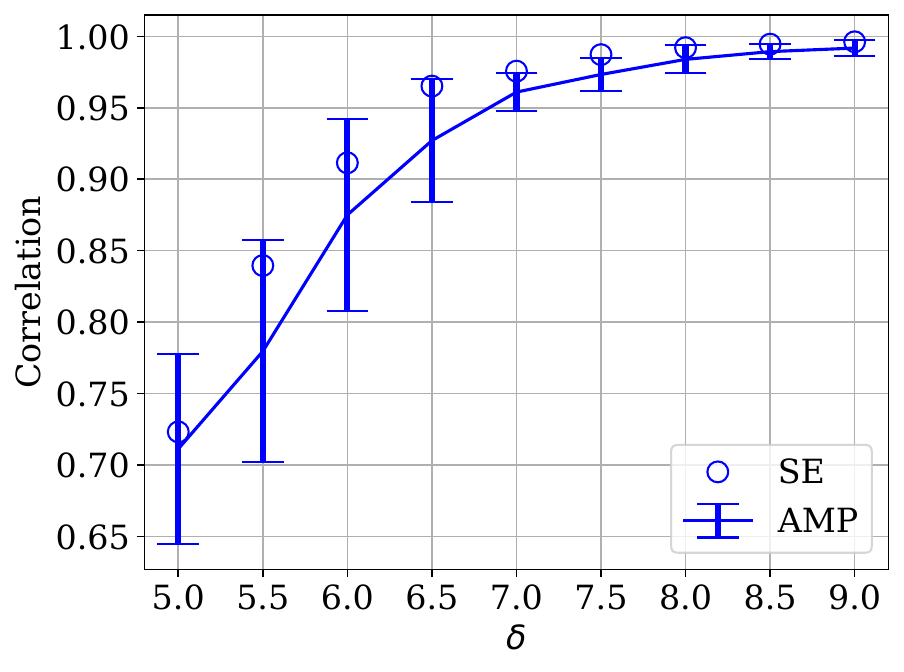}
  \caption{$\beta^{(1)}$}
\end{subfigure}
\begin{subfigure}[b]{0.32\textwidth}
  \centering
  \includegraphics[width=\textwidth]{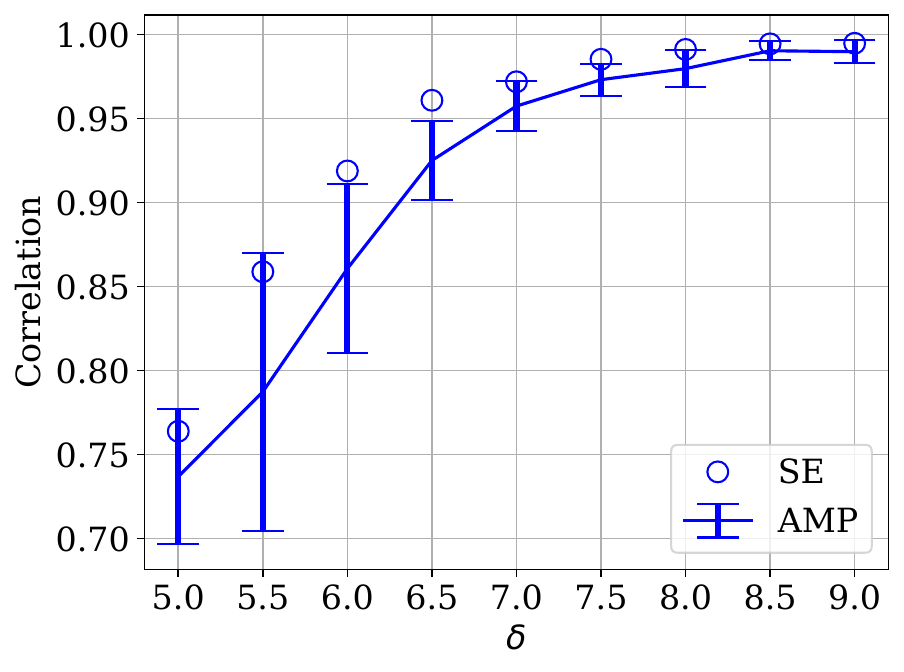}
  \caption{$\beta^{(2)}$}
\end{subfigure}
\begin{subfigure}[b]{0.32\textwidth}
  \centering
  \includegraphics[width=\textwidth]{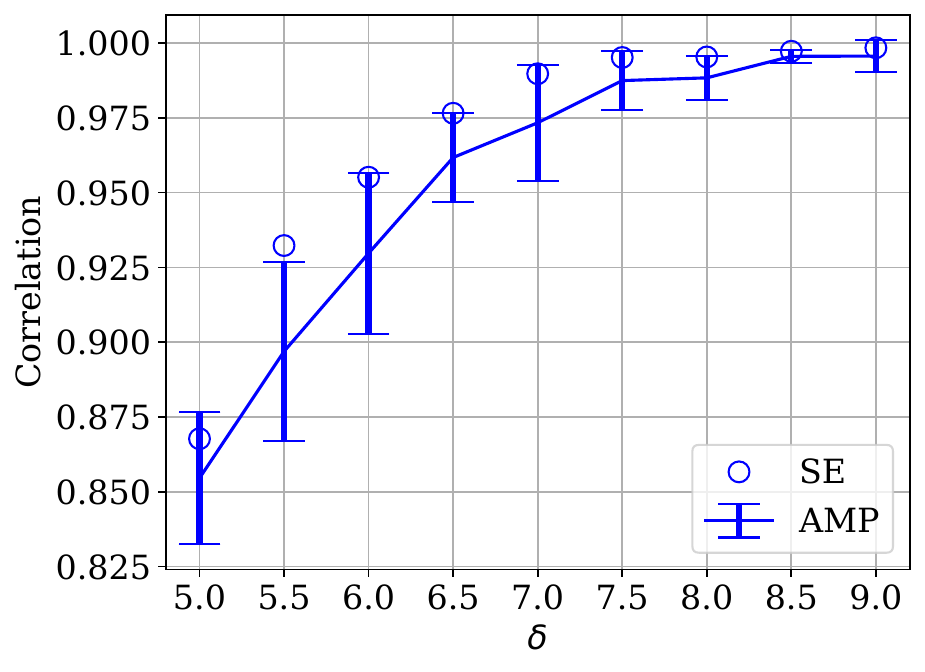}
  \caption{$\beta^{(3)}$}
\end{subfigure}
\caption{MLR with three signals: signals with different means and same proportions.}
\label{fig:MLR_3sig_diff_mean_same_prop}
\end{figure}

Finally, Figure \ref{fig:MLR_3sig_GAMP_v_others} compares the performances of AMP with other widely studied estimators for MLR, for the Gaussian signal prior in \eqref{eq:MLR_3sig_prior1} with $\E[\bar{B}]=[0,0.5,1]^\top$ and $(\alpha_1,\alpha_2,\alpha_3)=(1/3,1/3,1/3)$ (this is the case where signals have different prior distributions but appear in the same proportion of observations). The other estimators are: the spectral estimator proposed in \cite[Algorithm 2]{Yi14}; alternating minimization (AM) \citep[Algorithm 1]{Yi14}; and expectation maximization (EM) \citep[Section 2.1]{Far10}. We modified the grid search\footnote{In the two signal case, grid search was used to iterate over all possible combinations of the top two eigenvectors of $\frac{1}{n}\sum_{i=1}^nY_i X_iX_i^\top$ to get the best combination for each signal.} step of the spectral estimator in \cite[Algorithm 2]{Yi14} to sample evenly across a sphere instead of a circle (to account for the fact that we now have three signals instead of two). Since this step cannot be done exactly like in the 2D case, we used the Fibonacci sphere algorithm \citep{Gon10} to achieve this approximately and efficiently in our 3D case. As in the case of two-signal MLR, AMP significantly outperforms the other estimators as it is tailored to take advantage of the signal prior via the choice of the denoising function $f_k$.

\begin{figure}[t]
\centering
\begin{subfigure}[b]{0.32\textwidth}
  \centering
  \includegraphics[width=\textwidth]{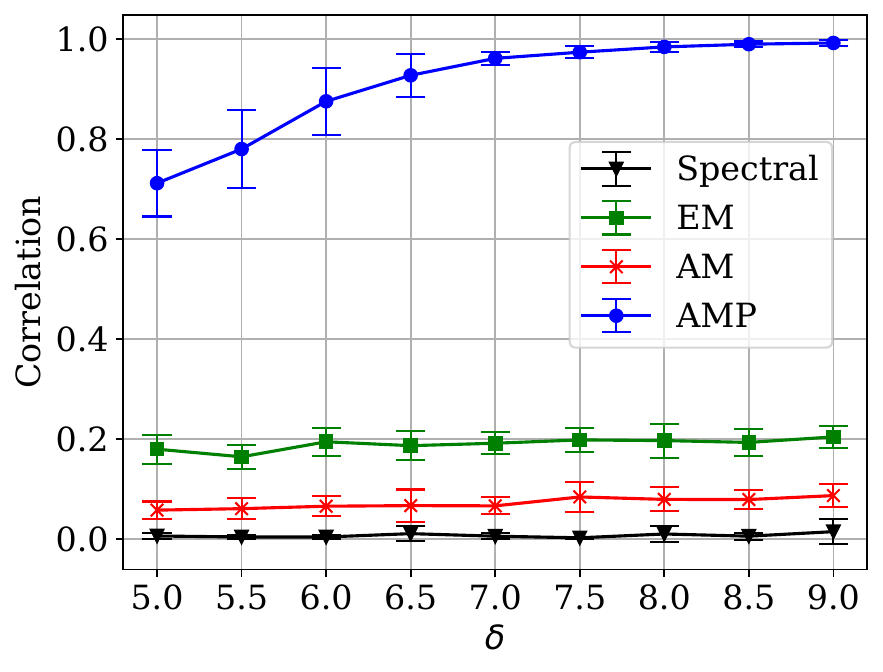}
  \caption{$\beta^{(1)}$}
\end{subfigure}
\begin{subfigure}[b]{0.32\textwidth}
  \centering
  \includegraphics[width=\textwidth]{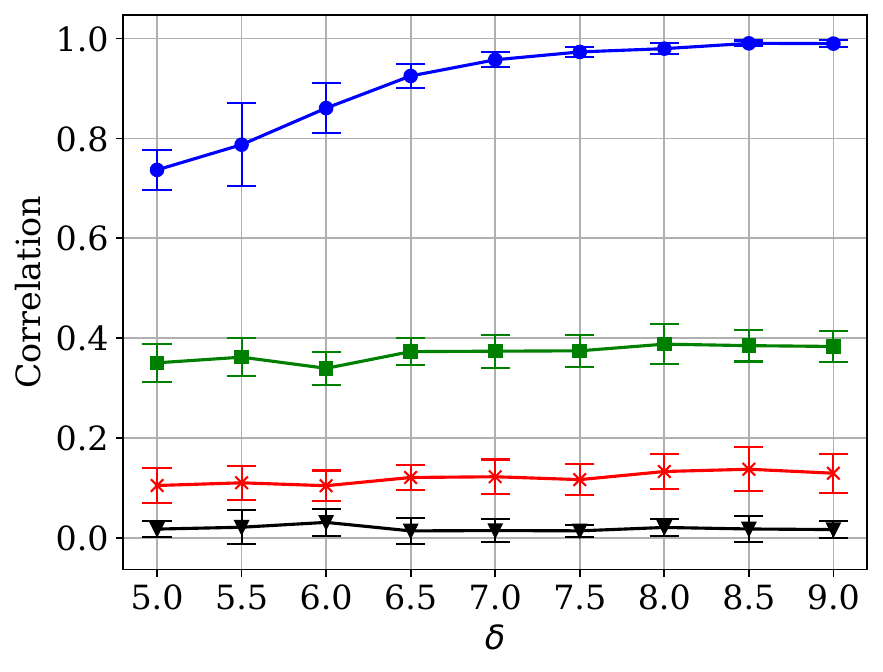}
  \caption{$\beta^{(2)}$}
\end{subfigure}
\begin{subfigure}[b]{0.32\textwidth}
  \centering
  \includegraphics[width=\textwidth]{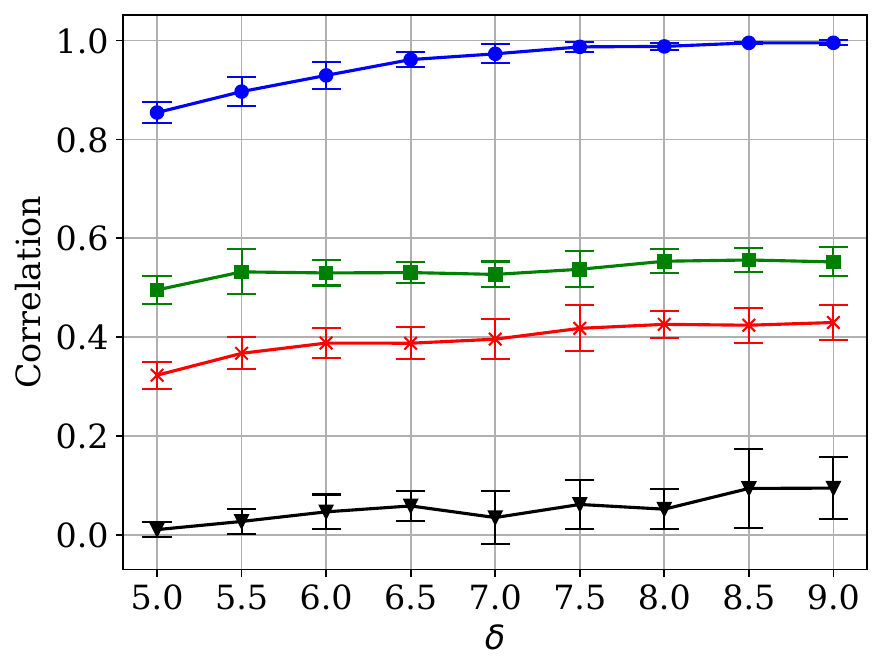}
  \caption{$\beta^{(3)}$}
\end{subfigure}
\caption{MLR with three signals:  Comparison with different estimators. The signals have different means and occur in the same proportions.}
\label{fig:MLR_3sig_GAMP_v_others}
\end{figure}

\subsection{Max-Affine Regression} \label{sec:max-affine_reg}

We consider on the MAR model \eqref{eq:max-affine_model} with two signals, which is given by:
\begin{align}
    Y_i=\max\Big\{\langle X_i,\beta^{(1)}\rangle+b_1, \, \langle X_i,\beta^{(2)}\rangle+b_2\Big\}+\epsilon_i, 
    \ \ i\in[n],
\end{align}
where $\epsilon_i\sim_{\iid}\mathcal{N}(0,\sigma^2)$, $X_i\sim_{\iid}\mathcal{N}(0,I_p/n)$ for $i\in[n]$. 
 Recall from \eqref{eq:MAR_model_concise} that  MAR can be written as an instance of the matrix GLM using the augmented features $X_i^{({\text{ma}})}= \begin{bmatrix} X_i \\ 1 \end{bmatrix}\in\mb{R}^{p+1}$,  $i \in [n]$. Since the augmented features are not i.i.d. Gaussian (due to the last component being 1),
we  use the original formulation of MAR in \eqref{eq:max-affine_model} and consider the intercepts $b_1$ and $b_2$ to be unknown parameters of the output function $q( \cdot, \cdot)$.  

Our solution is to estimate the unknown intercepts using an expectation-maximization (EM) algorithm. The EM algorithm iteratively produces intercept estimates,  denoted by $b^m \equiv (b^m_1, b^m_2)$, for $m \ge 1$. However, the expectation step of the EM algorithm requires an estimate of $\Theta=XB$, which we approximate  via AMP. This leads to a combined expectation-maximization approximate message passing (EM-AMP) algorithm, which is described in Algorithm \ref{alg:EM-AMP}. The idea of combining  AMP with the EM algorithm was introduced by \cite{Vil13},  for sparse linear regression with unknown parameters in the signal prior.

The AMP stage in step 3 of Algorithm \ref{alg:EM-AMP} is implemented with $g_k=g_k^*$ and $f_k=f_k^*$ (i.e., the optimal choices), computed using the current intercept estimates.  The details of computing the $g_k^*$  and the conditional expectation $\E\big[Z|\bar{Y};b^{m}\big]$ in this step are given in Appendix \ref{appen:MAR_implementation}. The derivation of the EM updates in steps 4 and 5 of Algorithm \ref{alg:EM-AMP} is given in Section \ref{sec:derivation_EM_step}.

\begin{algorithm}[!t]
    \begin{algorithmic}[1]
        \STATE Initialize the intercepts $b^0:=(b_1^0,b_2^0)$, and $\widehat{B}^{k_{\max},0}$.
        \FOR{iteration $m:=1,\dots,m_{\max}$}{
        \STATE Compute $\E\big[Z|\bar{Y};b^{m}\big]$, and run AMP with intercept estimates $b^{m}$ as part of the model for $k_{\max}$ iterations to produce $\widehat{\Theta}^{m}=X\widehat{B}^{k_{\max},m}$. Let $\widehat{\Theta}^{(1)}$, $\widehat{\Theta}^{(2)}$ be the two columns of $\widehat{\Theta}^m$.
        \STATE Compute $b_1^{m+1}:=\frac{1}{|\{i:\widehat{\Theta}_i^{(1)}+b_1^m>\widehat{\Theta}_i^{(2)}+b_2^m\}|}\sum_{i:\widehat{\Theta}_i^{(1)}+b_1^m>\widehat{\Theta}_i^{(2)}+b_2^m}Y_i-\E\big[Z_1|\bar{Y};b^m\big]$.
        \STATE Compute $b_2^{m+1}:=\frac{1}{|\{i:\widehat{\Theta}_i^{(1)}+b_1^m\leq\widehat{\Theta}_i^{(2)}+b_2^m\}|}\sum_{i:\widehat{\Theta}_i^{(1)}+b_1^m\leq\widehat{\Theta}_i^{(2)}+b_2^m}Y_i-\E\big[Z_2|\bar{Y};b^m\big]$.
        }\ENDFOR
        \STATE Output $b^{m_{\max}}$ and $\widehat{B}^{k_{\max},m_{\max}}$.
    \end{algorithmic}
    \caption{Expectation-maximization approximate message passing (EM-AMP) \label{alg:EM-AMP}}
\end{algorithm}

We set the signal dimension $p=500$ and vary the value of $n$ in our experiments. We consider different choices for the intercepts $b:=(b_1,b_2)$ and use a Gaussian prior for $B=\big(\beta^{(1)},\beta^{(2)}\big)$, where we generate
\begin{align}
    \big(\beta^{(1)}_j,\beta^{(2)}_j\big)
    \sim_\iid
    \normal\big(\E[\bar{B}],I_2\big),
    \ \ j\in[p].
\end{align}
The initializer $\widehat{B}^0\in\mb{R}^{p\times 2}$ is chosen  according to the same distribution, independently of the signal. The EM initialization  $b^0$ is taken to be $(0,0)$.

Figures \ref{fig:MAR_diff_mean_same_inter}-\ref{fig:MAR_diff_mean_same_inter_sig04} show the performance of  EM-AMP for max-affine regression with different choices of prior, intercepts, and noise level.   The performance in all the plots is measured via the normalized squared correlation between the full signals  $\beta_{\text{ma}}^{(1)}=((\beta^{(1)})^\top,b_1)^\top$ and $\beta_{\text{ma}}^{(2)}=((\beta^{(2)})^\top,b_2)^\top$ and their respective estimates $\hat{\beta}_{\text{ma}}^{(1)}=((\hat{\beta}^{(1)})^\top,\hat{b}_1)^\top$ and $\hat{\beta}_{\text{ma}}^{(2)}=((\hat{\beta}^{(2)})^\top,\hat{b}_2)^\top$, i.e.,
\begin{align}
\frac{\langle \beta_{\text{ma}}^{(l)}, \, \hat{\beta}_{\text{ma}}^{(l)} \rangle^2}{\|\hat{\beta}_{\text{ma}}^{(l)}\|_2^2\|\beta_{\text{ma}}^{(l)}\|_2^2}, \quad
    \text{where $l\in\{1,2\}$.}
\end{align}
Each point on the plots is obtained from 5 independent runs, where in each run, we execute EM-AMP with $m_{\max}=5$ and $k_{\max}=5$. We report the average and error bars at 1 standard deviation of the final iteration. EM-AMP is compared with:  (i) the alternating minimization algorithm \citep{Gho22} (the only known algorithm for MAR with theoretical guarantees), and (ii) the Oracle AMP (OR-AMP)  where we assume that the true intercepts $b$ are known and are part of the matrix GLM model. Though the intercepts are not known in practice,  OR-AMP provides an upper bound on the best correlation achievable by AMP since it uses the optimal denoising functions and the correct intercepts. Hence, it is reasonable to expect that OR-AMP would provide the best performance. 

We study the performance of our algorithms for the following three scenarios:
\begin{itemize}
    \item \textbf{Same intercept, signals with different means.} Figure \ref{fig:MAR_diff_mean_same_inter} shows the results for the setting $  b=(1,1),\,
        \E[\bar{B}]=[0,1]^\top, \,\sigma=0.1$.
When the intercepts are the same, the proportion of observations from each signal is roughly the same for large enough $\delta$. This is because $Y_i = \max_{l \in\{1, 2\}} \big( \langle X_i,\beta^{(l)}\rangle+b_l \big)$, where  $\langle X_i,\beta^{(l)}\rangle$ is a zero-mean Gaussian regardless of the mean of $\beta^{(l)}$. (The variance of $\langle X_i,\beta^{(l)}\rangle$ depends on the mean of $\beta^{(l)}$.)
In this setting, AM performs poorly for smaller $\delta$ values, but matches or slightly exceeds the performance of EM-AMP for large $\delta$.
    \item \textbf{Different intercepts, signals with  different means.} Figure \ref{fig:MAR_diff_mean_diff_inter} shows the results for the setting $ b=(1,0),\        \E[\bar{B}]=[0,1]^\top, \ \sigma=0.1$.
As mentioned above, the signal mean does not affect the proportion of observations from each sample. Hence, the signal with a larger intercept will have more observations. EM-AMP outperforms AM for the signal with fewer observations for all  values $\delta$, while for the other signal, AM is slightly better for larger values of $\delta$. 
    \item \textbf{Same intercept, signals with different means, higher noise level.} Figure \ref{fig:MAR_diff_mean_same_inter_sig04} shows the results for $ b=(1,1),\         \E[\bar{B}]=[0,1]^\top, \ \sigma=0.4 $.
  The plots show that EM-AMP significantly outperforms AM for all values of $\delta$. AM is quite sensitive to the presence of noise unlike EM-AMP, which is more robust and nearly matches the performance OR-AMP. 
\end{itemize}

\begin{figure}[t]
\centering
\begin{subfigure}[b]{0.45\textwidth}
  \centering
  \includegraphics[width=\textwidth]{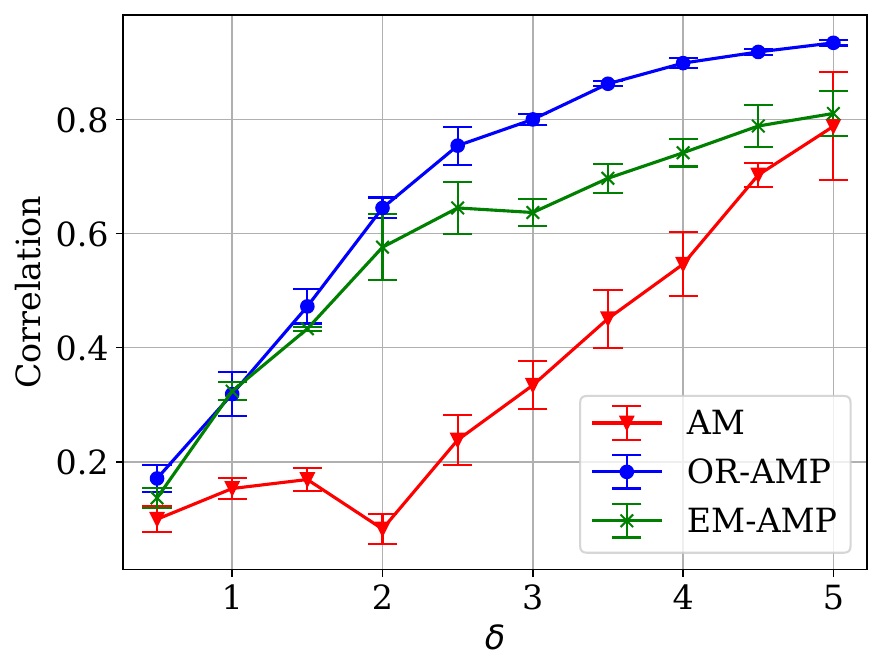}
  \caption{$\beta^{(1)}$}
\end{subfigure}
\begin{subfigure}[b]{0.45\textwidth}
  \centering
  \includegraphics[width=\textwidth]{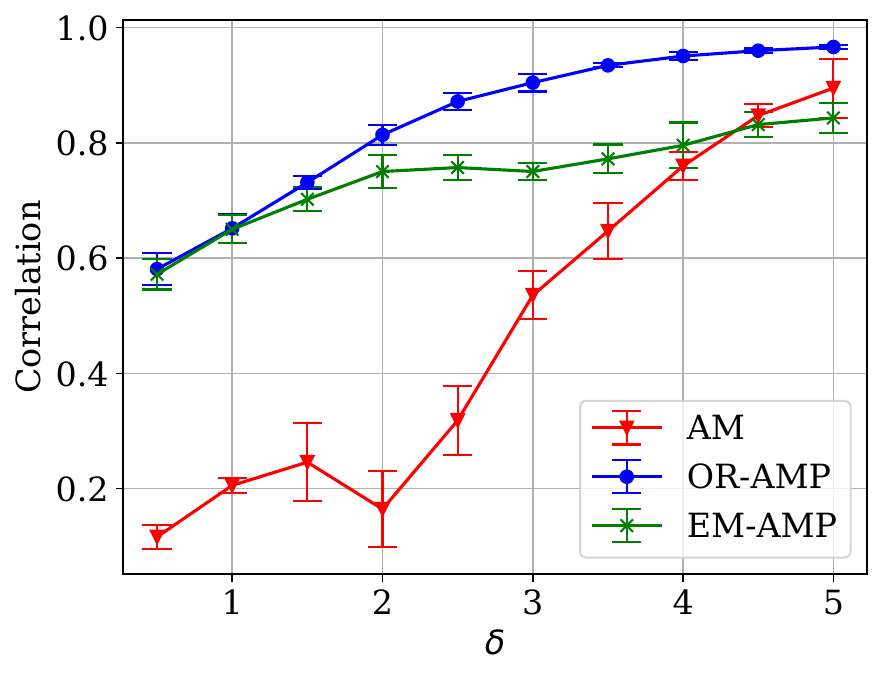}
  \caption{$\beta^{(2)}$}
\end{subfigure}
\caption{MAR: Same intercepts, signals with different means,  noise level $\sigma=0.1$.}
\label{fig:MAR_diff_mean_same_inter}
\end{figure}
\begin{figure}[t]
\centering
\begin{subfigure}[b]{0.45\textwidth}
  \centering
  \includegraphics[width=\textwidth]{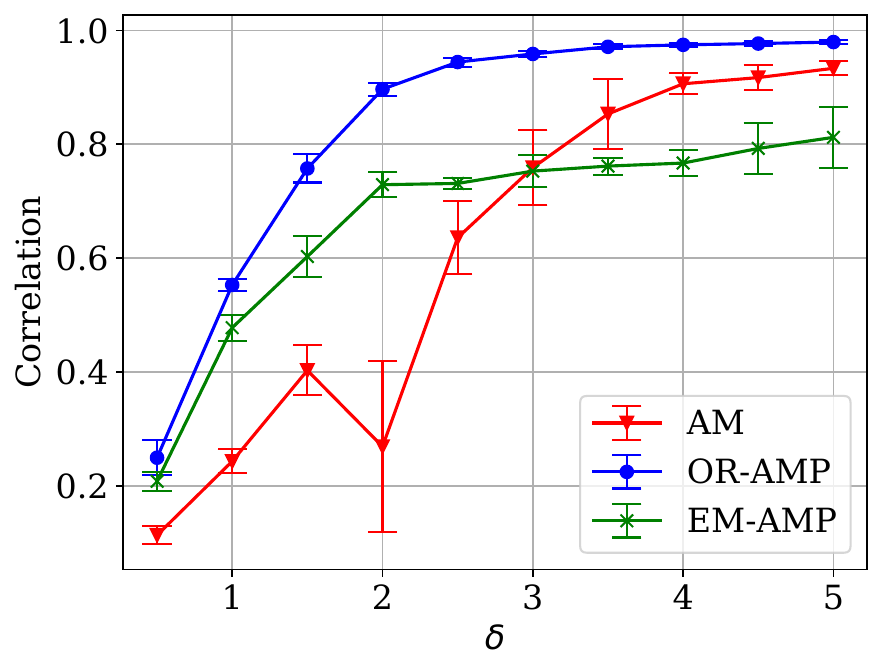}
  \caption{$\beta^{(1)}$}
\end{subfigure}
\begin{subfigure}[b]{0.45\textwidth}
  \centering
  \includegraphics[width=\textwidth]{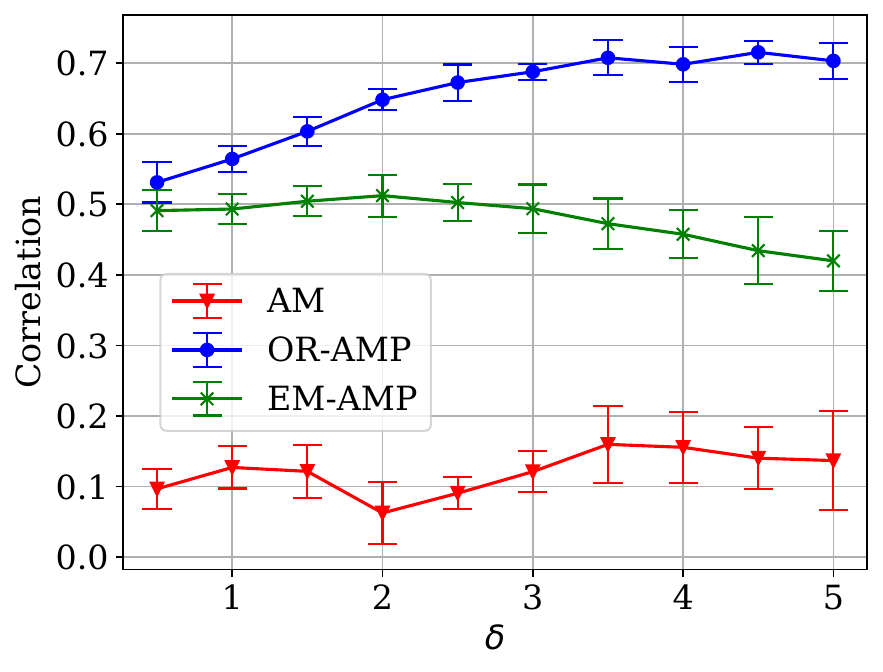}
  \caption{$\beta^{(2)}$}
\end{subfigure}
\caption{MAR: Different intercepts, signals with different means, noise level $\sigma=0.1$.}
\label{fig:MAR_diff_mean_diff_inter}
\end{figure}
\begin{figure}[t]
\centering
\begin{subfigure}[b]{0.45\textwidth}
  \centering
  \includegraphics[width=\textwidth]{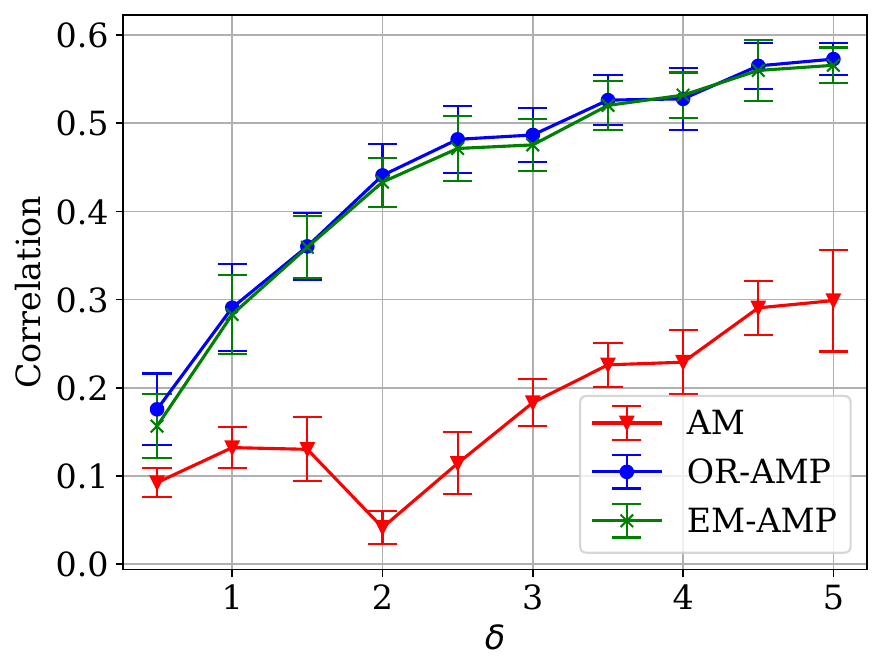}
  \caption{$\beta^{(1)}$}
\end{subfigure}
\begin{subfigure}[b]{0.45\textwidth}
  \centering
  \includegraphics[width=\textwidth]{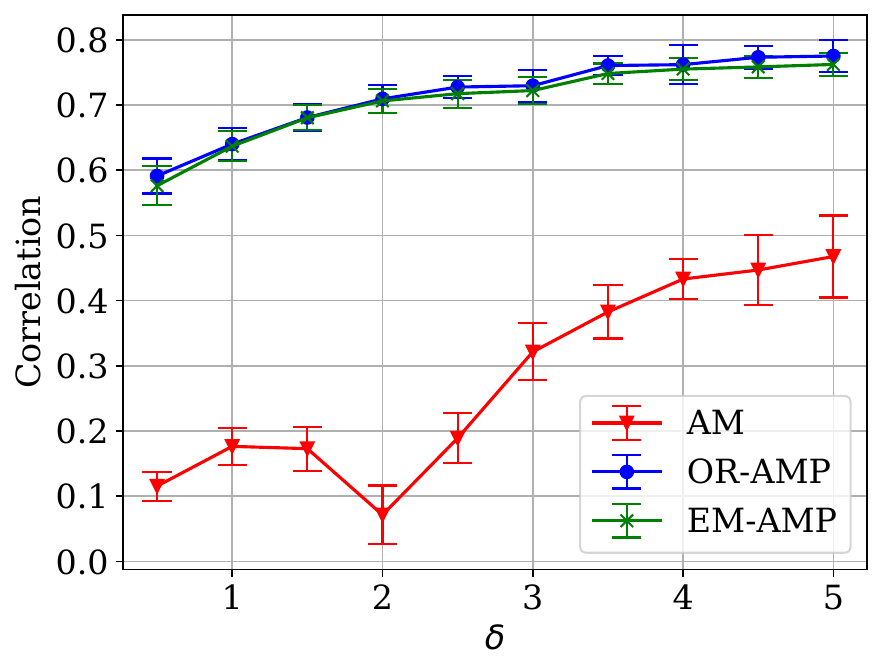}
  \caption{$\beta^{(2)}$}
\end{subfigure}
\caption{MAR: Same intercept, signals with different means, noise level $\sigma=0.4$.}
\label{fig:MAR_diff_mean_same_inter_sig04}
\end{figure}

\paragraph{Hard case.} When entries of both $\beta^{(1)}$ and $\beta^{(2)}$ are generated from the same distribution, and the intercepts $b_1$ and $b_2$ are the same, the estimation problem becomes very challenging. In this case, AM, EM-AMP, and AMP all struggle to give an accurate estimate. 

\begin{figure}[t]
\centering
\begin{subfigure}[b]{0.45\textwidth}
  \centering
  \includegraphics[width=\textwidth]{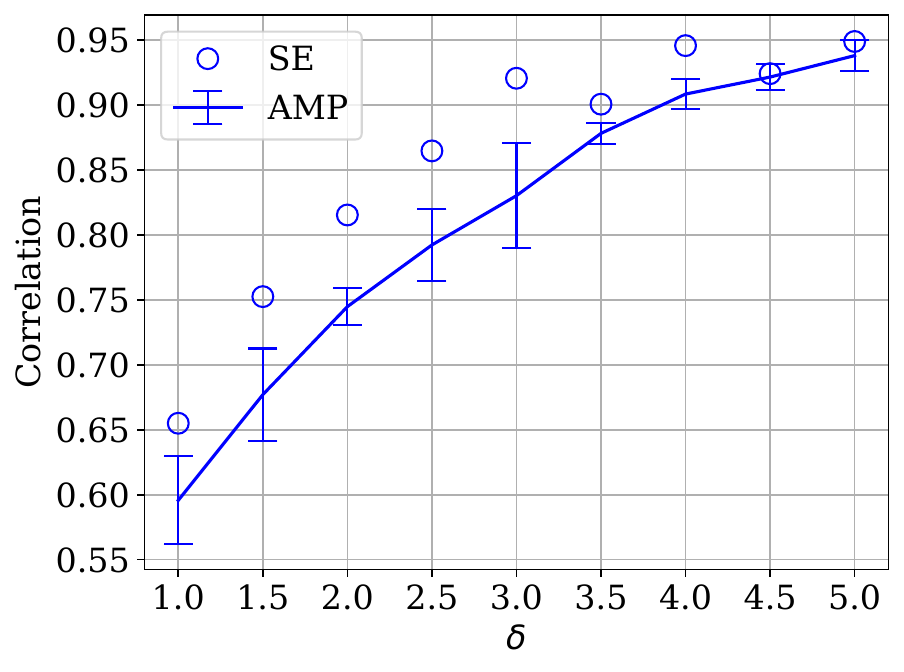}
  \caption{$\beta^{(1)}$}
\end{subfigure}
\begin{subfigure}[b]{0.45\textwidth}
  \centering
  \includegraphics[width=\textwidth]{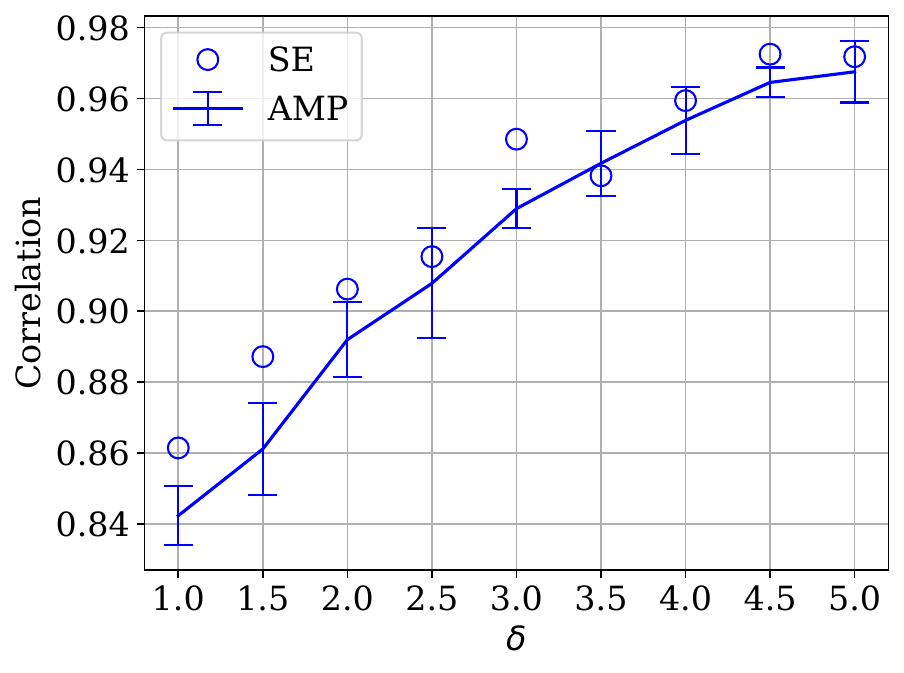}
  \caption{$\beta^{(2)}$}
\end{subfigure}
\begin{subfigure}[b]{0.45\textwidth}
  \centering
  \includegraphics[width=\textwidth]{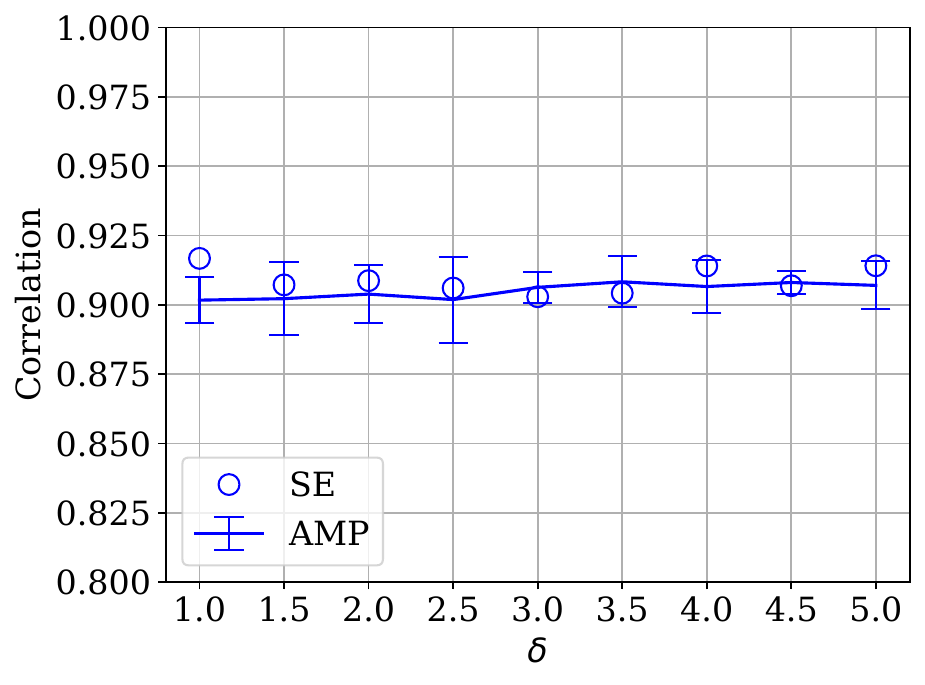}
  \caption{$w^{(1)}$}
\end{subfigure}
\begin{subfigure}[b]{0.45\textwidth}
  \centering
  \includegraphics[width=\textwidth]{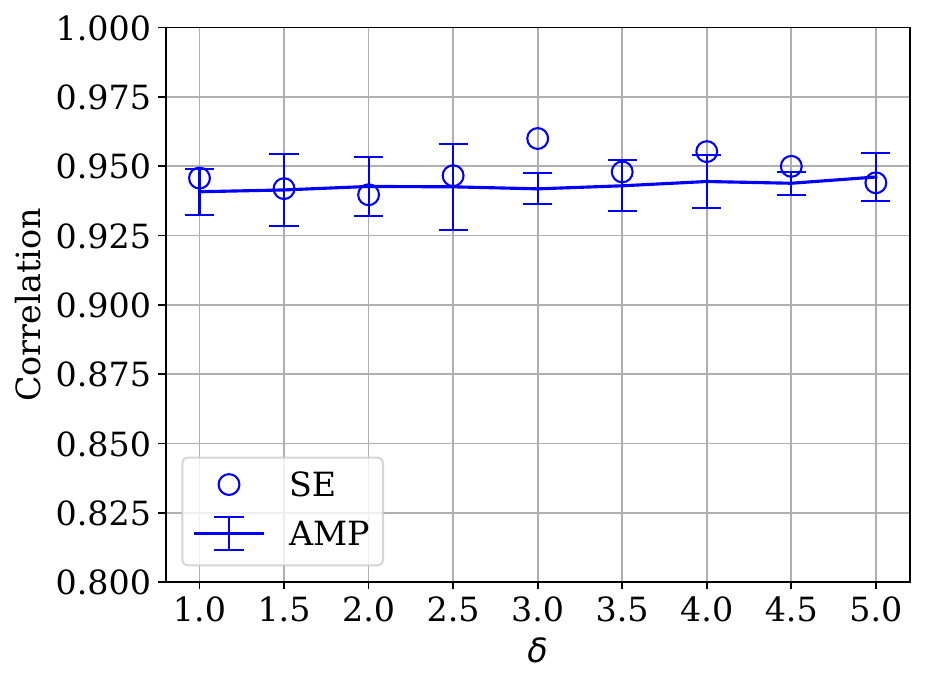}
  \caption{$w^{(2)}$}
\end{subfigure}
\caption{Mixture-of-experts: noise level $\sigma=0.1$.}
\label{fig:MOE_corr_v_delta}
\end{figure}

\subsection{Mixture-of-Experts} \label{sec:MOE}

We consider the MOE model \eqref{eq:MOE_def} with two regressors, two gating parameters, and the identity activation function $\tilde{q}(x)=x$. The model is given by
\begin{align}
    Y_i
    &=\langle X_i,\beta^{(1)}\rangle \, \mathds{1}\left\{\psi_i\leq\frac{\exp(\langle X_i,w^{(1)}\rangle)}{\exp(\langle X_i,w^{(1)}\rangle)+\exp(\langle X_i,w^{(2)}\rangle)}\right\} \nonumber \\
    &\qquad+\langle X_i,\beta^{(2)}\rangle \, \mathds{1}\left\{\psi_i>\frac{\exp(\langle X_i,w^{(1)}\rangle)}{\exp(\langle X_i,w^{(1)}\rangle)+\exp(\langle X_i,w^{(2)}\rangle)}\right\}+\epsilon_i,
\end{align}
where $\psi_i \sim_{\iid}\text{Uniform}[0,1]$, $\epsilon_i\sim_\iid\normal(0,\sigma^2)$, $X_i\sim_\iid\normal(0,I_p/n)$ for $i\in[n]$. 
The signal dimension  is set to $p=500$ and the value of $n$ is varied in our experiments.  
 We use a Gaussian prior for $B=(\beta^{(1)},\beta^{(2)},w^{(1)},w^{(2)})$, where we generate
\begin{align}
    (\beta_j^{(1)},\beta_j^{(2)},w_j^{(1)},w_j^{(2)})
    \sim_\iid
    \normal([1,2,3,4]^\top,I_4), \ \ j\in[p].
\end{align}

We run the AMP algorithm in \eqref{eq:GAMP} with  $g_k=g_k^*$ and $f_k=f_k^*$ (i.e., the optimal choices). 
The initializer $B^0\in\mb{R}^{p\times 4}$ is chosen according to the same distribution, independently of the signal.  The details of the implementation are given in Appendix \ref{appen:MOE_implementation}.
Figure \ref{fig:MOE_corr_v_delta} shows the performance of AMP for MOE. The performance in the plots is measured via the normalized squared correlation given in \eqref{eq:norm_sq_corr}. Each point on the plots is obtained from 5 independent runs, where in each run, we execute AMP with $k=5$. We report the average and error bars at 1 standard deviation of the final iteration. Figure \ref{fig:MOE_corr_v_delta} indicates a good match between the empirical performance of AMP and the theoretical state evolution predictions. The improvement across increasing $\delta$'s is more significant for the regressors than the gating parameters since the latter are easier to estimate because of their larger means.

\section{Proofs and Derivations} \label{sec:proofs}

\subsection{Proof of Theorem \ref{thm:GAMP}} \label{subsec:proof_GAMP}
 
To prove the theorem, we use a change of variables to rewrite \eqref{eq:GAMP} as a new matrix-valued AMP iteration. The new iteration is a special case of an abstract AMP iteration for which a state evolution result has been established by \cite{Jav13}. This state evolution result is then translated to obtain the results in \eqref{eq:matrix-GAMP_SE_result1}-\eqref{eq:matrix-GAMP_SE_result2}.

Given the iteration \eqref{eq:GAMP}, for $k \ge 0$ define
\begin{align}
    & \check{B}^{k+1}:=B^{k+1}-B(\Mu^{k+1}_{B})^\top,  \quad  \check{\Theta}^k
    :=(\Theta,\Theta^k),
  \label{checkB_checkT_def}
\end{align}
where we recall that $\Theta = X B$. For $k \ge 0$, we also define the function $\check{f}_k: \reals^{2L} \to \reals^{2L}$:
\begin{align}
\check{f}_k(\check{B}^{k},B)=(B, \ f_{k}(\check{B}^{k}+B(\Mu^{k}_{B})^\top)).
\label{eq:check_fk_def}
\end{align}
Then, we claim that the original AMP iteration \eqref{eq:GAMP} is equivalent to the following one:
\begin{equation}
\begin{split}
&        \check{\Theta}^k = X \check{f}_k(\check{B}^{k},B)  \,  -  \,  h_{k-1}(\check{\Theta}^{k-1}, \Psi) (\check{F}^k)^{\top} \\
& \check{B}^{k+1} = X^{\top} h_k(\check{\Theta}^k, \Psi) \,  -  \,
\check{f}_k(\check{B}^{k},B) (\check{C}^k)^{\top},
\end{split}
\label{eq:mod_AMP}
\end{equation}
where $h_k$ is defined in \eqref{eq:hk_def}, and the matrices  
$ \check{C}^k \in \reals^{L \times 2L} $, $\check{F}^{k+1}  \in \reals^{2L \times L}$ are defined as:
\begin{align}
\begin{split}
    & \check{C}^k = 
    \begin{bmatrix}
    \E[\partial_Z h_k(Z,Z^k,\bar{\Psi})],\ \ 
    \frac{1}{n}\sum_{i=1}^n\partial_{\Theta_i^k} h_k(\Theta_i,\Theta_i^k,\Psi_i)
    \end{bmatrix}  \\
    & \check{F}^{k+1} = \begin{bmatrix} 0_{L \times L} \\
    \frac{1}{n}\sum_{j=1}^p f_{k+1}'(\check{B}_j^k+B_j(\Mu^k_B)^\top)
    \end{bmatrix}.
\end{split}
\label{eq:check_CkFk1_def} 
\end{align}
 The iteration \eqref{eq:mod_AMP} is initialized with $\check{\Theta}^0 = (\Theta, \ X\hB^0)$, where $\hB^0$ is the initializer of the original AMP. The equivalence between the iteration in \eqref{eq:mod_AMP} and the original AMP in \eqref{eq:GAMP} can be seen by substituting the definitions \eqref{checkB_checkT_def} and \eqref{eq:check_fk_def} into \eqref{eq:mod_AMP}, and recalling from \eqref{eq:SE_Mk1B} that $\Mu^{k+1}_{B} = \E[\partial_Z h_k(Z,Z^k,\bar{\Psi})]$. 

A key feature of the new iteration in  \eqref{eq:mod_AMP} is that, in addition to the previous iterate, the inputs to the functions $\check{f}_k$ and $h_k$ are  auxiliary variables ($B, \Psi$, respectively) that are independent of $X$. This is in contrast to the AMP in \eqref{eq:GAMP} where the input $Y$ to the function $g_k$ is not independent of $X$. The recursion in \eqref{eq:mod_AMP} is a special case of an abstract AMP recursion with matrix-valued iterates for which a state evolution result has been established by \cite{Jav13}. (We will use a version of the result described in \cite[Sec. 6.7]{Fen21}).) The standard form of the abstract AMP recursion uses empirical estimates (instead of expected values) for the first $L$ columns of $\check{C}^k$ in \eqref{eq:check_CkFk1_def}. However,  the state evolution result still remains valid for the recursion \eqref{eq:mod_AMP} (see Remark 4.3 of \cite{Fen21}). This result, stated in Proposition \ref{prop:abst_AMP} below, guarantees that the empirical distributions of the rows of $\check{\Theta}^k$ and $\check{B}^{k+1}$ converge to the Gaussian distributions $\normal(0, \check{\Sigma}^{k})$  and  $\normal(0, \check{\Tau}^{k+1})$, respectively, where the deterministic covariance matrices $\check{\Sigma}^{k} \in \reals^{2L \times 2L}, \, \check{\Tau}^{k+1}  \in \reals^{L \times L}$ are defined by the following state evolution recursion. Let $\check{\Sigma}^{0} = \Sigma^0$ (defined in Assumption \textbf{(A1)}), and for $k \ge 0$:
\begin{align}
     \check{\Tau}^{k+1} &  = \E\big[h_k(G_\sigma^k,\bar{\Psi})h_k(G_\sigma^k,\bar{\Psi})^\top\big] \label{eq:check_Tau}\\
    \check{\Sigma}^{k+1} & =  \delta^{-1}\E\big[\check{f}_{k+1}(G^{k+1}_\tau,\bar{B})\check{f}_{k+1}(G^{k+1}_\tau,\bar{B})^\top\big]
    = 
    \begin{bmatrix}
    \delta^{-1} \E[\bar{B}\bar{B}^\top] &  \check{\Sigma}^{k+1}_{(12)} \\
   \big( \check{\Sigma}^{k+1}_{(12)} \big)^{\top} & \check{\Sigma}^{k+1}_{(22)}
    \end{bmatrix},
    \label{eq:check_SE_rec}
\end{align}
where  
\begin{align}
    & \check{\Sigma}^{k+1}_{(12)} = \left( \check{\Sigma}^{k+1}_{(21)} \right)^{\top} = \delta^{-1} \E\big[ \bar{B} f_{k+1}(G^{k+1}_\tau+\Mu_{k+1}^{B}\bar{B})^\top \big]  \\
    & \check{\Sigma}^{k+1}_{(22)} =  \delta^{-1} \E\big[ f_{k+1}(G^{k+1}_\tau+\Mu_{k+1}^{B}\bar{B})\cdot f_{k+1}(G^{k+1}_\tau+\Mu_{k+1}^{B}\bar{B})^\top \big].
\end{align}
Here we take $G_\sigma^k \sim N(0,\check{\Sigma}^k)$  independent of $\bar{\Psi}\sim P_{\bar{\Psi}}$, and $G^{k+1}_\tau \sim N(0,\check{\Tau}^{k+1})$ independent of $\bar{B}\sim P_{\bar{B}}$. 
Comparing the recursive definitions of $(\Tau_B^{k+1}, \Sigma^{k+1})$ in \eqref{eq:SE_Tk1B}-\eqref{eq:Sig_comps} and of $(\check{\Tau}^{k+1}, \check{\Sigma}^{k+1})$ in \eqref{eq:check_Tau}-\eqref{eq:check_SE_rec}, and noting that they are both initialized with $\Sigma^0$, we have that $\check{\Tau}^{k+1} = \Tau^{k+1}_B$ and $\check{\Sigma}^{k+1} = \Sigma^{k+1}$ for $k \ge 0$.

The following proposition follows from the state evolution result \citep[Sec.\ 6.7]{Fen21} for an abstract AMP recursion with  matrix-valued iterates. 
\begin{proposition}
\label{prop:abst_AMP} 
Assume the setting of Theorem \ref{thm:GAMP}. For the abstract AMP in \eqref{eq:mod_AMP}, for $k \ge 0$ we have:
\begin{align}
    &\sup_{\eta\in\textup{PL}_{2L}(r,1)}\Big|\frac{1}{p}\sum_{j=1}^{p}\eta( \check{B}_j^{k+1}, \, B_j)  -\E[\eta( G^{k+1}_\tau, \, \bar{B} )]\Big|\stackrel{c}{\rightarrow}0 ,\label{eq:abst_matrix-GAMP_SE_result1} \\
    &\sup_{\eta\in\textup{PL}_{2L + L_{\Psi}}(r,1)}\Big|\frac{1}{n}\sum_{i=1}^{n}\eta(\check{\Theta}^k_i,\Psi_i) -\E[\eta( G^{k}_\sigma,\bar{\Psi})]\Big|\stackrel{c}{\rightarrow}0, \label{eq:abst_matrix-GAMP_SE_result2}
\end{align}
as $n,p\rightarrow\infty$ with $n/p\rightarrow\delta$.
\end{proposition} 
 To obtain the result \eqref{eq:matrix-GAMP_SE_result1}, we recall the definition of $\check{B}^{k+1}$ from \eqref{checkB_checkT_def}, and in \eqref{eq:abst_matrix-GAMP_SE_result1} we take 
$\eta(\check{B}^{k+1}, B ) = c_{k,r} \phi(\check{B}^{k+1} + B(\Mu^{k+1}_B)^\top, B )$ for a suitably small constant $c_{k,r} >0$, and recall that $G^{k+1}_\tau \sim \normal(0, \Tau^{k+1}_B)$.  To obtain $\eqref{eq:matrix-GAMP_SE_result2}$,  we recall the definition of $\check{\Theta}^k$ from \eqref{checkB_checkT_def}, and in \eqref{eq:abst_matrix-GAMP_SE_result2} take $\eta(\check{\Theta}^k, \Psi) = \phi(\Theta^k, \Theta, \Psi)$. Since $\check{\Sigma}^k = \Sigma^k$,  we have: 
\begin{align}
    (G^{k}_\sigma,\bar{\Psi}) \stackrel{d}{=} (Z, Z^k, \bar{\Psi} )
\stackrel{d}{=} (Z, \, \Mu^k_{\Theta}Z +  G^k_{\Theta}, \bar{\Psi}),
\end{align}
where the last equality follows from \eqref{eq:ZZk_joint}. This completes the proof of the theorem. $\square$

\subsection{Proof of Proposition \ref{prop:optimal_fk}} \label{subsec:prop_proof}
The proof relies on the following  generalized Cauchy-Schwarz inequality for covariance matrices.
\begin{lemma}
\textup{\cite[Lemma 1]{Lav08}} \label{lem:CS_extension}
Let $U,V \in \reals^L$  random vectors such that $\E[\| U \|^2_2]<\infty$, $\E[\| V \|^2_2]<\infty$, and $\E[VV^\top]$ is invertible. Then
\begin{align}
    \E[U U^\top] - \E[U V^\top] \big( \E[V V^\top] \big)^{-1}\E[V U^\top] \succeq 0. \label{eq:CS_ineq}
\end{align}
\end{lemma}

\paragraph{Proof of part 1.}  Using the law of total expectation,  $\Sigma^k_{(12)}$ in \eqref{eq:Sig_comps} can be written as:
\begin{align}
    \delta \Sigma^k_{(12)}
    =\E[\bar{B}f_{k}(\Mu^{k}_{B}\bar{B}+G^{k}_B)^\top]  
    =\E\big[\E[\bar{B}f_{k}(\Mu^{k}_{B}\bar{B}+G^{k}_B)^\top\mid\Mu^{k}_{B}\bar{B}+G^{k}_B ]\big] 
    = \E[ f_{k}^* f_{k}^\top ], \label{eq:Sigk_12}
\end{align}
where we use the shorthand $f_k \equiv f_{k}(\Mu^{k}_{B}\bar{B}+G^{k}_B)$
and $f_k^* \equiv \E[\bar{B} \, | \, \Mu^{k}_{B}\bar{B}+G^{k}_B ]$. Using Lemma \ref{lem:CS_extension} we have that 
\begin{align}
    & \E[f_k^*(f_k^*)^\top]-\E[f_k^*f_k^\top]\E[f_kf_k^\top]^{-1}\E[f_k(f^*_k)^\top] \succeq 0  \nonumber \\
    & \implies \delta^{-1}\E[f_k^*(f_k^*)^\top]-\Sigma^k_{(12)}(\Sigma^k_{(22)})^{-1}\Sigma^k_{(21)} \succeq0, \label{eq:CS_ineq_gk}
\end{align}
 where we have used \eqref{eq:Sigk_12} and \eqref{eq:Sig_comps} for the second line. Adding and subtracting $\Tau^k_\Theta$ in \eqref{eq:CS_ineq_gk} we obtain
 \begin{align}
     & \Tau^k_{\Theta} - \big( \underbrace{\Tau^k_{\Theta} - \delta^{-1}\E[f_k^*(f_k^*)^\top] + \Sigma^k_{(12)}(\Sigma_{(22)}^k)^{-1}\Sigma_{(21)}^k}_{: = \Gamma^k_\Theta} \big)\succeq0.
     \label{eq:TkTh_ineq}
 \end{align}
Multiplying the matrix $(\Tau^k_{\Theta} - \Gamma^k_\Theta)$ in \eqref{eq:TkTh_ineq} by $\left( \Mu^{k}_{\Theta}\right)^{-1}$ on the left and $\left( \big(\Mu^{k}_{\Theta}\big)^{-1} \right)^\top$ on the right maintains positive definiteness. This yields
\begin{align}
    N^k_{\Theta}  - \left(\Mu^{k}_{\Theta}\right)^{-1} \Gamma^k_{\Theta} \left( \big(\Mu^{k}_{\Theta}\big)^{-1} \right)^\top \succeq 0, 
    \label{eq:Nk_Th_posdef}
\end{align}
where we have used the formula for $N^k_{\Theta}$ from \eqref{eq:eff_cov_def}. Eq. \eqref{eq:Nk_Th_posdef} implies 
\begin{align}
\Tr(N^k_{\Theta}) \ge \Tr\left( \left(\Mu^{k}_{\Theta}\right)^{-1} \Gamma^k_{\Theta} \left( \big(\Mu^{k}_{\Theta}\big)^{-1} \right)^\top \right). \label{eq:Tr_NkTh_LB}
\end{align}
Now, using the formula for $\Tau^k_\Theta$ in \eqref{eq:SE_TkTh} it can be verified that when $f_k =f_k^*$, we have
\begin{align}
    \Tau^k_\Theta = \Gamma^k_\Theta 
    = \frac{1}{\delta}\Big( \E\big[f_k^*  (f_k^*)^{\top}\big] - \E\big[f_k^*  (f_k^*)^{\top}\big] \big(\E\big[ \bar{B} \bar{B}^{\top} \big] \big)^{-1} \E\big[f_k^*  (f_k^*)^{\top}\big]  \Big). \label{eq:Tauk_fkstar}
\end{align}
Therefore \eqref{eq:TkTh_ineq}-\eqref{eq:Tr_NkTh_LB} are satisfied with equality when $f_k = f_k^*$, which proves the first part of the proposition.

\paragraph{Proof of part 2.} We begin by introducing the multivariate Stein's lemma:
\begin{lemma} \label{lem:gen_steins}
Let $x=(x_1,\dots,x_L)$ and $g:\mb{R}^L\rightarrow\mb{R}^L$ be such that for $j=1,\dots,L$, the function $x_j\rightarrow g_l(x_1,\dots,x_L)$ (where $g_l(x_1,\dots,x_L)$ is the $l$th entry of $g(x_1,\dots,x_L)$) is absolutely continuous for Lebesgue almost every $(x_i:i\neq j)\in\mb{R}^{L-1}$, with weak derivative $\partial_{x_j} g_l:\mb{R}^L\rightarrow\mb{R}$ satisfying $\E[|\partial_{x_j}g_l(x)|]<\infty$. Let $\nabla g(x)=(\nabla g_1(x),\dots,\nabla g_L(x))^\top\in\mb{R}^{L\times L}$ where $\nabla g_l(x)=\big(\partial_{x_1}g_l(x),\dots,\partial_{x_L}g_l(x)\big)^\top$ for $x\in\mb{R}^L$. If $X\sim \normal(\mu,\Sigma)$ with $\Sigma$ positive definite, then
\begin{align}
    \E[\nabla g(X)]=\Big(\Sigma^{-1}\E\big[(X-\mu)g(X)^\top\big]\Big)^\top.
\end{align}
\end{lemma}
\begin{proof}
We have
\begin{equation}
\begin{split}
     \E[(X-\mu)g(X)^\top]
    &=\big(\E[(X-\mu)g_1(X)],\dots,\E[(X-\mu)g_L(X)]\big) \\
    &\stackrel{(a)}{=}(\Sigma\E[\nabla g_1(X)],\dots,\Sigma\E[\nabla g_L(X)]) \\
    &=\Sigma \E[(\nabla g_1(X),\dots,\nabla g_L(X))] \\
    &\stackrel{(b)}{=}\Sigma \E[\nabla g(X)]^\top,
\end{split}
\end{equation}
where (a) uses the multivariate Stein's Lemma from \cite[Lemma 6.20]{Fen21} which states that under our conditions we have $\E[Xg_l(X)]=\Sigma\E[\nabla g_l(X)]$ for $l=1,\dots,L$, and (b) uses the definition of $\nabla g(x)$. Finally, rearranging the above equation and taking the transpose gives the result.
\end{proof}
Next, we use Lemma \ref{lem:gen_steins} to show that 
\begin{align}
    \Mu^{k+1}_B=\E\big[g_k(Z^k,\bar{Y})g_k^*(Z^k,\bar{Y})^\top\big],
    \label{eq:Mk1_B_alt}
\end{align}
where $g_k^*$ is defined in \eqref{eq:gk_opt_def}. Indeed, using the law 
of total expectation we have
\begin{equation}
\begin{split}
    \Mu^{k+1}_B 
    &= \E\Big[\E[\partial_{Z}h_k(Z,Z^k,\bar{\Psi})|Z^k]\Big] \\
    &\stackrel{(a)}{=}\E\Big[\Cov[Z|Z^k]^{-1}\E\big[(Z-\E[Z|Z^k])h_k(Z,Z^k,\bar{\Psi})^\top\big\lvert \,  Z^k \big]\Big]^\top \\
    & = \E[\Cov[Z|Z^k]^{-1}(Z-\E[Z|Z^k])h_k(Z,Z^k,\bar{\Psi})^\top]^\top \\
    & = \E\Big[\E\big[\Cov[Z|Z^k]^{-1}(Z-\E[Z|Z^k])h_k(Z,Z^k,\bar{\Psi})^\top \big\lvert \, Z^k,\bar{Y}\big]\Big]^\top \\
    & \stackrel{(b)}{=} \E\Big[g_k^*(Z^k,\bar{Y})h_k(Z,Z^k,\bar{\Psi})^\top\Big]^\top \\
    & \stackrel{(c)}{=} \E\big[g_k(Z^k,\bar{Y})g_k^*(Z^k,\bar{Y})^\top\big]. \label{eq:deriv_G_w.r.t._Z}
    \end{split}
\end{equation}
Here (a) applies Lemma \ref{lem:gen_steins}, (b) follows from the definition of $g_k^*$ in \eqref{eq:gk_opt_def}, and (c) from \eqref{eq:hk_def}.  Using the shorthand $g_k \equiv g_k(Z^k,\bar{Y})$ and 
$g_k^* \equiv g_k^*(Z^k,\bar{Y})$, 
from Lemma \ref{lem:CS_extension} we have:
\begin{align}
   \E[ g_k^* (g_k^*)^{\top} ]   -  \E[g_k^* g_k^{\top}] 
  \Big( \E[  g_k g_k^\top] \Big)^{-1}    \E[ g_k (g_k^*)^{\top}] 
    \succeq 0  
     \Leftrightarrow &  \ 
    \E[ g_k^* (g_k^*)^{\top} ]   -  \left( N^{k+1}_B \right)^{-1} \succeq 0
    \label{eq:N_k1_B_inv} \\
      \Leftrightarrow &  \ 
    \Big( \E[ g_k^* (g_k^*)^{\top} ] \Big)^{-1} - N^{k+1}_B \preceq 0,
    \label{eq:N_k1_B}
\end{align}
where \eqref{eq:N_k1_B_inv} is obtained by recalling from \eqref{eq:eff_cov_def} that $(N^{k+1}_B)^{-1} = \Mu^{k+1}_{B}\left( \Tau^{k+1}_{B} \right)^{-1} \big(\Mu^{k+1}_{B}\big)^\top$, and using the expressions for $\Mu^{k+1}_B$ and $\Tau^{k+1}_B$ in  \eqref{eq:Mk1_B_alt} and \eqref{eq:SE_Tk1B}. Eq.~\eqref{eq:N_k1_B} follows from the fact that if $P$ and $Q$ are positive definite matrices such that $P-Q \succeq0$, then $P^{-1}-Q^{-1}\preceq 0$. From \eqref{eq:N_k1_B}, we have that 
\begin{align} \Tr(N^{k+1}_B) \le  \Tr\left(\big( \E[ g_k^* (g_k^*)^{\top} ] \big)^{-1} \right),
\end{align}
with equality if $g_k = g_k^*$. This completes the proof of the second part of the proposition. $\square$

\subsection{Derivation of the EM Step  for Max-Affine Regression} \label{sec:derivation_EM_step}

In this section, we derive steps 4 and 5 of EM-AMP for max-affine regression (Algorithm \ref{alg:EM-AMP}). We follow an approach similar to the one in \cite{Vil13}, where EM was combined with AMP for compressed sensing with unknown parameters in the signal prior. We adapt their derivation to the MAR model. Recall that 
\begin{align}
    Y_i\
    &=\max\big\{\big\langle X_i,\beta^{(1)}\big\rangle+b_1,\big\langle X_i,\beta^{(2)}\big\rangle+b_2\big\}+\epsilon_i  \nonumber \\
    &=\max\big\{\Theta_i^{(1)}+b_1,\Theta_i^{(2)}+b_2\big\}+\epsilon_i,
    \qquad  i\in[n].
\end{align}
Here $\Theta^{(1)}$ and $\Theta^{(2)}$ are the first and second columns of $\Theta\in\mb{R}^{n\times 2}$ respectively,  and $\Theta_i:=\big(\Theta_i^{(1)},\Theta_i^{(2)}\big) \in \reals^2$. 

The parameter that we would like to estimate using EM is $b=(b_1,b_2)$. Before providing the derivation, we briefly review the main idea behind the EM algorithm. The EM algorithm iteratively produces estimates $b^m \equiv (b_1^m,b_2^m)$ for $m \ge 1$, with the goal of increasing the likelihood $p(Y;b)$ at each iteration, where $Y=[Y_1,\dots,Y_n]^\top$. It achieves this by iteratively increasing a lower bound on $p(Y;b)$, thus guaranteeing that the likelihood converges to a local maximum or at least a saddle point \citep{Wu83}. In our case, for an arbitrary distribution $\hat{p}$ on 
$(\Theta^{(1)},\Theta^{(2)})$  we have
\begin{align}
    \log p(Y;b) 
    &=\int_{\Theta{(1)}}\int_{\Theta^{(2)}} \hat{p}(\Theta^{(1)},\Theta^{(2)}) d\Theta^{(1)}d\Theta^{(2)}\log p(Y;b)\nonumber \\
    &=\int_{\Theta^{(1)}}\int_{\Theta^{(2)}} \hat{p}(\Theta^{(1)},\Theta^{(2)}) d\Theta^{(1)}d\Theta^{(2)}\log \bigg(\frac{p(\Theta^{(1)},\Theta^{(2)},Y;b)}{\hat{p}(\Theta^{(1)},\Theta^{(2)})}\cdot\frac{\hat{p}(\Theta^{(1)},\Theta^{(2)})}{p(\Theta^{(1)},\Theta^{(2)}|Y;b)}\bigg) \nonumber \\
    &\stackrel{(a)}{=}\E_{\Theta^{(1)},\Theta^{(2)} \sim \hat{p}}\big[\log p(\Theta^{(1)},\Theta^{(2)},Y;b)\big]+H(\hat{p})+D(\hat{p}\,\|\,p(\cdot|Y;b)) \label{eq:EM_mid_step}  \\
    &\stackrel{(b)}{\geq}\E_{\Theta^{(1)},\Theta^{(2)} \sim \hat{p} }\big[\log p(\Theta^{(1)},\Theta^{(2)},Y;b)\big]+H(\hat{p})
    :=\mathcal{L}(Y;b), \nonumber
\end{align}
where in (a), $H(\cdot)$ is the Shannon entropy and $D(\cdot\|\cdot)$ the Kullback-Leibler (KL) divergence, and the inequality (b) follows from  the non-negativity of the KL divergence. The EM-algorithm at step $m+1$ iterates over two steps:
\begin{itemize}
    \item In the E-step, we choose $\hat{p}$ to maximize the lower bound $\mathcal{L}(Y;b)$ for fixed $b=b^m$,
    \item In the M-step, we choose $b$ to maximize the lower bound $\mathcal{L}(Y;b)$ for fixed $\hat{p}=\hat{p}^m$.
\end{itemize}
For the E-step, since $\mathcal{L}(Y;b^m)=\log p(Y;b^m)-D(\hat{p}\,\|\,p(\cdot|Y;b^m))$ (via rearranging \eqref{eq:EM_mid_step}), the maximizing probability density function (pdf) would be $\hat{p}^{m}=p(\Theta^{(1)},\Theta^{(2)}|Y;b^m)$. Then, for the M-step, from the definition of $\mathcal{L}(Y;b)$, the maximizing $b$ is: 
\begin{align}
    b^{m+1}=\argmax_{b\in\mb{R}^2} \, \E_{\Theta^{(1)},\Theta^{(2)} \sim p(\Theta^{(1)},\Theta^{(2)}|Y; \, b^m) }\big[\log p\big(\Theta^{(1)},\Theta^{(2)},Y;b\big)\big].
    \label{eq:EM_argmax}
\end{align}
We can further expand the $p(\cdot)$ above as
\begin{align}
    p\big(\Theta^{(1)},\Theta^{(2)},Y;b\big)
    = p\big(\Theta^{(1)},\Theta^{(2)}\big) p\big(Y|\Theta^{(1)},\Theta^{(2)};b\big)
    = p\big(\Theta^{(1)},\Theta^{(2)}\big) \prod_{i=1}^n p(Y_i|\Theta_i;b),
    \label{eq:p_expand}
\end{align}
where $p\big(\Theta^{(1)},\Theta^{(2)}\big)= \prod_{i=1}^n p\big(\Theta^{(1)}_i,\Theta^{(2)}_i\big) $ does not depend on $b$ (recall that $\big(\Theta^{(1)},\Theta^{(2)}\big) = X B$). It is challenging to jointly optimize over $b=(b_1, b_2)$ in \eqref{eq:EM_argmax}, so we update one component of $b$  at a time while holding the other fixed. This is a well known ``incremental'' variant of EM  \citep{Nea98}. Using \eqref{eq:p_expand} and \eqref{eq:EM_argmax}, the updated $b_1$ estimate is
\begin{align}
    b_1^{m+1}
    &=\argmax_{b_1\in\mb{R}}\E\Big[\log\prod_{i=1}^np(Y_i|\Theta_i;b_1,b_2^m) \mid Y;b^m\Big]  \nonumber \\
    &=\argmax_{b_1\in\mb{R}}\sum_{i=1}^n\E\Big[\log p(Y_i|\Theta_i;b_1,b_2^m) \mid Y;b^m \Big] \nonumber  \\
    &=\argmax_{b_1\in\mb{R}}\sum_{i=1}^n\int_{\Theta_i}p(\Theta_i|Y;b^m)\log p(Y_i|\Theta_i;b_1,b_2^m)d\Theta_i \, ,
\end{align}
where we recall that $b^m=(b_1^m,b_2^m)$. At this point, note that
\begin{align}
    p(Y_i|\Theta_i;b_1,b_2^m)\sim \normal\Big(\max\big\{\Theta_i^{(1)}+b_1,\Theta_i^{(2)}+b_2^m\big\},\sigma^2\Big).
    \label{eq:PYi_Th_b1}
\end{align}
To get $b_1^{m+1}$, we need to solve
\begin{align}
    &\frac{\partial}{\partial b_1}\sum_{i=1}^n\int_{\Theta_i}p(\Theta_i|Y;b^m)\log p(Y_i|\Theta_i;b_1,b_2^m)d\Theta_i=0  \nonumber \\
    \iff&\sum_{i=1}^n\int_{\Theta_i}p(\Theta_i|Y;b^m)\frac{\partial}{\partial b_1}\log p(Y_i|\Theta_i;b_1,b_2^m)d\Theta_i=0, \label{eq:deriv_eq}
\end{align}
where we used Leibniz's integral rule to exchange differentiation and integration. Using \eqref{eq:PYi_Th_b1}, the derivative in \eqref{eq:deriv_eq} is
\begin{align}
    \frac{\partial}{\partial b_1}\log p(Y_i|\Theta_i;b_1,b_2^m)
   &=\frac{\partial}{\partial b_1}\log\bigg(\frac{1}{\sigma\sqrt{2\pi}}\exp\bigg(-\frac{1}{2}\bigg(\frac{Y_i-\max\{\Theta_i^{(1)}+b_1,\Theta_i^{(2)}+b_2^m\}}{\sigma}\bigg)^2\bigg)\bigg) \nonumber \\
    &=
    \begin{cases}
    \frac{1}{\sigma^2}(Y_i-\Theta_i^{(1)}-b_1)&\text{ if }\Theta_i^{(1)}+b_1>\Theta_i^{(2)}+b_2^m \\
    0&\text{ if }\Theta_i^{(1)}+b_1\leq\Theta_i^{(2)}+b_2^m
    \end{cases} \, .
\end{align}
Substituting the above back into \eqref{eq:deriv_eq}, we get
\begin{align}
    &\sum_{i:\Theta_i^{(1)}+b_1>\Theta_i^{(2)}+b_2^m}\int_{\Theta_i}p_{Z|Y}(\Theta_i|Y;b^m)\frac{1}{\sigma^2}(Y_i-\Theta_i^{(1)}-b_1)d\Theta_i=0  \nonumber \\
    \iff&\sum_{i:\Theta_i^{(1)}+b_1>\Theta_i^{(2)}+b_2^m}\bigg(Y_i\int_{\Theta_i}p_{Z|Y}(\Theta_i|Y;b^m)d\Theta_i-\int_{\Theta_i}p_{Z|Y}(\Theta_i|Y;b^m)\Theta_i^{(1)}d\Theta_i \nonumber \\
    &\qquad-b_1\int_{\Theta_i}p_{Z|Y}(\Theta_i|Y;b^m)d\Theta_i\bigg)=0 \nonumber \\
    \iff&\sum_{i:\Theta_i^{(1)}+b_1>\Theta_i^{(2)}+b_2^m}\Big(Y_i-\E[\Theta_i^{(1)}|Y;b^m]-b_1\Big)=0.
\end{align}
Rearranging gives
\begin{align}
    b_1
    &=\frac{1}{|\{i:\Theta_i^{(1)}+b_1>\Theta_i^{(2)}+b_2^m\}|}\sum_{i:\Theta_i^{(1)}+b_1>\Theta_i^{(2)}+b_2^m}\Big(Y_i-\E\big[\Theta_i^{(1)}|Y;b^m\big]\Big) \nonumber  \\
    &\approx\frac{1}{|\{i:\Theta_i^{(1)}+b_1>\Theta_i^{(2)}+b_2^m\}|}\sum_{i:\Theta_i^{(1)}+b_1>\Theta_i^{(2)}+b_2^m}Y_i-\E\big[Z_1|\bar{Y};b^m\big], \label{eq:b1_approx}
\end{align}
where we approximate $\E\big[\Theta_i^{(1)}|Y;b^m\big]$ by  $\E\big[Z_1|\bar{Y};b^m\big]$ because computing $\E\big[\Theta_i^{(1)}|Y;b^m\big]$ is intractable. The computation of 
$\E\big[Z_1|\bar{Y};b^m\big]$ is detailed in Appendix \ref{appen:MAR_implementation}.

Note that in \eqref{eq:b1_approx} it is not possible to compute $b_1$ on the LHS while using it in the RHS. An easy (but admittedly non-principled) fix is to just use $b_1^m$ on the RHS. This gives the update:
\begin{align}
    b_1^{m+1}:=\frac{1}{|\{i:\Theta_i^{(1)}+b_1^m>\Theta_i^{(2)}+b_2^m\}|}\sum_{i:\Theta_i^{(1)}+b_1^m>\Theta_i^{(2)}+b_2^m}Y_i-\E\big[Z_1|\bar{Y};b^m\big],
    \label{eq:b1m1_ideal}
\end{align}
where $\Theta_i$ and $\E\big[Z_1|Y;b^m\big]$ can be obtained from AMP in the previous iteration. Similarly, the update for the other intercept is:
\begin{align}
    b_2^{m+1}:=\frac{1}{|\{i:\Theta_i^{(1)}+b_1^m\leq \Theta_i^{(2)}+b_2^m\}|}\sum_{i:\Theta_i^{(1)}+b_1^m\leq\Theta_i^{(2)}+b_2^m}Y_i-\E\big[Z_2|\bar{Y};b^m\big].
    \label{eq:b2m1_ideal}
\end{align}
Since $\Theta^{(1)}, \Theta^{(2)}$ are unknown, we approximate them using AMP iterates. Specifically, the two columns of  $\widehat{\Theta}^{m}=X\widehat{B}^{k_{\max},m}$ provide estimates of $\Theta^{(1)}, \Theta^{(2)}$, respectively. Using these in \eqref{eq:b1m1_ideal} and \eqref{eq:b2m1_ideal} yields Steps 4 and 5 of the EM-AMP in Algorithm \ref{alg:EM-AMP}.


\acks{N.~Tan was supported by the Cambridge Trust and the Harding Distinguished Postgraduate Scholars Programme Leverage Scheme.}


\newpage

\appendix

\section{Implementation Details for MLR} \label{appen:MLR_implementation}

In this appendix, we  consider MLR with two signals and provide the implementation details of matrix-AMP with Bayes-optimal functions (see Proposition \ref{prop:optimal_fk}), for the Gaussian prior and the sparse discrete prior. While the implementation details stated here are for MLR with two signals, it is straightforward to generalize them to the case of three signals, which we have omitted.

\subsection{Gaussian prior} \label{sec:imp_normal_prior}

We start by writing the matrix-AMP algorithm  in  \eqref{eq:GAMP}-\eqref{eq:CkF_k1_def} with more details:

\begin{itemize}
    \item Initialize $\widehat{R}^{-1}=0\in\mathbb{R}^{n\times 2}$, $F_0=I_2$. Next, we initialize the rows of $B^0$ and $\hB^0$ independently using the jointly Gaussian prior.  
    Letting $\bar{B}=(\bar{\beta}^{(1)},\bar{\beta}^{(2)}) \in \reals^2$ be a random variable distributed according to the jointly Gaussian prior, we  initialize:
    \begin{align}
        B_j^0,\widehat{B}_j^0&
        \sim_{\iid} \normal\left(
        \begin{bmatrix}
        \E[\bar{\beta}^{(1)}] \\
        \E[\bar{\beta}^{(2)}]
        \end{bmatrix},
        \begin{bmatrix}
        \Var[\bar{\beta}^{(1)}] & \Cov[\bar{\beta}^{(1)},\bar{\beta}^{(2)}] \\
        \Cov[\bar{\beta}^{(2)},\bar{\beta}^{(1)}] & \Var[\bar{\beta}^{(2)}]
        \end{bmatrix}
        \right), \quad j\in[p],
    \end{align}
    and
    \begin{align}
        \Sigma^0&=\frac{p}{n}
        \begin{bmatrix}
        \E[(\bar{\beta}^{(1)})^2] & \E[\bar{\beta}^{(1)}\bar{\beta}^{(2)}] & (\E[\bar{\beta}^{(1)}])^2 & 
        \E[\bar{\beta}^{(1)}] \E[\bar{\beta}^{(2)}] \\
        \E[\bar{\beta}^{(1)}\bar{\beta}^{(2)}] & \E[(\bar{\beta}^{(2)})^2] & \E[\bar{\beta}^{(1)}] \E[\bar{\beta}^{(2)}] & (\E[\bar{\beta}^{(2)}])^2 \\
        (\E[\bar{\beta}^{(1)}])^2 & 
        \E[\bar{\beta}^{(1)}]\E[\bar{\beta}^{(2)}] & \E[(\bar{\beta}^{(1)})^2] & \E[\bar{\beta}^{(1)}\bar{\beta}^{(2)}] \\
        \E[\bar{\beta}^{(1)}]\E[\bar{\beta}^{(2)}] & (\E[\bar{\beta}^{(2)}])^2 & \E[\bar{\beta}^{(1)}\bar{\beta}^{(2)}] & \E[(\bar{\beta}^{(2)})^2]
        \end{bmatrix}.
    \end{align}
    \item For each iteration of matrix-AMP $k\in\mb{N}_0$, we have the following steps: \label{list:7_steps}
    \begin{enumerate}
        \item Compute $\Theta^k:=X\widehat{B}^k-\widehat{R}^{k-1}F_k^\top$
        \item Compute $\widehat{R}^k:=g_k(\Theta^k,Y)$
        \item Approximate $C^k:=\frac{1}{n}\sum_{i=1}^ng_k'(\Theta_i^k,Y_i)$
        \item Compute $B^{k+1}:=X^\top\widehat{R}^k-\widehat{B}^kC_k^\top$
        \item Approximate $\widehat{B}^{k+1}:=f_{k+1}(B^{k+1})$
        \item Approximate $F^{k+1}:=\frac{1}{n}\sum_{j=1}^pf_{k+1}'(B_j^{k+1})$
        \item Approximate $\Sigma^{k+1}$
    \end{enumerate}
\end{itemize}

The quantities in steps 1, 2, and 4  can be directly computed. The other steps  require some form of numerical approximation (based on  limiting properties of the iterates) to make the computation tractable. We now explain  how the quantities in steps 2, 3, 5, 6, and 7 can be computed or approximated.

\underline{\textbf{Step 2:}} We assume that $\Sigma^k$ has been approximated in the previous iteration. From Proposition \ref{prop:optimal_fk}, for MLR the function $g_k: \reals^2 \times \reals \to \reals$ is given by
 \begin{align}
     g_k(Z^k,\bar{Y})
    =\text{Cov}[Z|Z^k]^{-1}(\mathbb{E}[Z|Z^k,\bar{Y}]-\mathbb{E}[Z|Z^k]).
    \label{eq:gk_MLR}
 \end{align}
 Since $\begin{bmatrix} Z \\ Z^k \end{bmatrix} \sim \normal(0, \Sigma^k)$, using standard properties of Gaussian random vectors we have:
     \begin{align}
    \text{Cov}[Z|Z^k]=\Sigma^k_{(11)}-\Sigma^k_{(12)}(\Sigma^k_{(22)})^{-1}\Sigma^k_{(21)}, \qquad 
    \mathbb{E}[Z|Z^k]=\Sigma^k_{(12)}(\Sigma_{(22)}^k)^{-1}Z^k.
    \label{eq:cov_ZZk}
    \end{align}

To compute $\mathbb{E}[Z|Z^k,\bar{Y}]$, we recall that $Z=(Z_1, Z_2)^\top$  and let $\bar{\Psi} = (\bar{c}, \bar{\epsilon} )$, with $\bar{c} \sim \text{Bernoulli}(\alpha)$ and $\bar{\epsilon} \sim \normal(0, \sigma^2)$ independent.  Then,
    \begin{align}
        Y=q(Z,\bar{\Psi})=Z_1 \bar{c}+Z_2(1-\bar{c})+\bar{\epsilon},
   \label{eq:MLR_opeqn}
    \end{align}
    using which we have that
    \begin{align}
        \E[Z|Z^k,\bar{Y}]
        &= \E[Z|Z^k,\bar{Y},\bar{c}=1] \mb{P}[\bar{c}=1|Z^k,\bar{Y}] \,  + \, \E[Z|Z^k,\bar{Y},\bar{c}=0] \mb{P}[\bar{c}=0|Z^k,\bar{Y}].
        \label{eq:EZZ_kY_exp}
    \end{align}
    We now show how each of the four terms in \eqref{eq:EZZ_kY_exp} can be computed. 
    
 We first find the joint distribution of $(Z,Z^k,\bar{Y}|\bar{c}=1)$, which from \eqref{eq:MLR_opeqn} is jointly Gaussian.  We denote this distribution by $\normal(\boldsymbol{0},\Sigma^{k,1}_Y)$, and proceed to derive $\Sigma^{k,1}_Y \in \reals^{5 \times 5}$. Given a matrix $M \in \reals^{n_1 \times n_2}$, will use the notation
 $M_{[a],[b]}$ to denote the submatrix consisting of the first $a$ rows and first $b$ columns of $M$, and $M_{[a^+],[b^+]}$ to denote the submatrix with rows $\{ a, \ldots, n_1 \}$ and  columns $\{ b, \ldots, n_2 \}$ .
 
 We know from the joint distribution of $(Z,Z^k)$ that $(\Sigma^{k,1}_Y)_{[4],[4]}=\Sigma^k$. Hence, we only need to determine the remaining entries:
    \begin{equation}
    \begin{split}
        (\Sigma^{k,1}_Y)_{5,5}
        &=\Var[\bar{Y}\mid \bar{c}=1]
        =\Var[Z_1+\bar{\epsilon}]
        =\Sigma^k_{11}+\sigma^2, \\
        (\Sigma^{k,1}_Y)_{1,5}
        &=(\Sigma^{k,1}_Y)_{5,1}
        =\Cov[\bar{Y},Z_1|\bar{c}=1]
        =\Cov[Z_1+\bar{\epsilon},Z_1]
        =\Sigma^k_{11}, \\
        (\Sigma^{k,1}_Y)_{2,5}
        &=(\Sigma^{k,1}_Y)_{5,2}
        =\Cov[\bar{Y},Z_2|\bar{c}=1]
        =\Cov[Z_1+\bar{\epsilon},Z_2]
        =\Sigma^k_{12}, \\
        (\Sigma^{k,1}_Y)_{1,3}
        &=(\Sigma^{k,1}_Y)_{3,1}
        =\Cov[\bar{Y},Z^k_1|\bar{c}=1]
        =\Cov[Z_1+\bar{\epsilon},Z^k_1]
        =\Sigma^k_{13}, \\
        (\Sigma^{k,1}_Y)_{1,4}
        &=(\Sigma^{k,1}_Y)_{4,1}
        =\Cov[\bar{Y},Z^k_2|\bar{c}=1]
        =\Cov[Z_1+\bar{\epsilon},Z^k_2]
        =\Sigma^k_{14},
        \end{split}
        \label{eq:Sig_k1_last_row}
    \end{equation}
    where we have used the fact that $(Z,Z^k)$ and $\bar{\epsilon}$ are independent, and the notation $\Sigma^k_{ij}$ refers to the $(i,j)$-th entry of the matrix $\Sigma^k$. This gives
    \begin{align}
        \Sigma^{k,1}_Y=
        \begin{bmatrix}
        \Sigma^k_{11} & \Sigma^k_{12} & \Sigma^k_{13} & \Sigma^k_{14} & \Sigma^k_{11} \\
        \Sigma^k_{21} & \Sigma^k_{22} & \Sigma^k_{23} & \Sigma^k_{24} & \Sigma^k_{21} \\
        \Sigma^k_{31} & \Sigma^k_{32} & \Sigma^k_{33} & \Sigma^k_{34} & \Sigma^k_{31} \\
        \Sigma^k_{41} & \Sigma^k_{42} & \Sigma^k_{43} & \Sigma^k_{44} & \Sigma^k_{41} \\
        \Sigma^k_{11} & \Sigma^k_{12} & \Sigma^k_{13} & \Sigma^k_{14} & \Sigma^k_{11}+\sigma^2
        \end{bmatrix}.
    \end{align}
    From the joint distribution, we can compute
    \begin{align}
        \E[Z|Z^k,\bar{Y},\bar{c}=1]
        =(\Sigma^{k,1}_Y)_{[2],[3^+]}(\Sigma^{k,1}_Y)_{[3^+],[3^+]}^{-1}
        \begin{bmatrix}
        Z^k \\
        \bar{Y}
        \end{bmatrix},
        \label{eq:EZZKc1}
    \end{align}
    where $[3^+]:=\{3,4,5\}$. Using the same approach,  we can determine the joint distribution of $(Z,Z^k,\bar{Y}|\bar{c}=0)$ as  $\normal(\bzero, \Sigma^{k,0}_Y)$, where
    \begin{align}
        \Sigma^{k,0}_Y=
        \begin{bmatrix}
        \Sigma^k_{11} & \Sigma^k_{12} & \Sigma^k_{13} & \Sigma^k_{14} & \Sigma^k_{12} \\
        \Sigma^k_{21} & \Sigma^k_{22} & \Sigma^k_{23} & \Sigma^k_{24} & \Sigma^k_{22} \\
        \Sigma^k_{31} & \Sigma^k_{32} & \Sigma^k_{33} & \Sigma^k_{34} & \Sigma^k_{32} \\
        \Sigma^k_{41} & \Sigma^k_{42} & \Sigma^k_{43} & \Sigma^k_{44} & \Sigma^k_{42} \\
        \Sigma^k_{21} & \Sigma^k_{22} & \Sigma^k_{23} & \Sigma^k_{24} & \Sigma^k_{22}+\sigma^2
        \end{bmatrix}.
     \end{align}
    From this joint distribution, we can compute
    \begin{align}
        \E[Z|Z^k,\bar{Y},\bar{c}=0]
        =(\Sigma^{k,0}_Y)_{[2],[3^+]}(\Sigma^{k,0}_Y)_{[3^+],[3^+]}^{-1}
        \begin{bmatrix}
        Z^k \\
        \bar{Y}
        \end{bmatrix}.
        \label{eq:EZZKc0}
    \end{align}
    
The first conditional probability term in \eqref{eq:EZZ_kY_exp} can be computed as:
    \begin{align}
        \mb{P}[\bar{c}=1|Z^k,\bar{Y}]
        &=\frac{\mb{P}[\bar{c}=1]\mb{P}[Z^k,\bar{Y}|\bar{c}=1]}{\mb{P}[\bar{c}=1]\mb{P}[Z^k,\bar{Y}|\bar{c}=1]+\mb{P}[\bar{c}=0]\mb{P}[Z^k,\bar{Y}|\bar{c}=0]}  \nonumber \\
        &=\frac{\alpha\mb{P}[Z^k,\bar{Y}|\bar{c}=1]}{\alpha\mb{P}[Z^k,\bar{Y}|\bar{c}=1]+(1-\alpha)\mb{P}[Z^k,\bar{Y}|\bar{c}=0]},
        \label{eq:Pcondc1}
    \end{align}
    where given $\bar{c}=1$, we have $(Z^k,\bar{Y})^\top\sim \normal\big(\boldsymbol{0},(\Sigma^{k,1}_Y)_{[3^+],[3^+]}\big)$, and given $\bar{c}=0$, we have $(Z^k,\bar{Y})^\top\sim \normal\big(\boldsymbol{0},(\Sigma^{k,0}_Y)_{[3^+],[3^+]}\big)$.
 Similarly, we have:
    \begin{align}
        \mb{P}[\bar{c}=0|Z^k,\bar{Y}]
        &=\frac{(1-\alpha)\mb{P}[Z^k,\bar{Y}|\bar{c}=0]}{\alpha\mb{P}[Z^k,\bar{Y}|\bar{c}=1]+(1-\alpha)\mb{P}[Z^k,\bar{Y}|\bar{c}=0]},
        \label{eq:Pcondc0}
    \end{align}
    where given $\bar{c}=0
    $, we have $(Z^k,\bar{Y})^\top\sim \normal\big(\boldsymbol{0},(\Sigma^{k,0}_Y)_{[3^+],[3^+]}\big)$.

Using \eqref{eq:EZZKc1}-\eqref{eq:Pcondc0}, we can compute \eqref{eq:EZZ_kY_exp}, which together with the quantities in \eqref{eq:cov_ZZk} allows us to compute $g_k(Z^k,\bar{Y})$ in \eqref{eq:gk_MLR}. Finally,  compute $\widehat{R}^k$ by applying $g_k$ row wise to $\Theta^k$ and $Y$ (i.e., compute $g_k(\Theta_i^k,Y_i)$).

\underline{\textbf{Step 3:}} Recalling that $g_k(Z^k,\bar{Y}) = h_k(Z,Z^k,\bar{\Psi})$, we approximate $C_k=\frac{1}{n}\sum_{i=1}^bg_k'(\Theta_i^k,Y_i)$ by calculating $\E[g_k'(Z^k,\bar{Y})] = \E[\nabla_{Z^k}h_k(Z,Z^k,\bar{\Psi})]$. Here $\nabla_{Z^k} h_k$ denotes the Jacobian with respect to $Z^k$.   We compute the latter expectation by applying the generalized Stein's lemma (see Lemma \ref{lem:gen_steins}) to $(Z,Z^k)^\top\sim \normal(0,\Sigma^k)$ and $h_k(Z,Z^k,\bar{\Psi})$. This gives:
\begin{align}
    \E\left[
    \begin{bmatrix}
    Z \\
    Z^k
    \end{bmatrix}
    h(Z,Z^k,\bar{\Psi})^\top
    \right]
    =\Sigma^k\E\Big[\nabla_{(Z,Z^k)}h_k(Z,Z^k,\bar{\Psi})\Big]^\top 
    \, \in \reals^{4 \times 2}.
\end{align}
Writing the above  explicitly, we have
\begin{align}
    \begin{bmatrix}
    \E[Zh_k(Z,Z^k,\bar{\Psi})^\top] \\
    \E[Z^kh_k(Z,Z^k,\bar{\Psi})^\top]
    \end{bmatrix} 
    &=
    \begin{bmatrix}
    \Sigma^k_{(11)} & \Sigma^k_{(12)} \\
    \Sigma^k_{(21)} & \Sigma^k_{(22)}
    \end{bmatrix}
    \begin{bmatrix}
    \E[\nabla_Z h_k(Z,Z^k,\bar{\Psi})]^\top \\
    \E[\nabla_{Z^k}h_k(Z,Z^k,\bar{\Psi})]^\top
    \end{bmatrix}  \nonumber \\
    &=
    \begin{bmatrix}
    \Sigma^k_{(11)}\E[\nabla_Z h_k(Z,Z^k,\bar{\Psi})]^\top+\Sigma^k_{(12)}\E[\nabla_{Z^k} h_k(Z,Z^k,\bar{\Psi})]^\top \\
    \Sigma^k_{(21)}\E[\nabla_Z h_k(Z,Z^k,\bar{\Psi})]^\top+\Sigma^k_{(22)}\E[\nabla_{Z^k} h_k(Z,Z^k,\bar{\Psi})]^\top
    \end{bmatrix},
\end{align}
where the matrices $\Sigma^k_{(11)}, \Sigma^k_{(12}, \Sigma^k_{(21)}, \Sigma^k_{(22)} \in \reals^{2 \times 2}$ are as defined in \eqref{eq:Sig_comps}.
Looking at just the second row above and rearranging, we obtain:
\begin{align}
    \E[\nabla_{Z^k}h_k(Z,Z^k,\bar{\Psi})]=\Big\{(\Sigma^k_{(22)})^{-1}\Big(\E[Z^kh_k(Z,Z^k,\bar{\Psi})^\top]-\Sigma^k_{(21)}\E[\nabla_Zh_k(Z,Z^k,\bar{\Psi})]^\top\Big)\Big\}^\top.
\end{align}
Here $\E[Z^k h_k(Z,Z^k,\bar{\Psi})^\top]  = 
\E[Z^k  g_k(Z^k, \bar{Y}) ]$ can be approximated by $\frac{1}{n}\langle\Theta^k,g_k(\Theta^k,Y)\rangle$, and $\E[\nabla_Zh_k(Z,Z^k,\bar{\Psi})]$ can be approximated by $\frac{1}{n} g_k(\Theta^k,Y)^\top g_k(\Theta^k,Y)$ (this follows from the equivalent expressions for $\Mu^{k+1}_B$ in \eqref{eq:SE_Mk1B} and \eqref{eq:Mk1_B_alt}, noting that we have used $g_k = g_k^*$).
\label{text:gkpr_approx}

\underline{\textbf{Step 5:}} Since $\bar{B}$ is independent of  $G^{k+1}_B\sim \normal(0,\Tau^{k+1}_B)$, we have that
\begin{align}
    \begin{bmatrix}
    \bar{B} \\
    \Mu^{k+1}_B\bar{B}+G^{k+1}_B
    \end{bmatrix}
    \sim
    \normal\left(
    \begin{bmatrix}
    \E[\bar{B}] \\
    \Mu^{k+1}_B\E[\bar{B}]
    \end{bmatrix},
    \begin{bmatrix}
    \Cov[\bar{B}] & \Cov[\bar{B}](\Mu^{k+1}_B)^\top \\
    \Mu^{k+1}_B\Cov[\bar{B}] & \Mu^{k+1}_B\Cov[\bar{B}](\Mu^{k+1}_B)^\top+\Tau^{k+1}_B
    \end{bmatrix}
    \right)
\end{align}
This implies that
\begin{align}
    &f_{k+1}(\Mu^{k+1}_B\bar{B}+G^{k+1}_B=:s)
    =\E[\bar{B}|s] \nonumber  \\
    & =
    \E[\bar{B}]
    +\Cov[\bar{B}](\Mu^{k+1}_B)^\top\Big(\Mu^{k+1}_B\Cov[\bar{B}](\Mu^{k+1}_B)^\top+\Tau^{k+1}_B\Big)^{-1}
    \Big(s-\Mu^{k+1}_B\E[\bar{B}]
    \Big). \label{eq:f_k+1 formula}
\end{align}
We can use the above function to compute $f_{k+1}(B^{k+1}_j)$ if we can approximate $\Mu^{k+1}_B$ and $\Tau^{k+1}_B$ (which is the same as $\Mu^{k+1}_B$ under the Bayes-optimal choices, by \eqref{eq:SE_Tk1B} and \eqref{eq:Mk1_B_alt}). Using \eqref{eq:Mk1_B_alt}, this can be calculated using
\begin{align}
   \Tau^{k+1}_B =  \Mu^{k+1}_B\approx\frac{1}{n}g_k(\Theta^k,Y)^\top g_k(\Theta^k,Y). \label{eq:MU_k^B approx}
\end{align}

\underline{\textbf{Step 6:}} The expression for this can be obtained by taking the derivative of \eqref{eq:f_k+1 formula} w.r.t.~$s$, which gives
\begin{align}
    f_{k+1}'(s)=\Big(\Mu^{k+1}_B\Cov[\bar{B}](\Mu^{k+1}_B)^\top+\Tau^{k+1}_B\Big)^{-1}\Mu^{k+1}_B\Cov[\bar{B}],
\end{align}
where $\Mu^{k+1}_B$ and $\Tau^{k+1}_B$ can be approximated using \eqref{eq:MU_k^B approx}.

\underline{\textbf{Step 7:}} Using the formulas in \eqref{eq:SE_Sigk1}-\eqref{eq:Sig_comps} for $\Sigma^{k+1}$ and noting that $f_{k+1}$ is a conditional expectation, the covariance  $\Sigma^{k+1}$ can be approximated as 
\begin{align}
    \Sigma^{k+1}
    \approx\frac{p}{n}
    \begin{bmatrix}
    \Sigma^k_{(11)} & \frac{1}{p}f_{k+1}(B^{k+1})^\top f_{k+1}(B^{k+1}) \\
    \frac{1}{p}f_{k+1}(B^{k+1})^\top f_{k+1}(B^{k+1}) & \frac{1}{p}f_{k+1}(B^{k+1})^\top f_{k+1}(B^{k+1})
    \end{bmatrix}.
\end{align}

\subsection{Sparse Discrete Prior}

As described in Appendix \ref{sec:imp_normal_prior}, there are seven main steps in the AMP algorithm. A change in the prior requires us to make changes to our denoiser $f_k$ which affects steps 5, 6, and 7; the other steps remain unchanged. The changes are as follows, for the Bayes-optimal and soft-thresholding choices for the denoiser $f_k$.

\textbf{Bayes-optimal $f_k$}:

\underline{\textbf{Step 5:}} We have
\begin{align}
    f_{k+1}(\Mu^{k+1}_B\bar{B}+G^{k+1}_B=:s)
    =\E[\bar{B}|s] 
     =\frac{\sum_{\bar{b}}\bar{b} \, \mb{P}[\bar{B} =\bar{b}]\mb{P}[s|\bar{B}=\bar{b}]}{\sum_{\bar{b}}\mb{P}[\bar{B}=\bar{b}]\mb{P}[s|\bar{B}=\bar{b}]},
\end{align}
where $(s|\bar{B}=\bar{b}) \sim \normal(\Mu^{k+1}_B\bar{b},\Tau^{k+1}_B)$, i.e.,  $\mb{P}[s|\bar{B}=\bar{b}]$ is the bivariate Gaussian pdf with mean vector $\Mu^{k+1}_B\bar{b}$ and covariance matrix $\Tau^{k+1}_B$.

\underline{\textbf{Step 6:}} In the following, for brevity we write $f \equiv f_{k+1}$, with $f_1, f_2$ denoting its two components. By the definition of a Jacobian, we have
\begin{align}
    \nabla_s f(\Mu^{k+1}_B\bar{B}+G^{k+1}_B=s)=
    \begin{bmatrix}
    \frac{\partial f_1}{\partial s_1} & \frac{\partial f_1}{\partial s_2} \\
    \frac{\partial f_2}{\partial s_1} & \frac{\partial f_2}{\partial s_2}
    \end{bmatrix}
    =\begin{bmatrix}
    (\nabla_s f_1)^\top \\
    (\nabla_s f_2)^\top
    \end{bmatrix}
\end{align}
To compute $ \nabla_s f_1(s)$, letting $\bar{b} = [\bar{b}^{(1)}, \bar{b}^{(2)}]^{\top}$, we write 
\begin{align}
    f_1(s)
    =\frac{\sum_{\bar{b}} \bar{b}^{(1)} \,  \mb{P}[\bar{B}=\bar{b}]\mb{P}[s|\bar{B}=\bar{b}]}{\sum_{\bar{b}}\mb{P}[\bar{B}=\bar{b}]\mb{P}[s|\bar{B}=\bar{b}]}
    =:\frac{\text{num}_1}{\text{denom}_1}.
    \label{eq:quot_rule}
\end{align}
By the quotient rule for functions with a vector input and an output in $\mb{R}$, we have
\begin{align}
    \nabla_s f_1(s)
    =\frac{(\nabla_s\text{num}_1)(\text{denom}_1)-(\text{num}_1)(\nabla_s \text{denom}_1)}{\text{denom}_1^2}
\end{align}
Since $\mb{P}[s|\bar{B}=\bar{b}]$ is the bivariate Gaussian pdf with mean $\Mu^{k+1}_B\bar{b}$ and covariance matrix $\Tau^{k+1}_B$, we have that
\begin{align}
     \nabla_s \mb{P}[s|\bar{B}=\bar{b}]
    & =\nabla_s \bigg(\frac{\exp\{-\frac{1}{2}(s-\Mu^{k+1}_B\bar{b})^\top(\Tau^{k+1}_B)^{-1}(s-\Mu^{k+1}_B\bar{b})\}}{\sqrt{\text{det}(2\pi \Tau^{k+1}_B)}}\bigg) \nonumber \\
    &=(\Tau^{k+1}_B)^{-1}\big(\Mu^{k+1}_B\bar{b}-s\big)\mb{P}[s|\bar{B}=\bar{b}]
\end{align}
Using the above equation, we get
\begin{align}
    \nabla_s \text{num}_1
    &=\sum_{\bar{b}}\bar{b}^{(1)}(\Tau^{k+1}_B)^{-1}\big(\Mu^{k+1}_B\bar{b}-s\big)\mb{P}[\bar{B}=\bar{b}]\mb{P}[s|\bar{B}=\bar{b}], \nonumber \\
    \nabla_s \text{denom}_1
    &=\sum_{\bar{b}}(\Tau^{k+1}_B)^{-1}\big(\Mu^{k+1}_B\bar{b}-s\big)\mb{P}[\bar{B}=\bar{b}]\mb{P}[s|\bar{B}=\bar{b}],
\end{align}
using which $\nabla_s f_1(s)$ can be computed using \eqref{eq:quot_rule}.
The Jacobian $\nabla_s f_2(s)$ can be similarly computed.

\textbf{Soft-Thresholding $f_k$}:

\underline{\textbf{Step 5:}} We can directly compute the function in \eqref{eq:ST_denoiser}.

\underline{\textbf{Step 6:}} We can directly compute the Jacobian as shown in \eqref{eq:ST_gradient}.

\underline{\textbf{Step 7:}} We observe that the unlike the Bayes-optimal case, we no longer have the equality $\E[\bar{B}f_{k+1}(\Mu^{k+1}_B\bar{B}+G^{k+1}_B)^\top]=\E[f_{k+1}(\Mu^{k+1}_B\bar{B}+G^{k+1}_B)f_{k+1}(\Mu^{k+1}_B\bar{B}+G^{k+1}_B)^\top]$. Hence, we need to compute $\E[\bar{B}f_{k+1}(\Mu^{k+1}_B\bar{B}+G^{k+1}_B)^\top]$ separately. To do so, we evaluate each entry of $\E[\bar{B}f_{k+1}(\Mu^{k+1}_B\bar{B}+G^{k+1}_B)^\top]$ separately. We start by substituting the definitions of $\bar{B}$ and $f_{k+1}$, with $\bar{B}=(\bar{\beta}^{(1)},\bar{\beta}^{(2)})$. This gives:
\begin{align}
    &\{\bar{B}f_{k+1}^\top\}_{11} \nonumber =
    \bar{\beta}^{(1)}\text{ST}\Big(\{\bar{B}+(\Mu^{k+1}_B)^{-1}G^{k+1}_B\}_{1};\alpha\sqrt{\{N_B^{k+1}\}_{11}}\Big) \nonumber \\
    &=\bar{\beta}^{(1)}\text{ST}\Big(\bar{\beta}^{(1)}+\{(\Mu^{k+1}_B)^{-1}\}_{11}\{G^{k+1}_B\}_{1}+\{(\Mu^{k+1}_B)^{-1}\}_{12}\{G^{k+1}_B\}_2;\alpha\sqrt{\{N_B^{k+1}\}_{11}}\Big).
\end{align}
Expanding the function out and taking expectations over $\bar{\beta}^{(1)}$ and $G^{k+1}_B$, we get
\begin{align}
    \E[\{\bar{B}f_{k+1}^\top\}_{11}]
    &=
    \begin{cases}
    \varepsilon &\text{if $|\{\bar{B}+(\Mu^{k+1}_B)^{-1}G^{k+1}_B\}_1|>\alpha\sqrt{\{(\Mu^{k+1}_B)^{-1}\Tau^{k+1}_B(\Mu^{k+1}_B)^{-1}\}_{11}}$} \\
    0 &\text{otherwise}
    \end{cases}. \label{eq:entry_02}
\end{align}
Following a similar set of steps for the other entries, we have
\begin{align}
    \E[\{\bar{B}f_{k+1}^\top\}_{22}]
    &=
    \begin{cases}
    \varepsilon &\text{if $|\{\bar{B}+(\Mu^{k+1}_B)^{-1}G^{k+1}_B\}_2|>\alpha\sqrt{\{(\Mu^{k+1}_B)^{-1}\Tau^{k+1}_B(\Mu^{k+1}_B)^{-1}\}_{22}}$} \\
    0 &\text{otherwise}
    \end{cases},
    \label{eq:entry_22}
\end{align}
and
\begin{align}
    \E[\{\bar{B}f_{k+1}^\top\}_{12}]
    =\E[\{\bar{B}f_{k+1}^\top\}_{21}]
    =0.
\end{align}
In our algorithm, we do not have access to $\bar{B}+(\Mu^{k+1}_B)^{-1}G^{k+1}_B$, but have  $B^{k+1}$ whose empirical distribution (of rows) converges to $\bar{B}+(\Mu^{k+1}_B)^{-1}G^{k+1}_B$. We can therefore estimate the expectations in \eqref{eq:entry_02} and \eqref{eq:entry_22} by evaluating the right side for each row  $B^{k+1}_j$ and taking the average. For example, we compute $\E[\{\bar{B}f_{k+1}^\top\}_{11}]$ by evaluating \eqref{eq:entry_02} for each of the $p$ rows of $B^{k+1}$  and taking the average.

\section{Implementation Details for MAR} \label{appen:MAR_implementation}

Changing the matrix GLM model $q( \cdot , \cdot )$ only affects the denoising  function $g_k$ in AMP. Hence, in the seven-step implementation described in Appendix \ref{sec:imp_normal_prior}, the only  change is in the computation of $g_k$ in steps 2 and 3. (The Jacobian $g_k'$ in step 3 is approximated using $g_k$, as described on p. \pageref{text:gkpr_approx}.)  
Recall that
\begin{align}
    g_k(Z^k,\bar{Y})
    =\Cov[Z|Z^k]^{-1}\big(\E[Z|Z^k,\bar{Y}]-\E[Z|Z^k]\big).
\end{align}
Note that $\Cov[Z|Z^k]$ and $\E[Z|Z^k]$ are the same as that for mixed linear regression (with the formulas given in \eqref{eq:cov_ZZk}), so we only need to evaluate $\E[Z|Z^k,\bar{Y}] \in \reals^2$. This will be approximated using a Monte Carlo approach which we now  describe. We have
\begin{align}
    \E[Z|Z^k=u,\bar{Y}=y]
    &=\frac{\int z \, p_{Z^k}(u)p_{Z|Z^k}(z|u)p_{\bar{Y}|Z,Z^k}(y|z,u)dz}{\int p_{Z^k}(u)p_{Z|Z^k}(z|u)p_{\bar{Y}|Z,Z^k}(y|z,u)dz} \nonumber \\
    &=\frac{\E_{Z|Z^k=u}[Z \, p_{Z^k}(u)p_{\bar{Y}|Z,Z^k}(y|Z,u)]}{\E_{Z|Z^k=u}[ \, p_{Z^k}(u)p_{\bar{Y}|Z,Z^k}(y|Z,u)]}, \label{eq:MC_formula1}
\end{align}
where given $Z^k=u$ and $\bar{Y}=y$, the probability density functions inside the expectations can be easily evaluated since $Z^k\sim\normal(0,\Sigma^k_{(22)})$, and for $\sigma>0$ and $z=(z_1,z_2)$,
\begin{align}
    p_{\bar{Y}|Z,Z^k}(y|z,u)
    &=p_{\bar{Y}|Z}(y|z) \nonumber \\
    &=
    \begin{cases}
    \phi_{\normal}\Big(\frac{y-z_1-b_1}{\sigma}\Big), \text{ if $z_1+b_1>z_2+b_2$}  \\
    \phi_{\normal}\Big(\frac{y-z_2-b_2}{\sigma}\Big), \text{ otherwise}
    \end{cases}, \label{eq:p_Ybar_given_Z_Zk}
\end{align}
where $\phi_{\normal}$ is the standard Gaussian density. With the above density functions, we can approximate the numerator and denominator of \eqref{eq:MC_formula1} by sampling $z$'s from
\begin{align}
    (Z|Z^k=u)
    \sim
    \normal\Big(\Sigma^k_{(12)}\big(\Sigma^k_{(22)}\big)^{-1}u, \, \Sigma^k_{(11)}-\Sigma^k_{(12)}\big(\Sigma^k_{(22)}\big)^{-1}\Sigma^k_{(21)}\Big), 
\end{align}
and then taking the averages of the functions inside the expectations. 

Additionally, the EM part of the EM-AMP algorithm requires $\E[Z|\bar{Y}]$. This is computed using Monte Carlo in a similar fashion to $\E[Z|Z^k,\bar{Y}]$:
\begin{align}
    \E[Z|\bar{Y}=y ]
    &=\frac{\int z \, p_Z(z)p_{\Bar{Y}|Z}(y|z)dz}{\int p_Z(z)p_{\Bar{Y}|Z}(y|zdz}
    =\frac{\E_Z[Z \, p_{\Bar{Y}|Z}(y|Z)]}{\E_Z[p_{\Bar{Y}|Z}(y|Z)]}, \label{eq:MC_formula2}
\end{align}
where $p_{\Bar{Y}|Z}(y|z)$ is given in  \eqref{eq:p_Ybar_given_Z_Zk}. We can approximate the numerator and denominator of \eqref{eq:MC_formula2} by sampling $z$'s from $Z\sim\normal(0,\Sigma^k_{(11)})$ and then taking the averages of the functions inside the expectations.  

In the $m$th iteration  of EM-AMP, the expectations in \eqref{eq:MC_formula1} and \eqref{eq:MC_formula2} are computed by using the current intercept estimates $(b^m_1, b^m_2)$ in the formula for $p_{\Bar{Y}|Z}$ in \eqref{eq:p_Ybar_given_Z_Zk}.

\section{Implementation Details for MOE} \label{appen:MOE_implementation}

The changes required here are similar to those of MAR stated in Appendix \ref{appen:MAR_implementation}. In the seven-step implementation described in Appendix \ref{sec:imp_normal_prior}, the only  change is in  the computation of $g_k$ in steps 2 and 3.  Recall that
\begin{align}
    g_k(Z^k,\bar{Y})
    =\Cov[Z|Z^k]^{-1}\big(\E[Z|Z^k,\bar{Y}]-\E[Z|Z^k]\big).
\end{align}
Note that $\Cov[Z|Z^k]$ and $\E[Z|Z^k]$ are the same as that for mixed linear regression (with the formulas given in \eqref{eq:cov_ZZk}), so we only need to evaluate $\E[Z|Z^k,\bar{Y}] \in \reals^4$. This will be approximated using the same Monte Carlo approach as MAR, by writing 
\begin{align}
    \E[Z|Z^k=u,\bar{Y}=y]
    &=\frac{\E_{Z|Z^k=u}[Z \, p_{Z^k}(u)p_{\bar{Y}|Z,Z^k}(y|Z,u)]}{\E_{Z|Z^k=u}[ \, p_{Z^k}(u)p_{\bar{Y}|Z,Z^k}(y|Z,u)]}, \label{eq:MC_formula3}
\end{align}
where given $Z^k=u$ and $\bar{Y}=y$, the probability density functions inside the expectations can be easily evaluated since $Z^k\sim\normal(0,\Sigma^k_{(22)})$, and for $\sigma>0$ and $z=(z_1,z_2,z_3,z_4)$,
\begin{align*}
    p_{\bar{Y}|Z,Z^k}(y|z,u)
    &=p_{\bar{Y}|Z}(y|z)
    =\int_0^1p_{\bar{Y},\bar{\psi}|Z}(y,v|z)dv
    \stackrel{(a)}{=}\int_0^1p_{\bar{\psi}}(v)p_{\bar{Y}|\bar{\psi},Z}(y|v,z)dv \\
    &\stackrel{(b)}{=}\int_0^{\frac{\exp(z_3)}{\exp(z_3)+\exp(z_4)}}p_{\bar{\psi}}(v)\phi_\normal\Big(\frac{y-z_1}{\sigma}\Big) dv+\int_{\frac{\exp(z_3)}{\exp(z_3)+\exp(z_4)}}^1p_{\bar{\psi}}(v)\phi_\normal\Big(\frac{y-z_2}{\sigma}\Big)dv \\
    &\stackrel{(c)}{=}\frac{\exp(z_3)}{\exp(z_3)+\exp(z_4)}\phi_\normal\Big(\frac{y-z_1}{\sigma}\Big)+\Big(1-\frac{\exp(z_3)}{\exp(z_3)+\exp(z_4)}\Big)\phi_\normal\Big(\frac{y-z_2}{\sigma}\Big),
\end{align*}
where $\phi_{\normal}$ is the standard Gaussian density.
Here (a) uses the independence between $\bar{\psi}$ and $Z$ (see assumption in sentence below \eqref{eq:hk_def}), (b) uses the fact that $\bar{Y}=Z_1+\bar{\epsilon}$ when $\bar{\psi}\leq\frac{\exp(Z_3)}{\exp(Z_3)+\exp(Z_4)}$ and $\bar{Y}=Z_2+\bar{\epsilon}$ when $\bar{\psi}>\frac{\exp(Z_3)}{\exp(Z_3)+\exp(Z_4)}$, (c) uses $p_{\bar{\psi}}(v)=1$ for all $v\in[0,1]$ since $\bar{\psi}\sim\text{Uniform}[0,1]$. With the above density functions, we can approximate the numerator and denominator of \eqref{eq:MC_formula3} by sampling $z$'s from
\begin{align}
    (Z|Z^k=u)
    \sim
    \normal\Big(\Sigma^k_{(12)}\big(\Sigma^k_{(22)}\big)^{-1}u, \, \Sigma^k_{(11)}-\Sigma^k_{(12)}\big(\Sigma^k_{(22)}\big)^{-1}\Sigma^k_{(21)}\Big), 
\end{align}
and then taking the averages of the functions inside the expectations. 

\vskip 0.2in
{\small{\bibliography{main}}}

\end{document}